\documentclass[wcp]{jmlr2e}

\usepackage{longtable}

\usepackage{booktabs}

\usepackage{bm}
\usepackage{multirow}

\usepackage{lineno}

\makeatletter
\let\Ginclude@graphics\@org@Ginclude@graphics 
\makeatother

\jmlrvolume{222}
\jmlryear{2023}
\jmlrworkshop{ACML 2023}

\newtheorem{assumption}{Assumption}
\newtheorem{model}{Model}

\newcommand{\PP}{\mathbb{P}}
\newcommand{\EE}{\mathbb{E}}
\newcommand{\RR}{\mathbb{R}}
\renewcommand{\ge}{\geqslant}
\renewcommand{\le}{\leqslant}
\newcommand{\indicator}[1]{\operatorname{I} \left\{ {#1} \right\} }
\newcommand{\ball}{\mathcal{B}}
\newcommand{\T}{{\operatorname{T}}}

\newcommand{\tr}{\operatorname{Tr}}

\newcommand{\tv}{\mathbf{t}}
\newcommand{\xv}{\mathbf{x}}
\newcommand{\yv}{\mathbf{y}}

\newcommand{\labels}{\mathbf{Y}}
\newcommand{\meanLabels}{\mathbf{m}}

\newcommand{\weight}{\omega}
\newcommand{\weights}{\bm{\omega}}
\newcommand{\variances}{{\bm{\sigma}^2}}
\newcommand{\diag}{\mathbf{D}}

\newcommand{\support}{\mathcal{S}}

\newcommand{\meanY}{f}

\newcommand{\estimator}[1]{\widehat{#1}}
\newcommand{\estimatorn}[1]{\widehat{#1}_n}
\newcommand{\Var}{\operatorname{Var}}
\newcommand{\normDistribution}[2]{{\mathcal{N} \left(#1, #2 \right)}}
\newcommand{\risk}{\mathcal{R}}
\newcommand{\excessRisk}{\mathcal{E}}

\newcommand{\constCorTwo}{\mathtt{c}_{cor.2}}
\newcommand{\constCorTwoPrime}{\tilde{\mathtt{c}}_{cor.2}}
\newcommand{\logProb}{\mathtt{z}_{prop.1}}
\newcommand{\logProbBelow}{\mathtt{z}_{prop.2}}

\title[Selective Nonparametric Regression via Testing]{Selective Nonparametric Regression via Testing}

 \author{\Name{Fedor Noskov} \Email{fnoskov@hse.ru}\\
     \addr HSE University,\\ 
     Institute for Information Transmission Problems RAS, \\
     and Moscow Institute of Science and Technology (MIPT), Moscow, Russia
     \AND
     \Name{Alexander Fishkov} \Email{alexander.fishkov@skoltech.ru}\\
     \addr Skolkovo Institute of Science and Technology (Skoltech), Moscow, Russia
     \AND
     \Name{Maxim Panov} \Email{panov.maxim@gmail.com}\\
     \addr Mohamed bin Zayed University of Artificial Intelligence (MBZUAI), Abu Dhabi, UAE
}

\begin{document}

\maketitle

\begin{abstract}%
  Prediction with the possibility of abstention (or selective prediction) is an important problem for error-critical machine learning applications. While well-studied in the classification setup, selective approaches to regression are much less developed. In this work, we consider the nonparametric heteroskedastic regression problem and develop an abstention procedure via testing the hypothesis on the value of the conditional variance at a given point. Unlike existing methods, the proposed one allows to account not only for the value of the variance itself but also for the uncertainty of the corresponding variance predictor. We prove non-asymptotic bounds on the risk of the resulting estimator and show the existence of several different convergence regimes. Theoretical analysis is illustrated with a series of experiments on simulated and real-world data.
\end{abstract}

\begin{keywords}%
  nonparametric regression, selective regression, prediction with abstention, hypothesis testing
\end{keywords}

\section{Introduction}
  In many machine learning applications, there exists a possibility to reject the prediction of the model and entrust it to the human or other model. Abstention is usually done based on the estimation of uncertainty in predicted value. In classification problems uncertainty might be measured via the probability of wrong prediction while for regression it corresponds to the expected error. In both cases, the estimation of these quantities is usually much harder than the solution of the initial prediction problem. In this work, we target the problem of regression with abstention (or selective regression) in nonparametric setup.

  \textbf{Related Works.} There is a large variety of literature regarding classification with reject option. Most likely, the problem was firstly studied by Chow in papers~\citep{chow_optimum_1957, chow_optimum_1970}. Moreover, in the article~\citep{chow_optimum_1970} he introduced a risk function used in the majority of forthcoming works including the present one. \cite{herbei_classification_2006} studied an optimal procedure for this risk and proved consistency for the proposed plugin rule. Then the research was focused on investigation of either empirical risk minimization among a class of hypotheses~\citep{bartlett_classification_2008, cortes_learning_2016} or on other types of risk~\citep{denis2020consistency, el-yaniv_foundations_2010, lei_classification_2014}. Benefits of abstention for online and active learning were studied in~\citep{neu2020fast} and~\citep{puchkin2021exponential} correspondingly. Besides, the problem was studied in a number of more practical works; see, for example,~\citep{grandvalet_support_2009, geifman_selectivenet_2019, nadeem_accuracy-rejection_2009}.  Finally, conformal prediction approach~\citep{vovk1999machine,shafer2008tutorial} has recently been applied to the classification with reject option~\citep{linusson2018classification,johansson2023conformal}.

  Unfortunately, methods for selective regression are much less developed. \citet{denis_regression_2020} suggested an approach to regression via a plugin rule. In papers~\citep{shah_selective_2022} and~\citep{salem_gumbel-softmax_2022}, authors proposed new approaches of neural network learning for better uncertainty capturing. In~\citep{jiang_risk-controlled_2020}, the authors suggested an uncertainty measure for regression based on blending and a method to select samples with the least risk given some coverage.

  \textbf{Setup.} In this work, we focus on the selective algorithms for regression problems with heteroskedastic noise. We assume that the data $(X, Y)$ is coming from a standard regression model $Y = f(X) + \varepsilon$ with target function $f$ and i.i.d. noise $\varepsilon$. Covariate $X$ is assumed to follow some distribution $p(\cdot)$. The noise variance depends on the input point: $\sigma^2(\xv) = \Var[Y \mid X = \xv]$. The Chow model~\citep{chow_optimum_1970} assumes that the cost for abstention is given by a fixed value $\lambda > 0$, while for prediction the mean squared risk is paid. The abstention procedure for such a problem can be constructed based on the estimate of the variance $\estimator{\sigma}^2(\xv)$. The abstention rule $\estimator{\alpha}(\xv)$ proposed by~\citet{denis_regression_2020} is given by $\estimator{\alpha}(\xv) = \indicator{\estimator{\sigma}^2(\xv) \ge \lambda}$. The resulting method was proved to be consistent, and the corresponding rate of convergence was derived under standard nonparametric assumptions on functions $f$ and $\sigma$. However, the analysis was done only for the risk averaged over the covariate distribution $p(\xv)$, while one may expect that the convergence properties at a given point $\xv$ may significantly depend on the difference between the variance $\estimator{\sigma}^2(\xv)$ and cost of abstention $\lambda$. Moreover, the performance of the estimator $\estimator{\alpha}(\xv) = \indicator{\estimator{\sigma}^2(\xv) \ge \lambda}$ depends on how accurately $\estimator{\sigma}^2(\xv)$ estimates the true variance $\sigma^2(\xv)$. In particular, $\estimator{\sigma}^2(\xv)$ might give unreliable predictions in the areas of design space where there is little to no train data. Such situations arise when there is a covariate shift between train and test data. In this work, we aim to conduct in-depth theoretical analysis for the pointwise estimation risk for the considered problem and propose the abstention procedure that would be more robust to covariate shifts than the one based on the plugin rule. 

  The main \textbf{contributions} of our paper are the following.
  \begin{itemize}
    \item We show the natural way to construct the abstention rule for nonparametric heteroskedastic regression based on the hypothesis testing on the variance value at a given point. We implement the method via Nadaraya-Watson kernel estimates of regression and variance functions.

    \item We prove the accurate finite sample bounds for the risk of the resulting estimator. Our results show that the behavior of the risk significantly depends on the relative values of the variance $\sigma^2(\xv)$ and the abstention cost $\lambda$. The proposed method shows favorable performance over the plugin approach of~\citet{denis_regression_2020}, see Table~\ref{tab: rates of excess risk}.

    \item We illustrate the theoretical findings by experiments with simulated and real-world data.
  \end{itemize}

  The paper is organized as follows. We introduce the setup of the study in Section~\ref{sec:setup}. We propose a new abstention procedure based on hypothesis testing on the values of the conditional variance in Section~\ref{sec:method}. Theoretical properties of the developed method are studied in Section~\ref{sec:theory}. Finally, Section~\ref{sec:experiments} illustrates our experimental findings and Section~\ref{sec:conclusion} concludes the study.

\section{Regression with Abstention}
\label{sec:setup}
  Let us start by formalizing the problem. We assume that we observe pairs $(X, Y)$ with covariate $X \in \RR^d$ and output $Y \in \RR$. The regression task is to estimate $\EE_Y [Y \mid X = \xv]$ via some function $\estimator{f}(\xv)$, where $\EE_Y[\cdot \mid X = \xv]$ means the expectation over the distribution $Y \mid X = \xv$. For the case of regression with abstention, for each $\xv$ we decide to accept or to reject the prediction $\estimator{f}(\xv)$. Thus, we introduce an indicator of abstention $\estimator{\alpha}(\xv)$ which is equal to 1 if the prediction $\estimator{f}(\xv)$ was \textit{rejected}. The intuition suggests accepting the prediction if the expected squared error $\EE_Y[(\estimator{f}(X) - Y)^2 \mid X = \xv]$ is not too large, say less than some $\lambda$. 

  That leads to a natural definition of risk which is a variant of the risk proposed in~\citep{chow_optimum_1970}: 
  \begin{align*}
    \risk_\lambda(\xv) = \EE_Y \bigl[(\estimator{f}(X) - Y)^2 \indicator{\estimator{\alpha}(\xv) = 0} \mid X = \xv\bigr] + \lambda \indicator{\estimator{\alpha}(\xv) = 1},
  \end{align*}
  where $\indicator{\cdot}$ is an indicator function. The introduced risk has a natural interpretation. If we abstain from prediction then we should pay the fixed cost $\lambda$. Otherwise, we pay the expected squared error. Note that the provided risk is not the only option for the problem. For instance, people also considered coverage risk, see~\citep{jiang_risk-controlled_2020}.

  Given a risk function, the following question rises up. What are the estimators that minimize it in each point? We formulate the answer as a proposition.
  \begin{proposition}
  \label{proposition: optimal estimators}
    Define $\meanY(\xv) = \EE [Y \mid X = \xv]$ and $\sigma^2(\xv) = \Var [Y \mid X = \xv]$. Then, $\meanY$ is the optimal estimator of $Y \mid X = \xv$ and $\alpha(\xv) = \indicator{\sigma^2(\xv) \ge \lambda}$ is the optimal abstention function.
  \end{proposition}
  The risk related to the pair $\{\meanY(\xv), \alpha(\xv)\}$ we denote by $\risk^*_\lambda(\xv)$.

\section{Abstention via Testing of Variance Values}
\label{sec:method}
  The setup considered in previous section was previously explored in~\citep{denis_regression_2020} where it was proposed to use plugin approach, i.e. use some estimators $\estimator{f}$ and $\estimator{\sigma}^2$ of the population counterparts $\meanY$ and $\sigma^2$ directly in the rule given by Proposition~\ref{proposition: optimal estimators}. Their approach leads to consistent estimators in large sample regime. However, for finite samples not only $\estimator{f}$ can be imperfect but also the variance estimator $\sigma^2(\xv)$ 
  can become unreliable if $\xv$ lies far away from the train set under nonparametric setting. Basically, we might start rejecting or accepting the predictions based on the variance estimate which is far off from the actual variance values. 

  In this work, we aim to work with this issue by considering the uncertainty in the variance estimator $\sigma^2(\xv)$ itself. 
  We propose a natural way to take this into account via testing between the following hypotheses:
  \begin{align*}
    H_0\colon \sigma^2(\xv) \ge \lambda \text{ vs. } H_1\colon \sigma^2(\xv) < \lambda. 
  \end{align*}
  This problem assumes that it is safer to reject a good prediction than to accept a bad one. It is the standard situation for many applications of selective machine learning.

  Construction of the test requires some assumptions on the data which will be the same for the train and the test set. Thus, we introduce the studied model.
  \begin{model}
    \label{model: normal distribution of labels}
    Given a sample $X \in \RR^d$, the observed label $Y$ is normally distributed with the mean $\meanY(X)$ and the variance $\sigma^2(X)$ for some functions $\meanY\colon \RR^d \to \RR$, $\sigma^2\colon \RR^d \to \RR_+$.
  \end{model}
  The normality of the noise is not an obligatory requirement, but it allows computing some constants precisely. In our analysis, we mostly use concentration inequalities that can be naturally extended to sub-Gaussian setting. We will work under general nonparametric assumptions on functions $\meanY$ and $\sigma^2$, see the details in Section~\ref{sec:theory}.

\subsection{Construction of the Test}
  Nonparametric estimation offers a variety of tools for regression such as kNN, splines or kernel methods~\citep{tsybakov_introduction_2009}. In this work, we stick to kernel approaches and employ celebrated Nadaraya-Watson (NW) method that estimates a function at a point $\xv$ via weighted mean of its neighbours. Below, we introduce the method formally.

  Let $\mu$ be the Lebesgue measure in $\RR^d$. For a kernel $K\colon \RR^d \to \RR_+$, $\int_{\RR^d} K(\tv) d \mu(\tv) = 1$, NW method computes weights of samples $X_1, \ldots, X_n$ at the point $\xv$ as
  \begin{align}
  \label{eq: NW weight definition}
    \weight_i(\xv) = \frac{K\left(\frac{\xv - X_i}{h}\right)}{\sum_{i = 1}^n K\left(\frac{\xv - X_i}{h}\right)},
  \end{align}
  where $h$ is a bandwidth. Typically, $h$ tends to 0 as $n$ tends to infinity. Then, it computes the estimated mean
  \[
    \estimatorn{f}(\xv) = \sum_{i = 1}^n \weight_i(\xv) Y_i
  \] of the conditional distribution $Y \mid X = \xv$. This approach can be extended for computing the estimator of variance $\Var[Y \mid X = \xv]$:
  \begin{align*}
    \estimatorn{\sigma}^2(\xv) = \sum_{i = 1}^n \weight_i(\xv) Y_i^2 - \left(\sum_{i = 1}^n \weight_i(\xv) Y_i \right)^2.
  \end{align*}
  Generally, estimates for mean and variance can use different kernels and bandwidths. However, we stick to the single choice in this work to make the results simpler and more illustrative.

  In the paper~\citep{fan_efficient_1998}, it was shown that under some assumptions on $h, n$ and $K(\cdot)$, we have
  \begin{align}
  \label{eq: normal limit}
    \sqrt{n h^d} \left(\estimatorn{\sigma}^2(\xv) - \sigma^2(\xv)\right) \underset{nh^d \to \infty}{\longrightarrow} \normDistribution{0}{\sigma_V^2},
  \end{align}
  where $\sigma_V^2(\xv) = \sigma^4(\xv) \frac{2 \Vert K \Vert_2^2}{p(\xv)}$, $\Vert K \Vert_2^2 = \int_{\RR^d} K^2(\tv) d \mu(\tv)$ and $p(\cdot)$ is a marginal density of covariates $X$. Thus, we obtain
  \begin{align*}
    \lim \sup_{n h^d \to \infty } \PP \left( \sigma^2(\xv) - \estimatorn{\sigma}^2(\xv) \ge z_{1 - \beta} \sigma^2(\xv) \sqrt{\frac{2 \Vert K \Vert_2^2}{n h^d p(\xv)}} \right) \le \beta.
  \end{align*}
  This convergence result allows to construct confidence sets with the guaranteed asymptotic coverage. Since $\sigma^2(\xv) \ge \lambda$ under the null hypothesis, we obtain the test
  \begin{align}
  \label{eq: test definition}
    \estimatorn{\sigma}^2(\xv) \le \lambda \left (1 - z_{1 - \beta} \Vert K \Vert_2 \sqrt{\frac{2}{n h^d p(\xv)}} \right ).
  \end{align}
  Due to the Slutsky lemma, if we replace $p(\xv)$ with some consistent estimator $\estimatorn{p}(\xv)$, the above still will be the test of asymptotic significance level $\beta$. For the density estimator, $\estimatorn{p}(\xv)$ we suggest the nonparametric estimator $\estimatorn{p}(\xv) = \frac{1}{n h^d} \sum_{i = 1}^n K\left( \frac{\xv - X_i}{h}\right)$.

\subsection{Abstention Algorithm}
  The derived test allows to construct the procedure of regression with reject option. The only remaining thing we should check before applying the test is that $\estimatorn{p}(\xv)$ is not zero. Let $a$ and $b$ be such numbers that $K(\tv) \ge a \cdot \indicator{\Vert \tv \Vert \le b}$ for all $\tv \in \RR^d$. For theoretical purposes we demand $\estimatorn{p}(\xv)$ to be greater than $4 a / (n h^d)$ for any accepted point $\xv$, see the details in Section~\ref{sec:main_proof}. From the construction of the test, it also follows that the prediction $\estimatorn{f}(\xv)$ is rejected independent of the value $\estimatorn{\sigma}^2(\xv)$ if $\estimatorn{p}(\xv) \le \frac{2 z_{1 - \beta}^2}{n h^d} \int_{\RR^d} K^2(\tv) d \mu(\tv)$. The resulting procedure is summarized in Algorithm~\ref{algo: test}.

  \begin{algorithm2e}[t!]
    {\caption{Acceptance testing}
    \label{algo: test}}
    \DontPrintSemicolon
    \KwIn{samples $\{(X_i, Y_i)\}_{i = 1}^n$, bandwidth $h$, parameters $\lambda, \beta, a$}
    \KwOut{accept or reject the regression result}
    Calculate $\estimatorn{p}(\xv), \estimatorn{\sigma}^2(\xv)$\;
    \eIf{$\estimatorn{p}(\xv) \ge \frac{4 a}{n h^d}$ \textbf{and} criterion~\eqref{eq: test definition} holds
    }{
            accept results of the regression\;
    }{
        reject\;
    }
 \end{algorithm2e}
    
  The proposed procedure was designed for abstract features in $\RR^d$. However, in machine learning applications we often have quite complex features as images or texts, and neural networks are usually used for their processing. The considered method might be coupled with neural networks by applying it to some embedding space induced by a neural network. 

\section{Theoretical guarantees}
\label{sec:theory}
  In this section, we provide theoretical guarantees for our algorithm. There are some natural assumptions that should hold to obtain our results.

  \begin{assumption}
  \label{assumption: mean assumption}
    The Hessian of the function $\meanY$ exists and is bounded by $H_f$. Moreover, $\meanY$ is $L_f$-Lipschitz.
  \end{assumption}

  Assumption~\ref{assumption: mean assumption} helps to reduce the bias in the estimation of $\meanY$. Roughly speaking, if the kernel is symmetric then $\EE \estimatorn{f}(\xv) - \meanY(\xv)$ has order at most $h^2$ times the second derivative of $\meanY$. Otherwise, $h$ times the Lipschitz constant may appear in the decomposition of the bias $\EE \estimatorn{f}(\xv) - \meanY(\xv)$. We also impose the similar assumption for $\sigma^2(\xv)$.

  \begin{assumption}
  \label{assumption: variance assumption}
    The Hessian of the function $\sigma^2$ exists and is bounded by $H_\sigma$. Moreover, $\sigma^2$ is $L_{\sigma^2}$-Lipschitz.
  \end{assumption}

  As was previously mentioned, the bias term of order $h$ vanishes if the kernel $K$ is symmetric. Besides, to estimate $\meanY$ at a point $\xv$, the kernel should aggregate well the neighborhood of $\xv$. Thus, its support should cover some ball in $\RR^d$. But the kernel should not rely on the response provided by far points, so we require exponential tail for the kernel. The most common assumption is that the support of the kernel is bounded, but it is not the case of the Gaussian kernel which is widely used. Formally, the case of the Gaussian kernel implies that $\estimatorn{p}(\xv)$ is non-zero over the whole space $\RR^d$ but we start considering a point $\xv$ as explored only if it has estimated density at least $\Theta\bigl((n h^d)^{-1}\bigr)$. That allows to derive standard bias-variance decomposition and has a natural interpretation in terms of regression with abstention.

  \begin{assumption}
  \label{assumption: kernel assumptions}
    For the kernel $K\colon \RR^d \to \RR_+$, there exist constants $a$ and $b$ such that
    \begin{align*}
      K(\tv) \ge a \indicator{\Vert \tv \Vert \le b}
    \end{align*}
    holds for all $\tv \in \RR^d$. The kernel is symmetric, i.e. $K(\tv) = K(-\tv)$. Moreover, there are constants $R_K$ and $r_K$ such that for all $\tv$, it holds that
    \begin{align*}
      K(\tv) \le R_K e^{-r_K \Vert \tv \Vert}.
    \end{align*}
  \end{assumption}

  Finally, we impose some conditions on the density $p(\xv)$. In the classical nonparametric studies, it is usually assumed that the support of $p(\xv)$ has positive Lebesgue measure so we do not consider nonparametric low-dimensional manifold estimation. We define
  \begin{align*}
    \support_q = \{\xv \in \RR^d \mid p(\xv) > q\}.
  \end{align*}
  We denote the support $\operatorname{cl} (\support_0)$ by $\support$ and the boundary of $\support$ by $\partial \support$. Inside the support $\support$ we require $p(\xv)$ to be Lipschitz. That also helps to suppress summands of order $h$ in the bias of our estimator. So the density can be non-continuous at the boundary like the uniform distribution, but it will not affect the inference inside the support. 

  \begin{assumption}
  \label{assumption: density assumption}
    The density $p$ of $X$ is $L_p$-Lipschitz in $\support_0$ and bounded by $C_p$.
  \end{assumption}

  To bound the excess risk at a point $\xv$, we need its neighborhood to be explored a bit. So there should be large enough probability mass in a ball of radius $h$ around $\xv$. Thus, we require $p(\xv)$ to be larger than $C h$ and the Euclidean distance to the boundary $d(\xv, \partial \support)$ to be at least $C' h$. If $\partial \support = \varnothing$, then $d(\xv, \partial \support)$ is assumed to be infinite. 

  Finally, note that $\risk_\lambda(\cdot)$ depends on the training set. We bound the mean of the excess risk over all training sets $\mathcal{D} = \{(X_i, Y_i)\}_{i = 1}^n$ where $X_i$ are i.i.d. samples from the density $p(\cdot)$ and $Y_i$ generated according to Model~\ref{model: normal distribution of labels}.

  In the theorem below, we study the upper bounds for the risk. The notation $\lesssim$ means that the corresponding inequality holds with some multiplicative constant that is independent of $n, h, \beta$ and $p(\xv)$. The formulation with all the constants presented in the explicit way is given in Supplementary Material, see Theorem~\ref{theorem: finite-sample thm}.
  \begin{theorem}
      Suppose that Assumptions~\ref{assumption: mean assumption}-\ref{assumption: density assumption} hold. Define $\Delta(\xv) = |\sigma^2(\xv) - \lambda|$. Let $\excessRisk_\lambda(\xv)$ be the excess risk of the estimator $\estimatorn{f}(\xv)$ and the abstention rule $\estimatorn{\alpha}(\xv)$ introduced in Algorithm~\ref{algo: test}. Let $\EE_{\mathcal{D}}$ be the expectation with respect to training dataset $\mathcal{D} = \{(X_i, Y_i)\}_{i = 1}^n$, where $X_1, \ldots, X_n$ are i.i.d. samples from then density $p(\cdot)$. Then
      \begin{itemize}
        \item if $\sigma^2(\xv) \ge \lambda$ and $\Delta(\xv) \le C_1 \{ n h^d p(\xv) \}^{-1} + C_2 h^2 / p(\xv) - C_3 z_{1 - \beta} \{n h^d p(\xv) \}^{-1/2}$, we have
          \[
            \EE_{\mathcal{D}} (\excessRisk_\lambda(\xv)) \lesssim \{n h^d p(\xv) \}^{-1} + h^{4} p^{-2}(\xv) + \Delta(\xv),
          \]

        \item if $\sigma^2(\xv) \ge \lambda$ and $\Delta(\xv) \ge C_1 \{ n h^d p(\xv) \}^{-1} + C_2 h^2 / p(\xv) - C_3 z_{1 - \beta} \{n h^d p(\xv) \}^{-1/2}$, we have
        \[
            \EE_{\mathcal{D}} (\excessRisk_\lambda(\xv)) \lesssim \Delta(\xv) \exp \left ( - \Omega(n h^{d + 2} p(\xv))\right )
        \]

        \item if $\sigma^2(\xv) \ge \lambda$ and $\Delta(\xv) \ge C_1 \{ n h^d p(\xv) \}^{-1} + C_2 h - C_3 z_{1 - \beta} \{n h^d p(\xv) \}^{-1/2}$, we have
          \[
            \EE_{\mathcal{D}} (\excessRisk_\lambda(\xv)) \lesssim \exp \{ - n h^d p(\xv)\},
          \]

        \item if $\sigma^2(\xv) < \lambda$ and $\Delta(\xv) \le C_1' \{ n h^d p(\xv) \}^{-1} + C_2' h^2/p(\xv) + C_3' z_{1 - \beta} \{n h^d p(\xv) \}^{-1/2}$, we have
          \[
            \EE_{\mathcal{D}}(\excessRisk_\lambda(\xv)) \lesssim \{n h^d p(\xv) \}^{-1} + h^{4} p^{-2}(\xv) + \Delta(\xv),
          \]

        \item if $\sigma^2(\xv) < \lambda$ and $\Delta(\xv) \gg C_1' \{ n h^d p(\xv) \}^{-1} + C_2' h^2/p(\xv) + C_3' z_{1 - \beta} \{n h^d p(\xv) \}^{-1/2}$, we have
          \[
            \EE_{\mathcal{D}}(\excessRisk_\lambda(\xv)) \lesssim \{n h^d p(\xv) \}^{-1} + h^{4} p^{-2}(\xv).
          \]
      \end{itemize}
    \label{theorem: informal thm}
  \end{theorem}
  Let us note that Theorem~\ref{theorem: informal thm} applies not only to Algorithm~\ref{algo: test} but also to the plugin estimator proposed by~\cite{denis_regression_2020}. Indeed, by setting $\beta = 0.5$ one gets $z_{1 - \beta} = 0$ and we obtain plugin approach as a particular instance of our algorithm. 
  While Theorem~\ref{theorem: informal thm} determines only the upper bound of the risk, it satisfactorily captures the real behavior of Algorithm~\ref{algo: test}, see experimental evaluation in Section~\ref{sec:experiments}. Below, we discuss different estimation regimes implied by Theorem~\ref{theorem: informal thm}.

  For beginning, we consider the case when $\sigma^2(\xv) > \lambda$. In most of the applications, we assume that $n h^d \to \infty$ as $n$ tends to infinity. Typically, $h$ is chosen to minimize bias-variance trade-off so $h = \Theta\bigl(n^{-1/(d + 4)}\bigr)$. Assume additionally that $\beta < 0.5$ and
  \begin{align*}
      C_1 \{ n h^d p(\xv) \}^{-1} + C_2 h^2 / p(\xv) - C_3 z_{1 - \beta} \{n h^d p(\xv) \}^{-1/2} < 0,
  \end{align*}
  where constants $C_1, C_2, C_3$ come from the first case of Theorem~\ref{theorem: informal thm}.  This inequality can be satisfied if $h = C_{\beta} n^{-1/(d + 4)} p^{1/2}(\xv)$ for a small enough constant $C_{\beta}$ that depends on $\beta$. We refer to this condition as {\it ``undersmoothing''} since it requires the bias to be significantly less than the variance. Moreover, a similar condition is required to ensure~\eqref{eq: normal limit}. Then, 
  our approach provably becomes very efficient. Indeed, in that case the condition $\Delta(\xv) \le C_1 \{ n h^d p(\xv) \}^{-1} + C_2 h^2 / p(\xv) - C_3 z_{1 - \beta} \{n h^d p(\xv) \}^{-1/2}$ can be simplified as $\Delta(\xv) < 0$ so it never holds. Thus, for any $\xv$ such that $\sigma^2(\xv) > \lambda$, the expected excess risk converges exponentially. But if one chooses larger $h$, the advantages of our algorithm remain, since it becomes to converge exponentially earlier than the plugin.

  For the plugin, our upper bound can not achieve exponential convergence rates while $\Delta(\xv) \le C_1 \{ n h^d p(\xv)\}^{-1} + C_2 h^2/ p(\xv)$. That matches our observations for synthetic data, see Figure~\ref{fig:excess_risk_n} and Figure~\ref{fig:comparison_plugin}.

  To explain the behaviour of estimators for $\sigma^2(\xv) \le \lambda$, we impose the following proposition. 
  \begin{proposition}
  \label{proposition: excess risk decomposition}
    For any pair of estimators $(\estimator{f}, \estimator{\alpha})$ the expected excess risk can be decomposed as follows:
    \begin{align*}
      \EE_{\mathcal{D}} \risk_\lambda(\xv) - \risk^*(\xv) = \EE_{\mathcal{D}} \left[\bigl(\estimator{f}(\xv) - \meanY(\xv)\bigr)^2 \indicator{\estimator{\alpha}(\xv) = 0}\right] + \Delta(\xv) \cdot \PP \bigl(\estimator{\alpha}(\xv) \neq \alpha(\xv)\bigr).
    \end{align*}
  \end{proposition}
  In our case
  \begin{align*}
    \PP \bigl(\estimator{\alpha}(\xv) \neq \alpha(\xv)\bigr) & \le \PP\left(\estimatorn{\sigma}^2(\xv) \le \lambda \left[ 1 - \frac{C z_{1 - \beta}}{\sqrt{n h^d \estimatorn{p}(\xv)}} \right] \right) \\
    & \le \PP \left (\estimatorn{\sigma}^2(\xv) - \sigma^2(\xv) \le \Delta(\xv) - \frac{C \lambda z_{1 - \beta}}{\sqrt{n h^d \estimatorn{p}(\xv)}}  \right ) = : \mathbf{P}(\xv)
  \end{align*}

  The whole set $\{\xv \mid \sigma^2(\xv) \le \lambda \}$ can be divided into two sets. Roughly speaking, one is $\mathcal{A} = \{\xv \mid \Delta(\xv) \lesssim (n h^d)^{-1/2} \}$ and the other is $\mathcal{B} = \{ \xv \mid \Delta(\xv) \gg (n h^d)^{-1/2} \}$. 
  While $\xv \in \mathcal{A}$, the leading term of the excess risk is $\Delta(\xv) \cdot \mathbf{P}(\xv)$ that has order $(n h^d)^{-1/2}$. The factor $\mathbf{P}(\xv)$ does not go to zero, since, informally, $\sqrt{n h^d} \bigl(\estimatorn{\sigma}^2(\xv) - \sigma^2(\xv)\bigr) \approx \normDistribution{0}{\sigma_V^2}$ due to~\eqref{eq: normal limit} and so the difference between $\lambda$ and $\sigma^2(\xv)$ can not be captured by $\estimatorn{\sigma}^2(\xv)$. While this argument is not strict from the theoretical point of view, one may prove anticoncentration bounds via the Carbery-Wright theorem. On Figure~\ref{fig:excess_risk_n}, one may observe sets $\mathcal{A}$ for different $n$ as hills on the left of the point where $\sigma^2(\xv) = \lambda$.
  
  But if $\xv \in \mathcal{B}$, bias and variance suppress term $\Delta(\xv) \cdot \mathbf{P}(\xv)$ and we obtain usual rates of convergence for nonparametric estimators. Since for small $n$ the set $\mathcal{A}$ is large there maybe some warm-up when we see slower rates of convergence on plots. So in each point, the convergence may have two phases: one is when $\xv \in \mathcal{A}$ and the other is when $\xv \in \mathcal{B}$. That is how we explain two phases on Figure~\ref{fig:risk_nh} for $\xv \in \{-1.6, -0.5, 0.3\}$.

  We summarize the behaviour of our estimator and estimator proposed by~\cite{denis_regression_2020} in Table~\ref{tab: rates of excess risk}.

  \begin{table}[t!]
      \begin{tabular}{|c|c|c|c|}
            \hline
            \footnotesize $\sigma^2(\xv) \ge \lambda$ & \footnotesize $\Delta(\xv) < C_1 h^2 / p(\xv)$& \footnotesize $ C_2 h > \Delta(\xv) > C_3 h^2 / p(\xv)$ & \footnotesize $\Delta(\xv) \gg h$  \\
            \hline
           \footnotesize \vtop{\hbox{\strut testing-}\hbox{\strut based}}
           & \footnotesize \vtop{\hbox{\strut $O (h^2) / p(\xv) \cdot \exp\bigl\{-\Omega\bigl(n h^{d + 2} / p(\xv)\bigr)\bigr\}$}}
            &    \footnotesize \footnotesize \multirow{2}{*}[-10pt]{$O (h) \cdot \exp\bigl\{-\Omega\bigl(n h^{d + 2} p(\xv)\bigr)\bigr\}$} & \footnotesize \multirow{2}{*}[-10pt]{$\exp\bigl\{-\Omega\bigl(n h^d p(\xv)\bigr)\bigr\}$}  
            \\
           \cline{1-1}\cline{2-2}
           \footnotesize plugin &  \footnotesize \vtop{\hbox{\strut $O \left ( h^2 / p(\xv) \right)$} }
           & & \\
           \hline
       \end{tabular}

        \vspace{5pt}

        \begin{tabular}{|c|c|c|}
          \hline
           \footnotesize $\sigma^2(\xv) < \lambda$ & \footnotesize $\Delta(\xv) \lesssim \bigl(n h^d p(\xv)\bigr)^{-1/2}$ 
           & \footnotesize $\Delta(\xv) \gg \bigl(n h^d p(\xv)\bigr)^{-1/2}$ \\
           \hline
           \footnotesize testing-based
           &  \footnotesize \multirow{2}{*}{$O\bigl(\{n h^d p(\xv)\}^{-1/2}\bigr)$}
           & \footnotesize \multirow{2}{*}{$O(\{n h^d p(\xv)\}^{-1}) + O\bigl(h^4 / p^2(\xv)\bigr)$} \\
           \cline{1-1}
           \footnotesize plugin & & \\
           \hline
      \end{tabular}
      \caption{The upper bounds derived in Theorem~\ref{theorem: informal thm}, the case of undersmoothing.}
      \label{tab: rates of excess risk}
  \end{table}

\subsection{Sketch of the Proof}
  We start by the following bound on the kernel values:
  \begin{align*}
    a \indicator{X_i \in \ball_{b h}(\xv)} \le K\left(\frac{X_i - \xv}{h}\right),
  \end{align*}
  where $\ball_r(\xv)$ is a ball with radius $r$ and center $\xv$. These bounds allow us to deal with values of the kernel like they are Bernoulli random variable with certain mean. Thus, we show that with probability $1 - C \exp^{- \Omega(n h^d p(\xv))}$, we have
  \begin{align*}
    a b^d \omega_d p(\xv) \le \estimatorn{p}(\xv) \le 2 p(\xv) + L_p h \int_{\RR^d} \Vert \tv \Vert K(\tv) d \mu(\tv),
  \end{align*}
  see Propositions~\ref{proposition: denominator lower bound} and~\ref{proposition: density upper bound} in Supplementary Material.
  The bounds above are rough but they will be sufficient for our purposes.

  For any $L$-Lipschitz function $g$ we also can obtain the bound 
  \begin{align*}
    \left|\sum_{i = 1}^n g(X_i) \weight(X_i) - g(\xv) \right | \le \sum_{i = 1}^n |g(X_i) - g(\xv)| \weight(X_i) \lesssim \frac{L h}{p(\xv) n h^d},
  \end{align*}
  since $\weight(X_i)$ is, roughly speaking, $\frac{R_K e^{-r_K \Vert (X_i - \xv) / h \Vert}}{p(\xv) n h^d}$ up to a constant, see Corollary~\ref{corollary: weights bound} in Supplementary Material. This approximation is based on the fact that $K(\tv)$ is bounded above by a constant and the denominator of the weight with high probability is $\Omega\bigl(n h^d p(\xv)\bigr)$, see Proposition~\ref{proposition: denominator lower bound} in Supplementary Material. Finally, under some conditions on $n$, $h$ and $p(\xv)$ we establish the concentration bounds for any function $g$ which is Lipschitz and has bounded Hessian, see Corollary~\ref{corollary: nhp concentration bound} in Supplementary Material.
  
  If $\sigma^2(\xv) \ge \lambda$ then $\indicator{\estimator{\alpha}(\xv) = 0} = \indicator{\estimator{\alpha}(\xv) \neq \alpha(\xv)}$. So we may bound
  \begin{equation*}
    \EE_{\mathcal{D}} \bigl(\estimator{f}(\xv) - \meanY(\xv)\bigr)^2 \indicator{\estimator{\alpha}(\xv) = 0}
    \le \sqrt{\EE_{\mathcal{D}} \bigl(\estimator{f}(\xv) - \meanY(\xv)\bigr)^4 \indicator{\estimator{\alpha}(\xv) = 0} } \cdot \PP^{1/2}\bigl(\estimator{\alpha}(\xv) \neq \alpha(\xv)\bigr).
  \end{equation*}
  We bound the 4-th moment above by integrating concentration inequalities. That results in standard bias-variance trade-off, see Lemma~\ref{lemma: lp bias variance tradeoff} in Supplementary Material. Thus, the rate of the excess risk is determined by the factor $\PP^{1/2}(\estimatorn{\alpha}\bigl(\xv) \neq \alpha(\xv)\bigr)$. It can be reformulated as 
  \[
    \PP^{1/2}\left(|\estimatorn{\sigma}^2(\xv) - \sigma^2(\xv)| \ge \Delta(\xv) + O \left \{ (n h^d p(\xv))^{-1/2} \right \}\right).
  \]
  The random value $\estimatorn{\sigma}^2(\xv)$ behaves like sub-exponential random variable with the mean $\sigma^2(\xv) + o(1)$. Thus, under certain assumptions on $\Delta(\xv)$, we get exponential rates of convergence via the concentration argument, see Corollary~\ref{corollary: variance deviation bound} in Supplementary Material.

  If $\sigma^2(\xv) < \lambda$, two terms from Proposition~\ref{proposition: excess risk decomposition} demonstrate different behaviour. The first one can be bounded via standard bias-variance trade-off. The second one exponentially decreases if $h$ is smaller than some constant and $\Delta(\xv)$ is larger than some decreasing function of $n$ and $h$. The proof is similar to the case $\sigma^2(\xv) > \lambda$.

\section{Experiments}
\label{sec:experiments}

\subsection{How to choose $\lambda$ and $\beta$}
  In practice, two natural questions arise: how to choose $\lambda$ and how to choose $\beta$. Obviously, one may define $\lambda$ from the formulation of the problem as an inappropriate level of noise. The case of $\beta$ is a bit more sophisticated. From Algorithm~\ref{algo: test}, we infer that any $\xv$ will be rejected if $\estimatorn{p}(\xv) \le \frac{2 \Vert K \Vert_2^2 z_{1 - \beta}^2}{n h^d}$. Thus, any such $\xv$ is considered as outlier, and, hence, $z_{1 - \beta}$ is a tolerance level for outliers. Additionally, the choice of $\beta$ determines the trade-off between type I and type II errors. 

  \begin{figure}[t!]
    \subfigure[Acceptance probability.]{%
      \label{fig:acc_prob}
      \includegraphics[width=.46\linewidth]{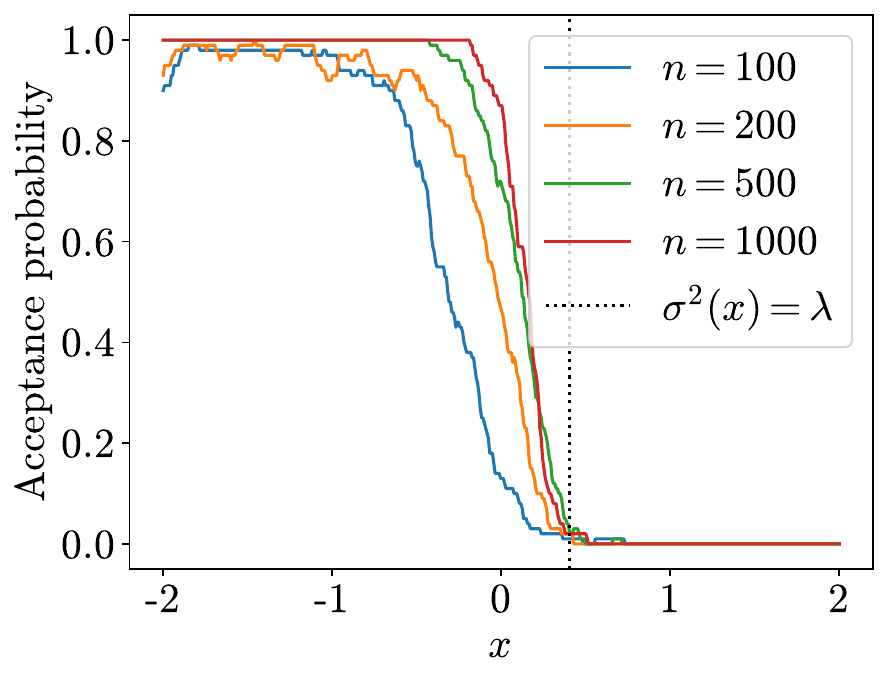}
    }\qquad
    \subfigure[Expected excess risk.]{%
      \label{fig:excess_risk_n}
      \includegraphics[width=.46\linewidth]{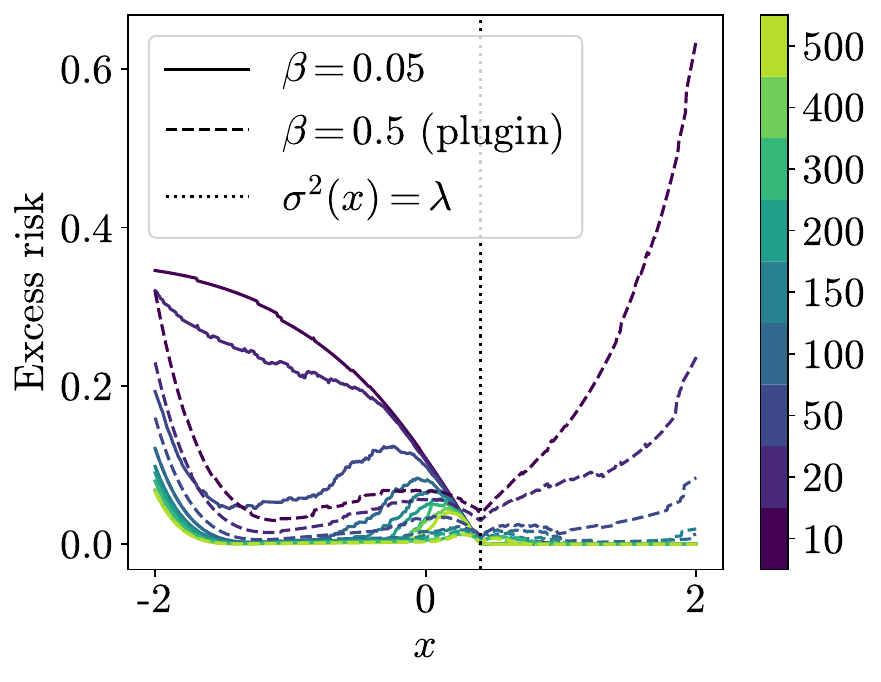}
    }
    \caption{Example with synthetic data: $X \sim \mathcal{U}(-2,2), \: \sigma(x) = \mathtt{sigmoid}(x)$. We sample multiple datasets of each sample size $n$. Confidence level $\beta=0.05$ and abstention cost $\lambda = 0.36$.}
  \end{figure}

  \begin{figure}
    \subfigure[Acceptance probability.]{%
      \label{fig:acceptance_beta}
      \includegraphics[width=.46\linewidth]{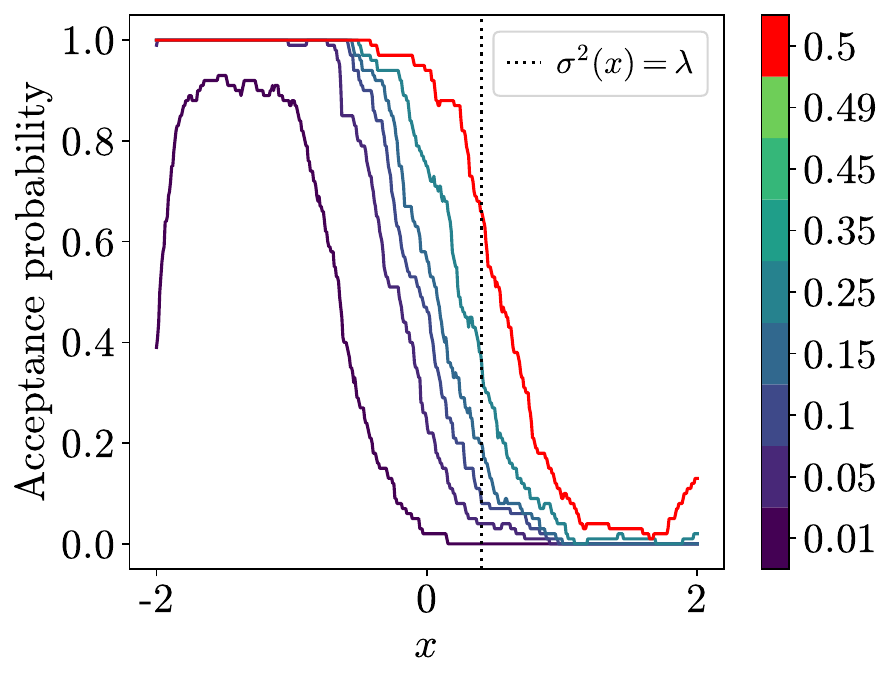}
    }\qquad
    \subfigure[Expected excess risk.]{%
      \label{fig:excess_risk_beta}
      \includegraphics[width=.46\linewidth]{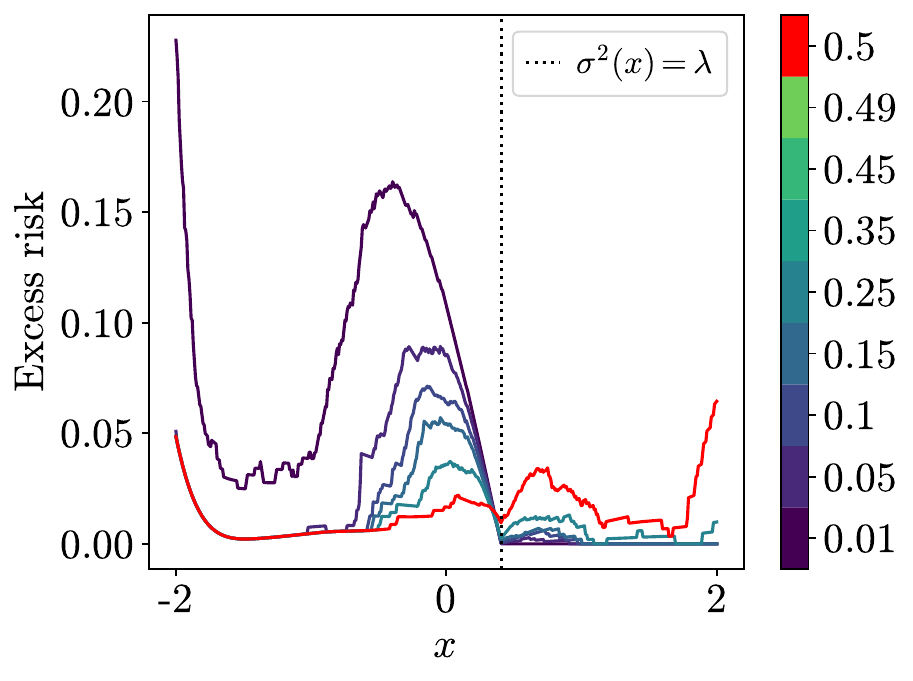}
    }
    \caption{Example with synthetic data: $X \sim \mathcal{U}(-2,2), \: \sigma(x) = \mathtt{sigmoid}(x)$. Sample size $n = 100$, abstention cost $\lambda = 0.36$. We apply the proposed testing-based method for different values of $\beta$. Plugin approach corresponds to $\beta=0.5$.}
    \label{fig:acc_vs_x_main}
  \end{figure}

\subsection{Synthetic data}
  For the first part of the experiments we use one dimensional data with known simple functions as true mean and variance at each point:
  \begin{equation}
  \label{data_model}
    Y = f(X) + \sigma(X) \varepsilon, \; X \sim p(\cdot), \: \varepsilon \sim \mathcal{N}(0, 1).
  \end{equation}
  Specifically, we consider normal and uniform distributions of the independent variable $p(\cdot) \in \{\mathcal{N}(0, 1), \mathcal{U}(-2, 2)\}$, a fixed mean function $f(x) = \frac{x^2}{4}$, and two choices of standard deviation: sigmoid function and Heaviside function. Parameter $\lambda$ was fixed at $0.36$ and parameter $\beta = 0.05$ unless otherwise noted. Optimal bandwidth was selected using leave-one-out cross-validation optimizing mean squared error of prediction by NW estimator. In all our experiments for each setting of hyperparameters we have generated 100 different random datasets from our data model and then averaged the results.

\subsubsection{Convergence of estimates}
  We sampled 100 datasets of sizes $n \in \{100, 200, 500, 1000\}$ and for each $x \in [-2, 2]$ we estimate the fraction of predictions that are accepted by the proposed method.  We present the resulting chart in Figure~\ref{fig:acc_prob} for $X \sim \mathcal{U}(-2, 2), \: \sigma(x) = \mathtt{sigmoid}(x)$, additional charts are in Supplementary Material, Section~\ref{sec:suppl_acc_prob}. The results demonstrate that for the area with $\sigma^2(x) > \lambda$ (to the right of the dashed line) the convergence is much faster than for the area with $\sigma^2(x) < \lambda$.

  Additionally, we also estimate expected excess risk, since we know the values of $f(x)$ and $\sigma(x)$ for any $x$. For the first plot (see Figure~\ref{fig:excess_risk_n}) we vary sample size $n$ from $10$ to $500$. We compare the proposed approach with $\beta = 0.05$ with ``plugin'' baseline, corresponding to $\beta = 0.5$. 
  For the testing-based method we see the very quick convergence for $\sigma^2(x) > \lambda$. For the points with $\sigma^2(x) < \lambda$ the convergence is slower especially for the smaller values of $\Delta(x)$. Thus, the observed behaviour well corresponds to the one predicted by the theory. For the plugin approach, the convergence is slower especially for the points with $\sigma^2(x) > \lambda$. 

\subsubsection{Dependence on $\beta$}
  In this experiment, we have studied the behavior of our method when changing its only hyperparameter $\beta$ in the range between $0.01$ and $0.5$. Since $\beta = 0.5$ corresponds to ``plugin'' method described previously, we show it in red. For this we fixed the number of samples at $n = 100$, sampled 100 datasets and calculated the expected excess risk for each $x \in [-2, 2]$. With the increase of $\beta$ the method becomes less conservative (more accepts), see Figure~\ref{fig:acceptance_beta}. It leads to the increased expected risk at points where prediction should be rejected and decreased risk at the points where predictions should be accepted, see Figure~\ref{fig:excess_risk_beta}. Thus, in practice parameter $\beta$ might be selected based on the trade-off between these two errors depending on the particular features of the considered applied problem.

\subsubsection{Pointwise convergence}
  \begin{figure}[t!]  
    \subfigure[True mean and standard deviation, and their estimates for $n = 100$ points.]{%
      \label{fig:synthetic_points}
      \includegraphics[width=.32\linewidth]{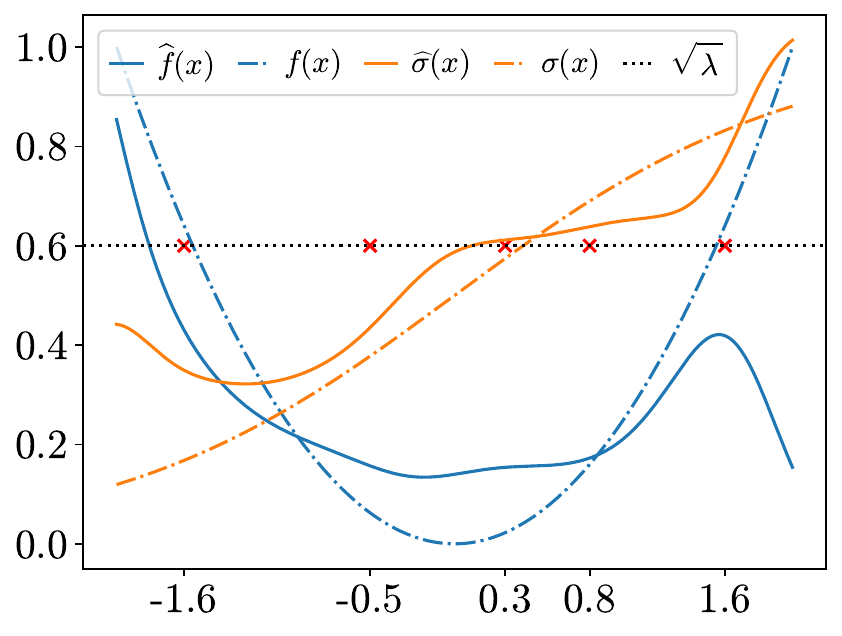}
    }
    \subfigure[Expected excess risk.]{%
      \label{fig:risk_nh}
      \includegraphics[width=.32\linewidth]{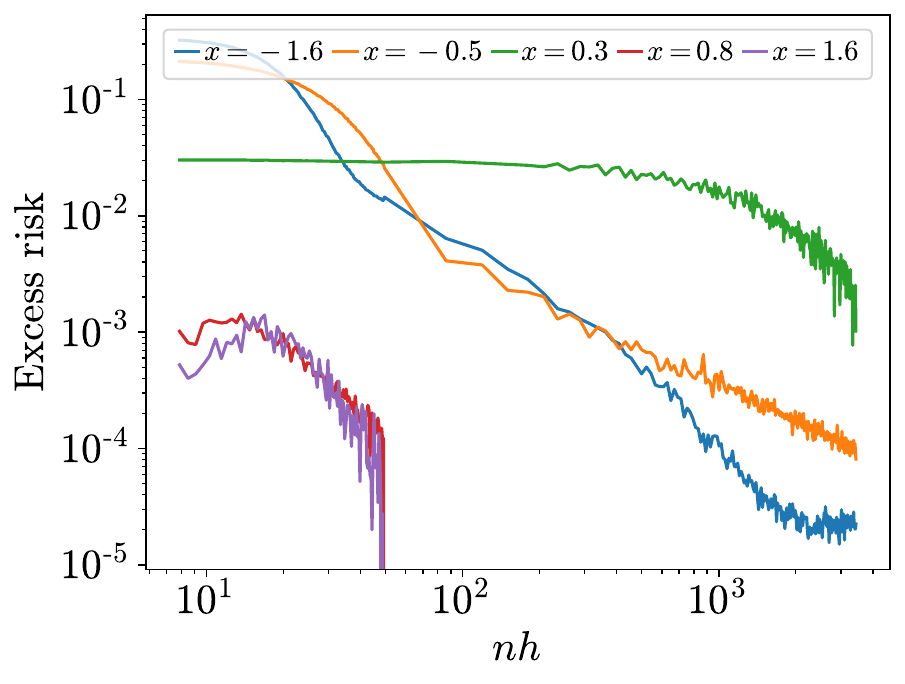}
    }
    \subfigure[Comparison with plugin method at a single point.]{%
      \label{fig:comparison_plugin}
      \includegraphics[width=.32\linewidth]{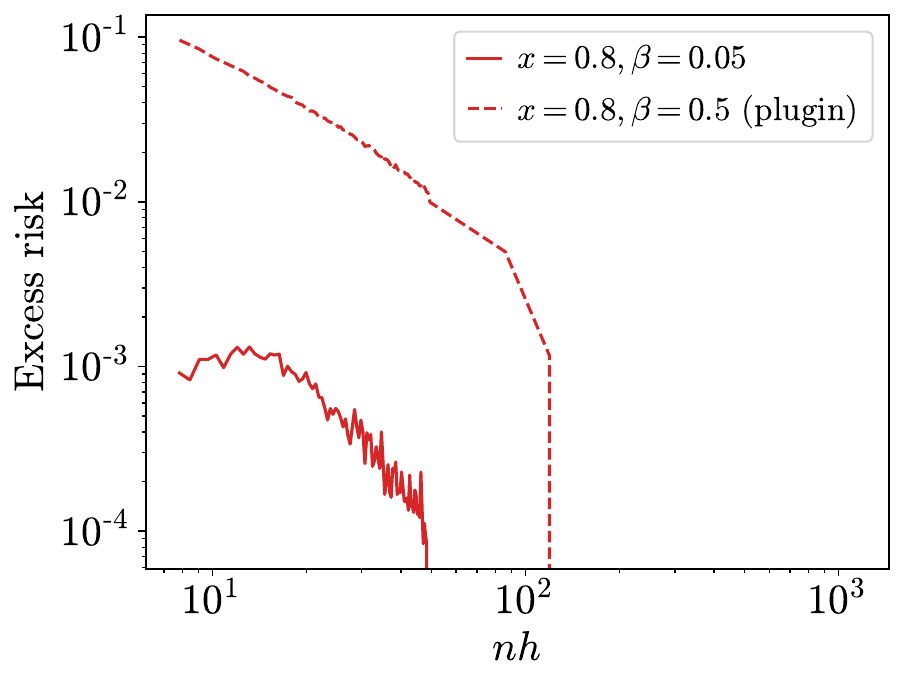}
    }
    \caption{We sample multiple datasets of each size and for selected points $x \in \{-1.6, -0.5, 0.3, 0.8, 1.6\}$ calculate the expected excess risk. In this experiment $X \sim \mathcal{U}(-2,2), \: \sigma(x) = \mathtt{sigmoid}(x)$, abstention cost $\lambda = 0.36$, $\beta = 0.05$.}
  \end{figure}

  Finally, we sampled multiple datasets of increasing sizes $n$ from $10$ to $20000$ and selected 5 diagnostic points: $x \in \{-1.6, -0.5, 0.3, 0.8, 1.6\}$, see Figure~\ref{fig:synthetic_points}. When sample size is less than $100$, we generate $20000$ datasets of each size, while for larger sample size we only use $100$. In order to perform a more straightforward averaging of the across datasets of the same size, we have used the same bandwidth $h \sim \frac{1}{n^5}$ that was selected to show the expected polynomial dependence in $n h$ of the risk at points $x$ with $\sigma^2(x) < \lambda$. The resulting dependencies of the risk on $n h$ are depicted on Figure~\ref{fig:risk_nh}. We observe all the main outcomes predicted by the theory:
  \begin{itemize}
    \item rapid convergence of the risk for the points with $\sigma^2(x) > \lambda$ (points $x = 0.8$ and $x = 1.6$);

    \item polynomial convergence of the risk as function of $n h$ for $\sigma^2(x) < \lambda$ with moderately large values of $\Delta(x)$, i.e., points $x = -1.6$ and $x = -0.5$;

    \item very slow convergence for the point with $\sigma^2(x) < \lambda$ and small value of $\Delta(x)$.
  \end{itemize}
  Additionally, on Figure~\ref{fig:comparison_plugin} we experimentally confirm that plugin method has slower convergence than testing-based method for $\sigma^2(x) > \lambda$.

\subsection{Airfoil Self-Noise Data Set}

  \begin{figure}[t!]
    \subfigure[Fraction of accepted points for each threshold $\lambda$.]{%
      \label{fig:acc_prob_airfoil}
      \includegraphics[width=.45\linewidth]{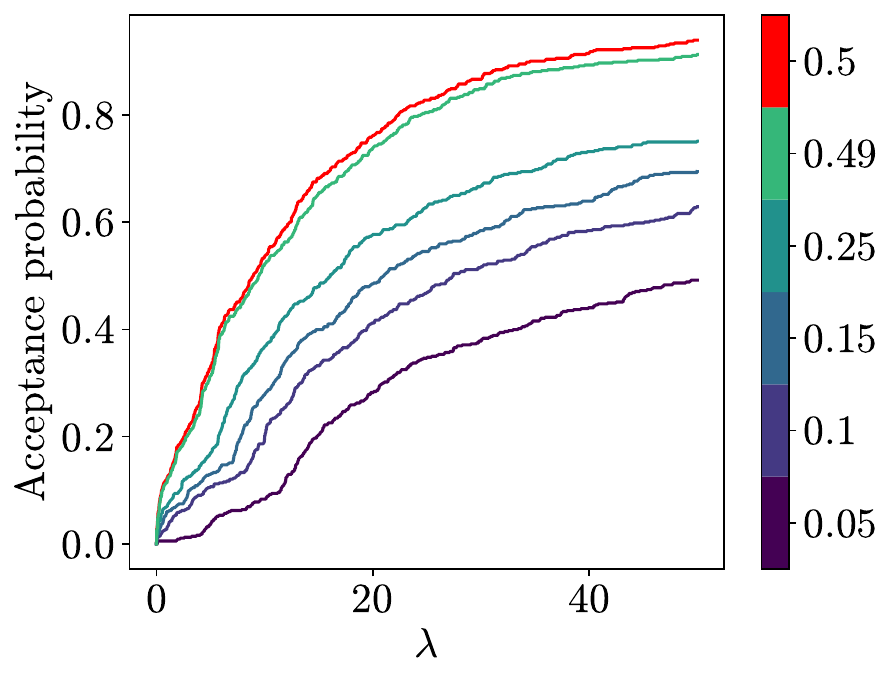}
    }\qquad 
    \subfigure[MSE with rejection.]
    {%
      \label{fig:mse_airfoil}
      \includegraphics[width=.45\linewidth]{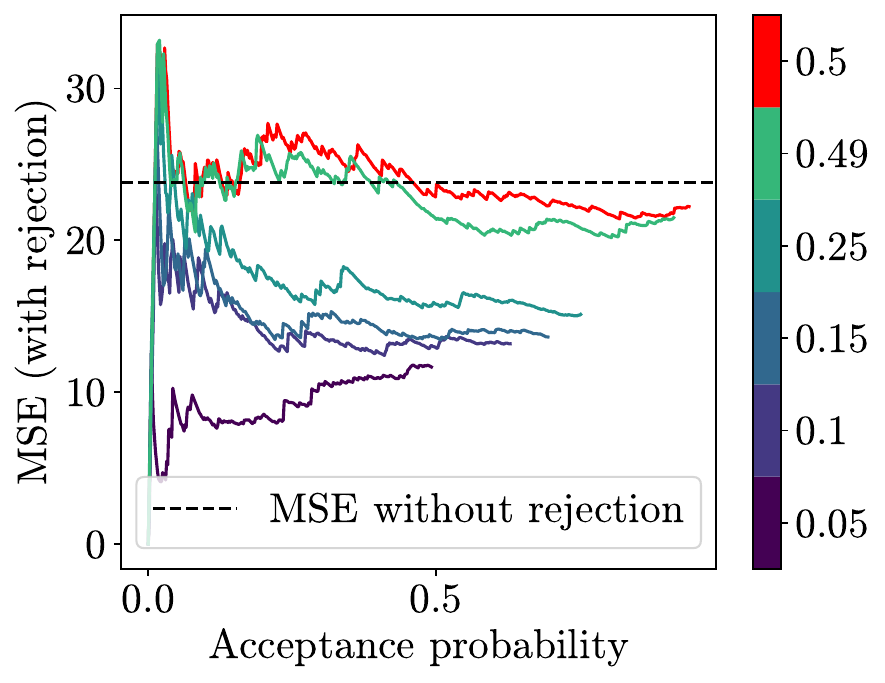}
    }
    \caption{Airfoil data, split 70/30 by the second feature. For $\lambda \in [0, 50]$ we calculate acceptance (retention) probability and MSE at accepted points. $x$-values of acceptance probabilities are inferred as fraction of accepted points for each $\lambda$.}
  \end{figure}


  We have tested our method on the Airfoil dataset from the UCI collection~\citep{Dua2019}. We do not perform any special preprocessing of the data or feature engineering, only standard scaling of features. We prepare train and test sets in two steps. First, we select a pivot feature and put $70\%$ of the data with the lowest values of this feature to part A and the rest of the data becomes part B. For the second step we select $20\%$ of each part (sampled uniformly) and put it in the other part. First part becomes the train set and the second part the test set. In this way we guarantee that test set will have data with low values of $\widehat{p}(\xv)$ as well as data distributed similar to train data.

  In our experiments we select different features as pivots for the split and then vary $\lambda \in [0, 50]$, calculating acceptance (retention) fraction and mean squared error. We present results for splitting by the second feature: ``Angle of attack''. Other configurations can be found in the Supplementary Material, Section~\ref{sec:suppl_airfoil}. On Figure~\ref{fig:acc_prob_airfoil} we show how acceptance probability varies as a function of $\lambda$. Figure~\ref{fig:mse_airfoil} illustrates the dependence of the mean squared error of estimation as a function of the fraction of points accepted for prediction. The curves show the expected trend to increase when accepting more points. Using more conservative estimates one obtains higher accuracy for the given acceptance rate. However, by construction, high acceptance rates are not achievable for the proposed method due to the limitations on the values of the estimated density $\estimatorn{p}(\xv)$.

  \subsection{CPU-small Data Set}

    \begin{figure}[t!]
    \label{fig:mse_cpu}
    \subfigure[Raw data.]{%
      \label{fig:mse_cpu_raw}
      \includegraphics[width=.45\linewidth]{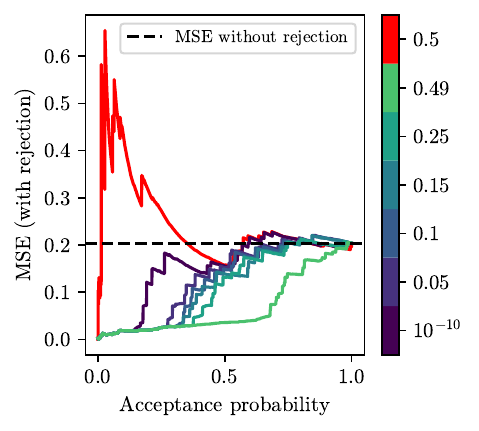}
    }\qquad 
    \subfigure[Embeddings.]
    {%
      \label{fig:mse_cpu_emb}
      \includegraphics[width=.45\linewidth]{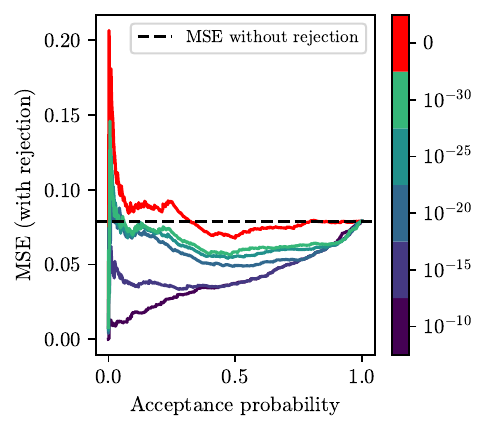}
    }
    \caption{CPU-small data, split 70/30 by the first feature. For $\lambda \in [0, 2]$ we calculate acceptance (retention) probability and MSE at accepted points. For the raw data we vary $\beta$, in case of embeddings we vary $z_{1-\beta}$ directly due to higher dimensionality of the data.} 
  \end{figure}

  Another dataset from UCI collection that we used is CPU-small. This dataset has 8192 instances and 12 features. Data splitting is done in the same manner as the Airfoil dataset. During the preprocessing we standardize the training data to have zero mean and unit variance and then apply the same scaling to the test set. Splitting was done based on the first feature ``lread''.
  
  In this experiment we have tried two scenarios: first is to use the data as it is and the second is to use higher dimensional version (embeddings) of the data, obtained with a neural network. First we trained a two layer neural network with 50 neurons in each layer and ReLU activations and then used the values from the last layer as input for our method. On Figures~\ref{fig:mse_cpu_raw} and~\ref{fig:mse_cpu_emb} we present the partial MSE scores obtained with rejection in these two setups. 
  
  Using embeddings provides much lower MSE without rejection as one would expect. It also shows that using our method we can significantly outperform the baseline plug-in method. For this dataset our algorithm is less sensitive to the choice of $\beta$ than for the previous one. We opted to vary $z_{1-\beta}$ directly in the embedding case since for a higher dimension of the data the values of $\widehat{p}$ span a larger strip in the logarithmic scale. In order to show dependence on $\beta$ we would need to choose values very close to $0.5$. Choosing the same set of $\beta$ values as for raw data case yields curves similar to $z_{1-\beta}=10^{-10}$.


\section{Conclusion}
\label{sec:conclusion}
  In this work, propose a new method for selective prediction in heteroskedastic regression tasks under the Chow risk model. The method is based on the natural idea of testing the values of conditional variance at a given point. Our theoretical analysis show the existence of exponential and polynomial convergence regimes that depend on the relative values of the variance and abstention cost. The proposed method compares favorably to the plugin baseline both in theory and in the conducted experimental valuation.

\acks{The research was supported by the Russian Science Foundation grant 21-11-00373. The work of F. Noskov was prepared within the framework of the HSE University Basic Research Program.}


\bibliography{main}

\begin{thebibliography}{25}
\providecommand{\natexlab}[1]{#1}
\providecommand{\url}[1]{\texttt{#1}}
\expandafter\ifx\csname urlstyle\endcsname\relax
  \providecommand{\doi}[1]{doi: #1}\else
  \providecommand{\doi}{doi: \begingroup \urlstyle{rm}\Url}\fi

\bibitem[Bartlett and Wegkamp(2008)]{bartlett_classification_2008}
Peter~L. Bartlett and Marten~H. Wegkamp.
\newblock Classification with a {Reject} {Option} using a {Hinge} {Loss}.
\newblock \emph{Journal of Machine Learning Research}, 9\penalty0
  (59):\penalty0 1823--1840, 2008.
\newblock URL \url{http://jmlr.org/papers/v9/bartlett08a.html}.

\bibitem[Chow(1970)]{chow_optimum_1970}
C.~Chow.
\newblock On optimum recognition error and reject tradeoff.
\newblock \emph{IEEE Transactions on Information Theory}, 16\penalty0
  (1):\penalty0 41--46, January 1970.
\newblock Conference Name: IEEE Transactions on Information Theory.

\bibitem[Chow(1957)]{chow_optimum_1957}
C.~K. Chow.
\newblock An optimum character recognition system using decision functions.
\newblock \emph{IRE Transactions on Electronic Computers}, EC-6\penalty0
  (4):\penalty0 247--254, December 1957.
\newblock Conference Name: IRE Transactions on Electronic Computers.

\bibitem[Cortes et~al.(2016)Cortes, DeSalvo, and Mohri]{cortes_learning_2016}
Corinna Cortes, Giulia DeSalvo, and M.~Mohri.
\newblock Learning with {Rejection}.
\newblock In \emph{{ALT}}, 2016.

\bibitem[Denis and Hebiri(2020)]{denis2020consistency}
Christophe Denis and Mohamed Hebiri.
\newblock Consistency of plug-in confidence sets for classification in
  semi-supervised learning.
\newblock \emph{Journal of Nonparametric Statistics}, 32\penalty0 (1):\penalty0
  42--72, 2020.

\bibitem[Dua and Graff(2017)]{Dua2019}
Dheeru Dua and Casey Graff.
\newblock {UCI} machine learning repository, 2017.
\newblock URL \url{http://archive.ics.uci.edu/ml}.

\bibitem[El-Yaniv and Wiener(2010)]{el-yaniv_foundations_2010}
Ran El-Yaniv and Yair Wiener.
\newblock On the {Foundations} of {Noise}-free {Selective} {Classification}.
\newblock \emph{Journal of Machine Learning Research}, 11\penalty0
  (53):\penalty0 1605--1641, 2010.
\newblock URL \url{http://jmlr.org/papers/v11/el-yaniv10a.html}.

\bibitem[Fan and Yao(1998)]{fan_efficient_1998}
Jianqing Fan and Qiwei Yao.
\newblock Efficient {Estimation} of {Conditional} {Variance} {Functions} in
  {Stochastic} {Regression}.
\newblock \emph{Biometrika}, 85\penalty0 (3):\penalty0 645--660, 1998.
\newblock URL \url{http://www.jstor.org/stable/2337393}.
\newblock Publisher: [Oxford University Press, Biometrika Trust].

\bibitem[Geifman and El-Yaniv(2019)]{geifman_selectivenet_2019}
Yonatan Geifman and Ran El-Yaniv.
\newblock {SelectiveNet}: {A} {Deep} {Neural} {Network} with an {Integrated}
  {Reject} {Option}.
\newblock In \emph{Proceedings of the 36th {International} {Conference} on
  {Machine} {Learning}}, pages 2151--2159. PMLR, May 2019.
\newblock URL \url{https://proceedings.mlr.press/v97/geifman19a.html}.

\bibitem[Grandvalet et~al.(2009)Grandvalet, Rakotomamonjy, Keshet, and
  Canu]{grandvalet_support_2009}
Yves Grandvalet, Alain Rakotomamonjy, Joseph Keshet, and Stéphane Canu.
\newblock Support {Vector} {Machines} with a {Reject} {Option}.
\newblock In \emph{Advances in {Neural} {Information} {Processing} {Systems}},
  volume~21. Curran Associates, Inc., 2009.
\newblock URL
  \url{https://papers.nips.cc/paper/2008/hash/3df1d4b96d8976ff5986393e8767f5b2-Abstract.html}.

\bibitem[Herbei and Wegkamp(2006)]{herbei_classification_2006}
Radu Herbei and Marten~H. Wegkamp.
\newblock Classification with {Reject} {Option}.
\newblock \emph{The Canadian Journal of Statistics / La Revue Canadienne de
  Statistique}, 34\penalty0 (4):\penalty0 709--721, 2006.
\newblock URL \url{https://www.jstor.org/stable/20445230}.
\newblock Publisher: [Statistical Society of Canada, Wiley].

\bibitem[Jiang et~al.(2020)Jiang, Zhao, and Wang]{jiang_risk-controlled_2020}
Wenming Jiang, Ying Zhao, and Zehan Wang.
\newblock Risk-controlled selective prediction for regression deep neural
  network models.
\newblock In \emph{2020 International Joint Conference on Neural Networks
  ({IJCNN})}, pages 1--8, 2020.

\bibitem[Johansson et~al.(2023)Johansson, L{\"o}fstr{\"o}m, S{\"o}nstr{\"o}d,
  and L{\"o}fstr{\"o}m]{johansson2023conformal}
Ulf Johansson, Tuwe L{\"o}fstr{\"o}m, Cecilia S{\"o}nstr{\"o}d, and Helena
  L{\"o}fstr{\"o}m.
\newblock Conformal prediction for accuracy guarantees in classification with
  reject option.
\newblock In \emph{International Conference on Modeling Decisions for
  Artificial Intelligence}, pages 133--145. Springer, 2023.

\bibitem[Lei(2014)]{lei_classification_2014}
Jing Lei.
\newblock Classification with confidence.
\newblock \emph{Biometrika}, 101\penalty0 (4):\penalty0 755--769, 2014.
\newblock URL \url{https://www.jstor.org/stable/43304686}.
\newblock Publisher: [Oxford University Press, Biometrika Trust].

\bibitem[Linusson et~al.(2018)Linusson, Johansson, Bostr{\"o}m, and
  L{\"o}fstr{\"o}m]{linusson2018classification}
Henrik Linusson, Ulf Johansson, Henrik Bostr{\"o}m, and Tuve L{\"o}fstr{\"o}m.
\newblock Classification with reject option using conformal prediction.
\newblock In \emph{Advances in Knowledge Discovery and Data Mining: 22nd
  Pacific-Asia Conference, PAKDD 2018, Melbourne, VIC, Australia, June 3-6,
  2018, Proceedings, Part I 22}, pages 94--105. Springer, 2018.

\bibitem[Nadeem et~al.(2009)Nadeem, Zucker, and
  Hanczar]{nadeem_accuracy-rejection_2009}
Malik Sajjad~Ahmed Nadeem, Jean-Daniel Zucker, and Blaise Hanczar.
\newblock Accuracy-{Rejection} {Curves} ({ARCs}) for {Comparing}
  {Classification} {Methods} with a {Reject} {Option}.
\newblock In \emph{Proceedings of the third {International} {Workshop} on
  {Machine} {Learning} in {Systems} {Biology}}, pages 65--81. PMLR, March 2009.
\newblock URL \url{https://proceedings.mlr.press/v8/nadeem10a.html}.

\bibitem[Neu and Zhivotovskiy(2020)]{neu2020fast}
Gergely Neu and Nikita Zhivotovskiy.
\newblock Fast rates for online prediction with abstention.
\newblock In \emph{Conference on Learning Theory}, pages 3030--3048. PMLR,
  2020.

\bibitem[Puchkin and Zhivotovskiy(2021)]{puchkin2021exponential}
Nikita Puchkin and Nikita Zhivotovskiy.
\newblock Exponential savings in agnostic active learning through abstention.
\newblock In \emph{Conference on Learning Theory}, pages 3806--3832. PMLR,
  2021.

\bibitem[Salem et~al.(2022)Salem, Ahmed, Tung, and
  Oliveira]{salem_gumbel-softmax_2022}
Mahmoud Salem, Mohamed~Osama Ahmed, Frederick Tung, and Gabriel Oliveira.
\newblock Gumbel-softmax selective networks, 2022.
\newblock URL \url{http://arxiv.org/abs/2211.10564}.

\bibitem[Shafer and Vovk(2008)]{shafer2008tutorial}
Glenn Shafer and Vladimir Vovk.
\newblock A tutorial on conformal prediction.
\newblock \emph{Journal of Machine Learning Research}, 9\penalty0 (3), 2008.

\bibitem[Shah et~al.(2022)Shah, Bu, Lee, Das, Panda, Sattigeri, and
  Wornell]{shah_selective_2022}
Abhin Shah, Yuheng Bu, Joshua~K. Lee, Subhro Das, Rameswar Panda, Prasanna
  Sattigeri, and Gregory~W. Wornell.
\newblock Selective regression under fairness criteria.
\newblock In \emph{Proceedings of the 39th International Conference on Machine
  Learning}, pages 19598--19615. {PMLR}, 2022.
\newblock URL \url{https://proceedings.mlr.press/v162/shah22a.html}.

\bibitem[Tsybakov(2009)]{tsybakov_introduction_2009}
Alexandre~B. Tsybakov.
\newblock \emph{Introduction to Nonparametric Estimation}.
\newblock Springer Series in Statistics. Springer, 2009.
\newblock URL \url{http://link.springer.com/10.1007/b13794}.

\bibitem[Vovk et~al.(1999)Vovk, Gammerman, and Saunders]{vovk1999machine}
Volodya Vovk, Alexander Gammerman, and Craig Saunders.
\newblock Machine-learning applications of algorithmic randomness.
\newblock In \emph{Proceedings of the Sixteenth International Conference on
  Machine Learning}, pages 444--453, 1999.

\bibitem[Wainwright(2019)]{wainwright_high-dimensional_2019}
Martin~J. Wainwright.
\newblock \emph{High-Dimensional Statistics: A Non-Asymptotic Viewpoint}.
\newblock Cambridge Series in Statistical and Probabilistic Mathematics.
  Cambridge University Press, 2019.
\newblock URL
  \url{https://www.cambridge.org/core/books/highdimensional-statistics/8A91ECEEC38F46DAB53E9FF8757C7A4E}.

\bibitem[Zaoui et~al.(2020)Zaoui, Denis, and Hebiri]{denis_regression_2020}
Ahmed Zaoui, Christophe Denis, and Mohamed Hebiri.
\newblock Regression with reject option and application to knn.
\newblock \emph{Advances in Neural Information Processing Systems},
  33:\penalty0 20073--20082, 2020.

\end{thebibliography}

\newpage

\appendix

\section{Full Statement of Main Theorem}
\label{sec:statement}
  We suppress some terms and constants in the statement of Theorem~\ref{theorem: informal thm}, so in this section we provide the full statement of our theorem. 

  \begin{theorem}
  \label{theorem: finite-sample thm}
    Suppose that Assumptions~\ref{assumption: mean assumption}-\ref{assumption: density assumption} hold. Consider some $\xv \in \RR^d$. Assume 
    \begin{align*}
        p(\xv) \ge \max \left \{ 2 L_p b h, 8/ (\omega_d b^d n h^d), \frac{4 \pi^{d/2} R_K}{\Gamma(d/2) r_K^2} \cdot L_p h \right \}
    \end{align*}
    and $d(\xv, \partial \support) \ge b h$.
    Define $\Delta(\xv) = |\sigma^2(\xv) - \lambda|$. Let $\risk(\xv)$ be the Chow risk of Algorithm~\ref{algo: test} and $\omega_d$ be volume of the unit ball in $\RR^d$. If $\sigma^2(\xv) > \lambda$ and
    \begin{align*}
        \Delta(\xv) & \ge \frac{20 R_K \sigma^2(\xv)}{a b^d \omega_d n h^d p(\xv)} +5 \constCorTwo^\sigma \frac{h^2}{p(\xv)} \left( 1 + \frac{2 R_K}{a b^d \omega_d n h^d p(\xv)} \right) - \frac{
                    \sqrt{2} \lambda \Vert K \Vert_2 z_{1 - \beta}
                }{
                    \sqrt{n h^d \left (2 p(\xv) + L_p h \int_{\RR^d} \Vert \tv \Vert K(\tv) d \mu(\tv) \right )}
                }, \\
        \Delta (\xv) & \ge \frac{5 L_f^2 h^2}{r_k^2} \log^2 \frac{
            2 e^2 R_K
        }{
            a b^d \omega_d h^d p(\xv)
        } - \frac{
                    \sqrt{2} \lambda \Vert K \Vert_2 z_{1 - \beta}
                }{
                    \sqrt{n h^d \left (2 p(\xv) + L_p h \int_{\RR^d} \Vert \tv \Vert K(\tv) d \mu(\tv) \right )}
                },
    \end{align*}
    then the excess risk is at most
    \begin{align*}
        \EE_{\mathcal{D}} [\risk_\lambda(\xv) - \risk^*_\lambda(\xv)] & \lesssim \Biggl (
            \frac{
                \sigma^2(\xv) + \frac{2 L_{\sigma^2} h}{r_K} \log \frac{e n R_K}{a r_K}
            }{n h^d p(\xv)} + \frac{h^4}{p^2(\xv)} +  \{ n h^{d - 2} p(\xv)\}^{-1} + \Delta(\xv) \Biggr ) \cdot \mathbf{P}(\xv), \\
        \mathbf{P}(\xv) & \lesssim e^{- \Omega(n h^{d + 2} p(\xv))} \\
        & \quad + \exp \left(
            \frac{- \Omega(n h^d p(\xv) \cdot \delta(\xv))}{\sigma^2(\xv) + \frac{L_{\sigma^2} h}{e r_K}} 
            \cdot
            \min \left\{
                \frac{\delta(\xv) / 80}{\sigma^2(\xv) + \constCorTwo^\sigma h}, \frac{1}{2}
            \right\}
        \right) \\
         & \quad + \exp \left(
                - \frac{\Omega(n h^{d - 2} p(\xv)\delta^2(\xv))}{
                    \sigma^2(\xv) + \frac{L_{\sigma^2} h}{e r_K}
                }
                \cdot \left(
                    [\constCorTwo^f]^2 + \frac{2 L_f^2}{r_K^2} \log^2 \frac{e^2 R_K}{ a b^d \omega_d h^d p(\xv)}
                \right)^{-1}
            \right), \\
        \delta(\xv) & = \Delta(\xv) + \lambda \Vert K \Vert_2 z_{1 - \beta} \sqrt{\frac{2 \omega_d^{-1} r_K^d / \log^d (n^2 R_K)}{n h^d p(\xv) \cdot \left (1 +  \frac{\log (n^2 R_K)}{2 b r_K} \right )}}
    \end{align*}
    for some constants $\constCorTwo^f$ and $\constCorTwo^\sigma$ that do not depend on $\xv$ as well as constants inside $\lesssim$ and $\Omega(\cdot)$. If, additionally, we have
    \begin{align*}
        \Delta(\xv) \ge \frac{20 R_K \sigma^2(\xv)}{a b^d \omega_d n h^d p(\xv)} +5 \constCorTwoPrime^\sigma h \left( 1 + \frac{2 R_K}{a b^d \omega_d n h^d p(\xv)} \right) - \frac{
                    \sqrt{2} \lambda \Vert K \Vert_2 z_{1 - \beta}
                }{
                    \sqrt{n h^d \left (2 p(\xv) + L_p h \int_{\RR^d} \Vert \tv \Vert K(\tv) d \mu(\tv) \right )}
                },
    \end{align*}
    where $\constCorTwoPrime^\sigma$ does not depend on $\xv, n, h$,
    then
    \begin{align*}
        \mathbf{P}(\xv) & \lesssim e^{- \Omega(n h^d p(\xv))} \\
        & \quad + \exp \left(
            \frac{- \Omega(n h^d p(\xv) \cdot \delta(\xv))}{\sigma^2(\xv) + \frac{L_{\sigma^2} h}{e r_K}} 
            \cdot
            \min \left\{
                \frac{\delta(\xv) / 80}{\sigma^2(\xv) + \constCorTwo^\sigma h}, \frac{1}{2}
            \right\}
        \right) \\
         & \quad + \exp \left(
                - \frac{\Omega(n h^{d - 2} p(\xv)\delta^2(\xv))}{
                    \sigma^2(\xv) + \frac{L_{\sigma^2} h}{e r_K}
                }
                \cdot \left(
                    [\constCorTwo^f]^2 + \frac{2 L_f^2}{r_K^2} \log^2 \frac{e^2 R_K}{ a b^d \omega_d h^d p(\xv)}
                \right)^{-1}
            \right)
    \end{align*}
    If $\sigma^2(\xv) \le \lambda$ and
    \begin{align*}
        \Delta(\xv) & \ge \frac{20 R_K \sigma^2(\xv)}{a b^d \omega_d n h^d p(\xv)} +5 \constCorTwo^\sigma \frac{h^2}{p(\xv)} \left( 1 + \frac{2 R_K}{a b^d \omega_d n h^d p(\xv)} \right) + 
        \frac{
            2 \lambda (4 \pi)^{-d/4} z_{1 - \beta}
        }{
            \sqrt{
                a b^d \omega_d n h^d p(\xv)
            }
        }, \\
        \Delta (\xv) & \ge \frac{5 L_f^2 h^2}{r_k^2} \log^2 \frac{
            2 e^2 R_K
        }{
            a b^d \omega_d h^d p(\xv)
        } + 
        \frac{
            2 \lambda (4 \pi)^{-d/4} z_{1 - \beta}
        }{
            \sqrt{
                a b^d \omega_d n h^d p(\xv)
            }
        },
    \end{align*}
    then we have
    \begin{align*}
        \EE_{\mathcal{D}} [\risk(\xv) - \risk^*(\xv)] & \lesssim 
            \frac{
                \sigma^2(\xv) + \frac{2 L_{\sigma^2} h}{r_K} \log \frac{e n R_K}{a r_K}
            }{n h^d p(\xv)} + \frac{h^4}{p^2(\xv)} + \{ n h^{d - 2} p(\xv)\}^{-1} + \Delta(\xv) \cdot \mathbf{P}'(\xv), \\
        \mathbf{P}'(\xv) & \lesssim e^{- \Omega(n h^{d + 2} p(\xv))} \\
        & \quad + \exp \left(
            \frac{-\Omega(n h^d p(\xv) \cdot \delta'(\xv))}{\sigma^2(\xv) + \frac{L_{\sigma^2} h}{e r_K}} 
            \cdot
            \min \left\{
                \frac{\delta'(\xv) / 80}{\sigma^2(\xv) + \constCorTwo^\sigma h}, \frac{1}{2}
            \right\}
        \right) \\
         & \quad +
          \exp \left(
                -\frac{\Omega(n h^{d - 2} p(\xv)\delta'^2(\xv))}{
                    \sigma^2(\xv) + \frac{L_{\sigma^2} h}{e r_K}
                }
                \cdot  \left(
                    [\constCorTwo^f]^2 + \frac{2 L_f^2}{r_K^2} \log^2 \frac{e^2 R_K}{ a b^d \omega_d h^d p(\xv)}
                \right)^{-1}
            \right), \\
        \delta'(\xv) & = \Delta(\xv) - \frac{
            2 \lambda (4 \pi)^{-d/4} z_{1 - \beta}
        }{
            \sqrt{
                a b^d \omega_d n h^d p(\xv)
            }
        }.
    \end{align*}
    where constants in $\lesssim$ and $\Omega(\cdot)$ do not depend on $\xv$. If, additionally, we have
    \begin{align*}
        \Delta(\xv) & \ge \frac{20 R_K \sigma^2(\xv)}{a b^d \omega_d n h^d p(\xv)} +5 \constCorTwoPrime^\sigma h \left( 1 + \frac{2 R_K}{a b^d \omega_d n h^d p(\xv)} \right) + 
        \frac{
            2 \lambda (4 \pi)^{-d/4} z_{1 - \beta}
        }{
            \sqrt{
                a b^d \omega_d n h^d p(\xv)
            }
        },
    \end{align*}
    where $\constCorTwoPrime^\sigma$ does not depend on $n, h, \xv$, then it holds that
    \begin{align*}
        \mathbf{P}'(\xv) & \lesssim e^{- \Omega(n h^{d} p(\xv))} \\
        & \quad + \exp \left(
            \frac{-\Omega(n h^d p(\xv) \cdot \delta'(\xv))}{\sigma^2(\xv) + \frac{L_{\sigma^2} h}{e r_K}} 
            \cdot
            \min \left\{
                \frac{\delta'(\xv) / 80}{\sigma^2(\xv) + \constCorTwo^\sigma h}, \frac{1}{2}
            \right\}
        \right) \\
         & \quad +
          \exp \left(
                -\frac{\Omega(n h^{d - 2} p(\xv)\delta'^2(\xv))}{
                    \sigma^2(\xv) + \frac{L_{\sigma^2} h}{e r_K}
                }
                \cdot  \left(
                    [\constCorTwo^f]^2 + \frac{2 L_f^2}{r_K^2} \log^2 \frac{e^2 R_K}{ a b^d \omega_d h^d p(\xv)}
                \right)^{-1}
            \right), \\
        \delta'(\xv) & = \Delta(\xv) - \frac{
            2 \lambda (4 \pi)^{-d/4} z_{1 - \beta}
        }{
            \sqrt{
                a b^d \omega_d n h^d p(\xv)
            }
        }.
    \end{align*}
    Finally, for any value of $\Delta(\xv)$, we may bound
    \begin{align*}
        \EE_{\mathcal{D}} \risk_{\lambda}(\xv) - \risk_\lambda^*(\xv) \lesssim \frac{
                \sigma^2(\xv) + \frac{2 L_{\sigma^2} h}{r_K} \log \frac{e n R_K}{a r_K}
            }{n h^d p(\xv)} + \frac{h^4}{p^2(\xv)} + \{ n h^{d - 2} p(\xv)\}^{-1} + \Delta(\xv).
    \end{align*}
  \end{theorem}

\section{Proofs}

\subsection{Notation}
  Before we start our proves, we declare some notation we use:
  \begin{itemize}
    \item $\mu$ -- Lebesgue's measure on $\RR^d$;

    \item $\ball_r(\xv)$ -- the ball of the radius $r$ and the center $\xv$;

    \item $\weights$ -- weights of Nadaraya-Watson estimator, i.e.
    \begin{align*}
        \weights_i = \frac{K \left( \frac{\xv - X_i}{h}\right)}{
            \sum_{j = 1}^n K \left( \frac{\xv - X_j}{h}\right)
        };
    \end{align*}

    \item $\variances$ -- a vector that consists of $\sigma^2(X_i)$;

    \item $\meanLabels$ -- a vector of means with respect to labels $Y_i \sim \normDistribution{\meanY(X_i)}{\sigma^2(X_i)}$, i.e.
    \begin{align*}
      \meanLabels_i = \meanY(X_i);
    \end{align*}

    \item $\diag_{\yv}$ -- a diagonal matrix whose entries consists of vector $\yv$'s elements;

    \item $\omega_d$ -- the volume of a unit ball in $\RR^d$.
  \end{itemize}

\subsection{Proof of Proposition~\ref{proposition: optimal estimators}}
    Fix an estimator $\estimator{a}(\xv)$. Then the risk
    \begin{align*}
      \risk_{\lambda}(\xv) & = \EE \left [ (Y - \estimator{f}(X))^2 \indicator{\estimator{\alpha}(X) = 0} \mid X = \xv \right ] + \lambda \indicator{\estimator{\alpha}(\xv) = 1} \\
      & = \EE \left [(Y - \estimator{f}(X))^2 \mid X = \xv \right ] \indicator{\estimator{\alpha}(\xv) = 0} + \lambda \indicator{\estimator{\alpha}(\xv) = 1}
    \end{align*}
    attains the minimum for $\estimator{f}(\xv) = \EE [Y \mid X = \xv]$. For such $\estimator{f}(\xv)$ we have
    \begin{align*}
      \risk_\lambda(\xv) = \sigma^2(\xv) \indicator{\estimator{\alpha}(\xv) = 0} + \lambda \indicator{\estimator{\alpha}(\xv) = 1}.
    \end{align*}
    Clearly, $\alpha(\xv)$ is the optimal reject function.

\subsection{Proof of Proposition~\ref{proposition: excess risk decomposition}}
    Consider two cases. If $\sigma^2(\xv) \ge \lambda$, then $\risk^*_\lambda(\xv) = \lambda$. Thus,
    \begin{align*}
        \EE_{\mathcal{D}} \risk_\lambda(\xv) - \risk^*_\lambda(\xv) & = \EE \left [ (Y - \estimator{f}(\xv))^2 \indicator{\estimator{\alpha}(\xv) = 0} \mid X = \xv \right ] + \lambda \PP \left ( \estimator{\alpha}(\xv) = 1 \right) - \lambda \\
        & = \EE \left [\EE \left [ (Y - \estimator{f}(\xv))^2 \mid X = \xv \right] \indicator{\estimator{\alpha}(\xv) = 0} \right ] - \lambda \PP \left ( \estimator{\alpha}(\xv) = 0 \right).
    \end{align*}
    Then
    \begin{align*}
        & \EE \left [ (Y - \estimator{f}(\xv))^2 \mid X = \xv \right] = \EE \left[ (Y - \meanY(\xv))^2 \mid X = \xv \right ] \\
        & \quad + 2 \EE \left [ (Y - \meanY(\xv)) (\estimator{f}(\xv) - \meanY(\xv)) \mid X = \xv \right ] + \EE \left [ (\estimator{f}(\xv) - \meanY(\xv))^2 \mid X = \xv \right ] \\
        & = \sigma^2(\xv) + (\estimator{f}(\xv) - \meanY(\xv))^2.
    \end{align*}
    Thus,
    \begin{align*}
        \EE_{\mathcal{D}} \risk_\lambda(\xv) - \risk^*_\lambda(\xv) = \EE_{\mathcal{D}} \left [ (\estimator{f}(\xv) - \meanY(\xv))^2 \indicator{\estimator{\alpha}(\xv) = 0} \right ] + \Delta(\xv) \cdot \PP \left ( \estimator{\alpha}(\xv) = 0 \right ).
    \end{align*}
    Since $\alpha(\xv) = 1$, the proposition holds for the case $\sigma^2(\xv) \ge \lambda$. The case $\sigma^2(\xv) < \lambda$ can be checked analogously.  

\subsection{Weights bounding}
\begin{proposition}
  \label{proposition: denominator lower bound}
    Under Assumptions~\ref{assumption: kernel assumptions}-\ref{assumption: density assumption} it holds that
    \begin{align*}
        \PP \left( \sum_{i = 1}^n K \left(\frac{\xv - X_i}{h} \right) \le \logProb \right)
        & \le \exp \left( 
        - \logProb / 26
        \right), \\
        \logProb & = a b^d \omega_d \cdot p(\xv) \cdot n h^d / 2
    \end{align*}
    for any $\xv$ such that $p(\xv) > 2 L_p b h$ and $d(\xv, \partial S) \ge b h$.
  \end{proposition}

  \begin{proof}
    From Assumption~\ref{assumption: kernel assumptions}, we have
    \begin{align}
    \label{eq: proposition 1, kernel below estimation}
        \sum_{i = 1}^n K \left(\frac{\xv - X_i}{h} \right) 
        \ge 
        \sum_{i = 1}^n a \indicator{\Vert \xv - X_i \Vert \le h b}.
    \end{align}
    The right-hand side is a sum of Bernoulli random variables multiplied by $a$. The probability of one indicator is
    \begin{equation*}
        \PP (\Vert \xv - X_i \Vert \le h b) = \int_{\ball_{h b}(\xv)} p(\yv) d\mu( \yv)
        \le (p(\xv) + L_p b h) \cdot \mu \left( \ball_{b h}(0) \right)
    \end{equation*}
    since $p(\cdot)$ is $L_p$-Lipschitz according to Assumption~\ref{assumption: density assumption}. Thus,
    \begin{align*}
        \Var \indicator{\Vert \xv - X_i \Vert \le b h} \le \EE \indicator{\Vert \xv - X_i \Vert \le b h} \le \omega_d \cdot  (p(\xv) + L_p b h) \cdot (b h)^d,
    \end{align*}
    where $\omega_d$ is the measure of a unit ball in $\RR^d$. Analogously, if $d(\xv, \partial S) \ge b h$ we have
    \begin{align*}
        \PP \left(
            \Vert \xv - X_i \Vert \le h b
        \right) \ge \omega_d \cdot (b h)^d \cdot \left(
            p(\xv) - L_p b h
        \right).
    \end{align*}
    Applying the Bernstein inequality, we obtain
    \begin{align}
        \PP \bigg (
            \sum_{i = 1}^n a & \indicator{\Vert \xv - X_i \Vert \le h b}
            \le
            \frac{1}{2} a n \PP (\Vert \xv - X_i \Vert \le h b)
        \bigg ) \nonumber \\
        & \le
        \exp \left(
            - \frac{1}{8} 
            \cdot 
            \frac{
                a^2 n^2 (b h)^{2 d} \omega_d^2 (p(\xv) - L_p b h)^2
             }{
                a n \omega_d \cdot (b h)^d \cdot (p(\xv) + L_p b h)
                +
                \frac{1}{6}
                a n \omega_d \cdot (p(\xv) - L_p b h) (b h)^d
             }
        \right) \nonumber \\
        & \le 
        \exp \left(
            - \frac{1}{8}
            \cdot
            \frac{
                a b^d \omega_d \cdot (p(\xv) - L_p b h) n h^d 
            }{
                \frac{1}{6} + (p(\xv) + L_p b h) / (p(\xv) - L_p b h)
            }
        \right).
      \label{eq: proposition 1 exponent}
    \end{align}
    Since $p(\xv) \ge 2 L_p b h$, we have 
    \begin{align*}
        p(\xv) - L_p b h & \ge p(\xv) / 2, \\
        \frac{p(\xv) + L_p b h}{p(\xv) - L_p b h} & = 1 + \frac{2 L_p b h}{p(\xv) - L_p b h} \le 3,
    \end{align*}
    and we can bound~\eqref{eq: proposition 1 exponent} by $e^{-\logProb / 26}$. Combining it with~\eqref{eq: proposition 1, kernel below estimation}, we obtain the proposition.
  \end{proof}
  From the proposition, the following corollary follows:
  \begin{corollary}
  \label{corollary: weights bound}
    Under Assumptions~\ref{assumption: kernel assumptions}-\ref{assumption: density assumption} it holds that
    \[
        \PP \left(
            \max_i g(X_i) \weights_i \ge g(\xv) \frac{R_K}{\logProb} + \frac{L R_K h}{e r_K \logProb} 
        \right)
        \le 
        \exp \left(
         - \logProb / 26
        \right)
    \]
    simultaneously for $L$-Lipschitz $g$ and any $\xv$ such that $p(\xv) > 2 L_p b h$ and $d(\xv, \partial S) \ge b h$.
  \end{corollary}

    \begin{proof}
    Since $g(\xv)$ is $L$-Lipschitz, we may state
    \begin{align*}
        g(X_i) \le g(\xv) + L \Vert X_i - \xv \Vert.
    \end{align*}
    Thus, $\max_i g(X_i) \weights_i \le g(\xv) \max_i \weights_i + L \max_i \Vert X_i - \xv \Vert \weights_i$. We may bound
    \begin{align*}
        & \max_i \weights_i \le \frac{K \left( \frac{X_i - \xv}{h} \right)}{\sum_{i = 1}^n K \left( \frac{X_i - \xv}{h} \right)}
        \le \frac{R_K}{\sum_{i = 1}^n K \left( \frac{X_i - \xv}{h} \right)}, \\
        & \max_i \Vert X_i - \xv \Vert \weights_i \le \frac{
            \max_i \Vert X_i - \xv \Vert K \left( \frac{X_i - \xv}{h} \right)
        }{
            \sum_{i = 1}^n K \left( \frac{X_i - \xv}{h} \right)
        } \\
        & \le \frac{
            R_K h \max_i \left\Vert \frac{X_i - \xv}{h} \right \Vert e^{- r_K \frac{X_i - \xv}{h}}
        }{
            \sum_{i = 1}^n K \left( \frac{X_i - \xv}{h} \right)
        }
        \le \frac{
            R_K h 
        }{
            e r_K \sum_{i = 1}^n K \left( \frac{X_i - \xv}{h} \right)
        }.
    \end{align*}
    Thus,
    \begin{align*}
       \PP \left(
            \max_i g(X_i) \weights_i \ge g(\xv) \frac{2 R_K}{\logProb} + \frac{L R_K h}{e r_K \logProb} 
        \right)
        & \le \PP \left( 
            \sum_{i = 1}^n K \left( \frac{X_i - \xv}{h} \right) \le \logProb
        \right ) \\
        & \le \exp \left( 
        - \logProb / 26 
        \right)
    \end{align*}
    due to Proposition~\ref{proposition: denominator lower bound}.
  \end{proof}


\subsection{Deviation of estimated noiseless mean}
  Before the main lemma of this section we introduce a simple auxiliary proposition:

  \begin{proposition}
  \label{proposition: bounds on maxima and integrals}
    Suppose that for a kernel $K\colon \RR^d \to \RR_{+}$ Assumption~\ref{assumption: kernel assumptions} holds. Then, we have
    \begin{align*}
        \max_{\tv \in \RR^d} \Vert \tv \Vert^m K(\tv) & \le R_K \left(\frac{m}{r_K} \right)^m e^{-m}, \\
        \int_{\tv \in \RR^d} \Vert \tv \Vert^m K^k(\tv) d \mu(\tv) & \le \frac{2 \pi^{d/2} R_K^k \cdot m!}{\Gamma(d/2) (r_K k)^{m + 1}}
    \end{align*}
    for any non-negative integers $k, m$.
  \end{proposition}

  \begin{proof}
    From Assumption~\ref{assumption: kernel assumptions} we have $K(\tv) \le R_K e^{- r_K \Vert \tv \Vert}$. The first inequality is obtained via maximizing $\Vert \tv \Vert^m R_K e^{- r_K \Vert \tv \Vert}$, the second one from calculations
    \begin{align*}
        & \int_{\RR^d} \Vert \tv \Vert^m K^k(\tv) d \mu(\tv) \le \int_{\RR^d} \Vert \tv \Vert^m R_K^k e^{- r_K k \Vert \tv \Vert} d \mu(\tv)
        = R^k_K \cdot \mu(\mathbb{S}^{d-1}) \int_{0}^{+ \infty} \rho^m e^{- r_K \rho k} d \rho \\
        & = R_k^k \cdot \frac{2 \pi^{d/2}}{\Gamma(d/2)} \cdot \frac{\Gamma(m + 1)}{(r_K k)^{m + 1}}
        = \frac{2 \pi^{d/2} R_K^k \cdot m!}{\Gamma(d/2) (r_K k)^{m + 1}},
    \end{align*}
    where $\mathbb{S}^{d - 1}$ stands for a $(d-1)$-dimensional sphere.
  \end{proof}

  \begin{proposition}
  \label{proposition: deviations of kernal estimator}
    Assume that a kernel $K(\cdot)$ and $p(\cdot)$ satisfy Assumption~\ref{assumption: kernel assumptions} and Assumption~\ref{assumption: density assumption} respectively. Let a function $g(\cdot)$ be twice differential with the Hessian bounded by $H$ in the spectral norm and the gradient bounded by $L$. Finally, let $X_1, \ldots, X_n \sim p(\cdot)$ be identically independently distributed random variables. Then
    \begin{align*}
        \PP \left(
            \left|
              \sum_{i = 1}^n \weights_i g(X_i) - g(\xv)
            \right|
            \ge t
        \right)
        \le 2 \exp \left(
        - \frac{1}{2} \cdot
        \frac{
            n h^d p(\xv) \cdot \mathtt{r}_1
        }{
            \mathtt{r}_2 / \mathtt{r}_1 + \mathtt{r}_3 / 3
        }
    \right),
    \end{align*}
    if $\mathtt{r}_1 > 0$ where
    \begin{align}
        \mathtt{r}_1 & = \left\{ 
            1 - \frac{L_p h }{p(\xv)} \cdot \frac{2\pi^{d/2} R_K}{\Gamma(d/2) r_K^2}
        \right\} \cdot t 
        - 
        \frac{h^2}{p(\xv)} \left\{
            \frac{2 L L_p + H C_p}{2} 
            \cdot 
            \frac{4 \pi^{d/2} R_K}{\Gamma(d/2) r_K^3}
        \right\}, \label{eq: r_1 definition} \\
        \mathtt{r}_2 & = \left\{
            \frac{2 \pi^{d/2} R_K^2}{\Gamma(d/2) r_k} + \frac{2 h L_p}{p(\xv)} \cdot \frac{4 \pi^{d/2} R_k}{\Gamma(d/2) r_k^2}
        \right\} \cdot t^2 
        +
        \left\{
            2 L^2 \frac{\pi^{d/2} R_K^2}{2 \Gamma(d/2) r_K^3}
            +
            \frac{2 h L^2 L_p}{p(\xv)} 
            \cdot
            \frac{
                4 \pi^{d/2} R_K^2
            }{3 \Gamma(d/2) r_K^4}
        \right\} \cdot h^2, \label{eq: r_2 definition} \\
        \mathtt{r}_3 & = \frac{L h R_K}{e r_K} + t R_K. \label{eq: r_3 definition}
    \end{align}
  \end{proposition}

  \begin{proof}
    We analyze the probability via bounding
    \begin{align*}
      \PP \left(
            \left|
                \sum_{i = 1}^n \weights_i g(X_i) - g(\xv)
            \right|
            \ge t
        \right)
    \le  \PP \left(
                \sum_{i = 1}^n \weights_i g(X_i) - g(\xv)
            \ge t
        \right)
        +
         \PP \left(
                \sum_{i = 1}^n \weights_i g(X_i) - g(\xv)
            \le -t
        \right).
    \end{align*}
    We consider only the first term, the second one can be processed analogously. By rearranging terms, the problem reformulates as the bounding the probability of
    \begin{align*}
      \sum_{i = 1}^n K \left( \frac{X_i - \xv}{h} \right) \left(
          g(X_i) - g(\xv) - t
      \right) \ge 0.
    \end{align*}
    The expectation of the above is
    \begin{equation*}
      \EE \sum_{i = 1}^n K \left( \frac{X_i - \xv}{h} \right) \left(
          g(X_i) - g(\xv) - t
      \right) = n \EE \left( g(X_i) - g(\xv) \right) K \left( \frac{X_i - \xv}{h} \right)
      - n t \cdot \EE K \left( \frac{X_i - \xv}{h} \right).
    \end{equation*}
    Meanwhile,
    \begin{align*}
      & \left|
        \EE \left(
            g(X_i) - g(\xv) 
        \right) K \left( \frac{X_i - \xv}{h} \right)
      \right|
      = \left| \int_{\RR^d} (g(\yv) - g(\xv)) K \left(\frac{\yv - \xv}{h} \right) p(\yv) d\mu(\yv) \right| \\
      & \le  \left| \int_{\RR^d} \langle \nabla g(\xv) , \yv - \xv \rangle K \left(\frac{\yv - \xv}{h} \right) p(\yv) d \mu(\yv) \right|
      + \frac{H}{2} \int_{\RR^d} \Vert \yv - \xv \Vert^2 K \left( \frac{\xv - \yv}{h} \right) p(\yv) d \mu(\yv)
    \end{align*}
    since
    \begin{align*}
      g(\xv) + \langle \nabla g(\xv), \yv - \xv \rangle - \frac{H}{2} \Vert \yv - \xv \Vert^2 \le g(\yv) \le g(\xv) + \langle \nabla g(\xv), \yv - \xv \rangle + \frac{H}{2} \Vert \yv - \xv \Vert^2.
    \end{align*}
    At the same time, we have
    \begin{align*}
      & \langle \nabla g(\xv), \yv - \xv \rangle K \left(\frac{\yv - \xv}{h} \right) p(\yv) 
      \le
      p(\xv) \langle \nabla g(\xv), \yv - \xv \rangle K \left(\frac{\yv - \xv}{h} \right)  +
      L L_p \Vert \yv - \xv \Vert^2 K \left(\frac{\yv - \xv}{h} \right) 
    \end{align*}
    and 
    \begin{align*}
      & \langle \nabla g(\xv), \yv - \xv \rangle K \left(\frac{\yv - \xv}{h} \right) p(\yv) 
      \ge
      p(\xv) \langle \nabla g(\xv), \yv - \xv \rangle K \left(\frac{\yv - \xv}{h} \right)  -
      L L_p \Vert \yv - \xv \Vert^2 K \left(\frac{\yv - \xv}{h} \right) .
    \end{align*}
    Bounding $p(\yv) \le C_p$, we obtain
    \begin{align*}
      \left|
      \EE \left(
          g(X_i) - g(\xv) 
      \right) K \left( \frac{X_i - \xv}{h} \right)
      \right|
      & \le p(\xv) 
      \int_{\RR^d}
          \langle \nabla g(\xv), \yv - \xv \rangle K \left( \frac{\yv - \xv}{h} \right) d \mu(\yv) \\
      & \quad + LL_p \int_{\RR^d} \Vert \yv - \xv \Vert^2 K \left( \frac{\yv - \xv}{h} \right) d \mu(\yv) \\
      & \quad + \frac{H C_p}{2} \int_{\RR^d} \Vert \yv - \xv \Vert^2 K \left( \frac{\yv - \xv}{h} \right) d \mu(\yv).
    \end{align*}
    Using
    \begin{align*}
      \int_{\RR^d} \langle \nabla g(\xv), \yv - \xv \rangle K \left( \frac{\yv - \xv}{h} \right) d \mu (\yv) = 0
    \end{align*}
    from Assumption~\ref{assumption: kernel assumptions}, we get
    \begin{align*}
      \left|\EE \left(
        g(X_i) - g(\xv) 
      \right) K \left( \frac{X_i - \xv}{h} \right) \right|
      & \le L L_p\int_{\RR^d} \Vert \yv - \xv \Vert^2 K \left( \frac{\yv - \xv}{h} \right) d \mu(\yv).
  \end{align*}
    Finally, changing variables leads us to
    \begin{align*}
      \int_{\RR^d} \Vert \yv - \xv \Vert^2 K \left( \frac{\yv - \xv}{h} \right) d \mu (\yv) & = h^d \cdot h^2 \int_{\RR^d} \Vert \tv \Vert^2 K(\tv) d \mu(\tv).
    \end{align*}
    At the same time, we have
    \begin{align*}
      \EE K \left( \frac{X_i - \xv}{h}\right) & = h^d \int_{\RR^d} p(\xv + \tv h) K(\tv) d \mu(\tv) \\ 
      & \ge h^d \left(p(\xv) - L_p h \int_{\RR^d} \Vert \tv \Vert K(\tv) d \mu(\tv) \right).
    \end{align*}
    Consequently,
    \begin{align}
    \label{eq: expectation in bernstein for general function}
      \left| \EE \sum_{i = 1}^n K \left( \frac{X_i - \xv}{h} \right) \left(
          g(X_i) - g(\xv) - t
      \right) \right|
      \ge 
      n h^d \left( \mathtt{c}_t t - \mathtt{c}_{h^2} h^2 \right),
    \end{align}
    where
    \begin{align*}
      \mathtt{c}_t & = p(\xv) - L_p h \int_{\RR^d} \Vert \tv \Vert K(\tv) d \mu(\tv), \\
      \mathtt{c}_{h^2} & = \frac{2 L L_p + H C_p}{2} \int_{\RR^d} \Vert \tv \Vert^2 K \left( \tv \right) d \mu(\tv).
    \end{align*}
    Next, we bound
    \begin{align*}
      \Var K \left( \frac{X_i - \xv}{h}\right) \left\{ 
          g(X_i) - g(\xv) - t
      \right\} & \le \EE K^2 \left( \frac{X_i - \xv}{h}\right) \left\{ 
          g(X_i) - g(\xv) - t
      \right\}^2 \\
      & \le 2 \EE K^2 \left( \frac{X_i - \xv}{h}\right) \left\{ 
          g(X_i) - g(\xv)
      \right\}^2 + 2 t^2 \EE K^2 \left( \frac{X_i - \xv}{h}\right) \\
      & \le 2 L^2 \EE K^2 \left( \frac{X_i - \xv}{h}\right) \Vert X_i - \xv \Vert^2 + 2 t^2 \EE K^2 \left( \frac{X_i - \xv}{h}\right).
    \end{align*}
    As previously,
    \begin{align*}
      \EE K^2 \left( \frac{X_i - \xv}{h}\right) & = \int_{\RR^d} K \left( \frac{\yv - \xv}{h}\right) p(\yv) d \mu(\yv) \\
      & \le h^d p(\xv) \int_{\RR^d} K^2(\tv) d \mu(\tv) + h^{d + 1} L_p \int_{\RR^d} \Vert \tv \Vert K(\tv) d \mu(\tv)
    \end{align*}
    and
    \begin{align*}
     \EE K^2 \left( \frac{X_i - \xv}{h}\right) \Vert X_i - \xv \Vert^2 & \le  h^{d + 2} p(\xv) \int_{\RR^d} \Vert \tv \Vert^2 K^2(\tv) d \mu(\tv) + h^{d + 3} L_p \int_{\RR^d} \Vert \tv \Vert^3 K^2(\tv) d \mu(\tv).
    \end{align*}
    Consequently,
    \begin{align}
    \label{eq: var in bernstein for general function}
      \Var K \left( \frac{X_i - \xv}{h}\right) \left\{ 
          g(X_i) - g(\xv) - t
      \right\} \le h^d \left\{ \mathtt{c}_t' t^2 + \mathtt{c}_{h^2}' h^2 \right\}
    \end{align}
    where
    \begin{align*}
      \mathtt{c}_t' & = 2 p(\xv) \int_{\RR^d} K^2(\tv) d \mu(\tv) + 2 h L_p \int_{\RR^d} \Vert \tv \Vert K(\tv) d \mu(\tv), \\
      \mathtt{c}_{h^2}' & = 2 L^2 p(\xv) \int_{\RR^d} \Vert \tv \Vert^2 K^2(\tv) d \mu(\tv) + 2 h L^2 L_p \int_{\RR^d} \Vert \tv \Vert^3 K^2(\tv) d \mu(\tv).
    \end{align*}
    Finally,
    \begin{align}
      \left| K \left( \frac{X_1 - \xv}{h} \right) (g(X_i) - g(\xv) - t) \right| & \le K \left( \frac{X_1 - \xv}{h} \right) \left(|g(X_i) - g(\xv)| + t \right) \nonumber \\
      & \le L h \max_{\tv \in \RR^d} \Vert \tv \Vert K(\tv) + t \max_{\tv \in \RR^d} K(\tv). \label{eq: esssup bound for Bernstein}
    \end{align}
    Define
    \begin{align*}
      \xi_i = K \left( \frac{X_i - \xv}{h} \right) \left(
        g(X_i) - g(\xv) - t
      \right),
    \end{align*}
    then the probability from the statement can be bounded as
    \begin{align*}
      \PP \left(
          \sum_{i = 1}^n \xi_i \ge 0
      \right)
      = 
      \PP \left(
          \sum_{i = 1}^n \xi_i - \sum_{i = 1}^n \EE \xi_i \ge - \sum_{i = 1}^n \EE \xi_i
      \right).
    \end{align*}
    If $\mathtt{c}_t t > \mathtt{c}_{h^2} h^2$ then $\EE \xi_i$ is negative due to inequality~\eqref{eq: expectation in bernstein for general function}, and, whence, the above can be bounded via the Bernstein inequality:
    \begin{align}
    \label{eq: probability of proposition 3}
      \PP \left(
        \sum_{i = 1}^n \xi_i - \sum_{i = 1}^n \EE \xi_i \ge - \sum_{i = 1}^n \EE \xi_i
      \right)
      \le \exp \left( - \frac{1}{2} \cdot \frac{(n \EE \xi_1)^2}{n \Var \xi_1 +n |\EE \xi_1| / 3 \cdot \operatorname{ess \, sup}_{X_1} |\xi_1| }\right).
    \end{align}
    Substituting bounds~\eqref{eq: expectation in bernstein for general function},~\eqref{eq: var in bernstein for general function},~\eqref{eq: esssup bound for Bernstein} instead of $\EE \xi_1$, $\Var \xi_1$ and $\operatorname{ess\, sup}_{X_1} |\xi_1|$ respectively, we obtain
    \begin{align*}
      \eqref{eq: probability of proposition 3} & \le 
      \exp \left( 
        - \frac{1}{2} \cdot \frac{
            n^2 h^{2d} (\mathtt{c}_t t - \mathtt{c}_{h^2} h^2)^2
        }{
            n h^d \{\mathtt{c}_t' t^2 + \mathtt{c}_{h^2}' h^2 \} + \frac{n h^d}{3} (\mathtt{c}_t t - \mathtt{c}_{h^2} h^2) \{ 
                L h \max_{\tv \in \RR^d} \Vert \tv \Vert K(\tv) + t \max_{\tv \in \RR^d} K(\tv)
            \}
        }
      \right) \\
      & \le \exp \left(
        - \frac{1}{2} \cdot
        \frac{
            n h^d p(\xv) \cdot \mathtt{r}_1
        }{
            \mathtt{r}_2 / \mathtt{r}_1 + \mathtt{r}_3 / 3
        }
      \right),
    \end{align*}
    where
    \begin{equation*}
      \mathtt{r}_1 = \frac{\mathtt{c}_t t - \mathtt{c}_{h^2} h^2}{p(\xv)}, \quad
      \mathtt{r}_2 = \frac{\mathtt{c}_t ' t^2 + \mathtt{c}'_{h^2} h^2 }{p(\xv)}, \quad
      \mathtt{r}_3 = L h \max_{\tv \in \RR^d} \Vert \tv \Vert K(\tv) + t \max_{\tv \in \RR^d} K(\tv).
    \end{equation*}
    Replacing integrals and maxima in the above with their bounds from Proposition~\ref{proposition: bounds on maxima and integrals}, we obtain the statement of the proposition.
  \end{proof}

  In most of the cases, it is sufficient to use the simplified version of the proposition.
  \begin{corollary}
  \label{corollary: nhp concentration bound}
    Assume that a kernel $K(\cdot)$ and $p(\cdot)$ satisfy Assumption~\ref{assumption: kernel assumptions} and Assumption~\ref{assumption: density assumption} respectively. Let a function $g(\cdot)$ be twice differential with the Hessian bounded by $H$ in the spectral norm and the gradient bounded by $L$. Finally, let $X_1, \ldots, X_n \sim p(\cdot)$ be identically independently distributed random variables. 
    If
    \begin{enumerate}
        \item $p(\xv) \ge 2 L_p h \cdot \max \{b, \frac{2 \pi^{d/2} R_K}{\Gamma(d/2) r_K^2}\}$,
        \item and 
        \begin{align*}
            t > \frac{\constCorTwo h^2}{p(\xv)}, \quad \constCorTwo =  \left\{
                (2 L L_p + H C_p) \frac{4 \pi^{d/2} R_K}{\Gamma(d/2) r_K^3}
            \right\},
        \end{align*}
    \end{enumerate}
    then
    \begin{align*}
        \PP \left(
            \left|
                \sum_{i = 1}^n \weights_i g(X_i) - g(\xv)
            \right|
            \ge t
        \right)
        \le 2 \exp \left(
        - C n h^{d + 2} p(\xv)
    \right),
    \end{align*}
    for some constant $C$ that does not depend on $\xv$. Additionally, if $t > \constCorTwoPrime h$, where
    \begin{align*}
        \constCorTwoPrime = \frac{1}{2 L_p b} \left\{
                (2 L L_p + H C_p) \frac{4 \pi^{d/2} R_K}{\Gamma(d/2) r_K^3} 
            \right\},
    \end{align*}
    when 
    \begin{align*}
        \PP \left(
            \left|
                \sum_{i = 1}^n \weights_i g(X_i) - g(\xv)
            \right|
            \ge t
        \right)
        \le 2 \exp \left(
        - C' n h^{d} p(\xv)
    \right),
    \end{align*}
    for some constant $C'$ that does not depend on $\xv$.
  \end{corollary}

  \begin{proof}
    Consider the definition~\eqref{eq: r_1 definition} of $\mathtt{r}_1$ in Proposition~\ref{proposition: deviations of kernal estimator}. First, we analyze the coefficient of $t$. Since $p(\xv) \ge 2 L_p h \cdot \frac{2\pi^{d/2} R_K}{\Gamma(d/2) r_K^2}$, we have:
    \begin{align*}
      1 - \frac{L_p h }{p(\xv)} \cdot \frac{2\pi^{d/2} R_K}{\Gamma(d/2) r_K^2}  \ge \frac{1}{2}.
    \end{align*}
    %
    %
    Thus, we may state that $\mathtt{r}_1 \ge \frac{1}{2} (t^2 - \mathtt{c}_1 h^2/p(\xv))$ where $\mathtt{c}_1$ does not depend on $\xv$.

    Next, we bound coefficients of $\mathtt{r}_2$ defined by~\eqref{eq: r_2 definition}. The first condition ensures that
    \begin{align*}
      \frac{2 \pi^{d/2} R_K^2}{\Gamma(d/2) r_k} + \frac{2 h L_p}{p(\xv)} \cdot \frac{4 \pi^{d/2} R_k}{\Gamma(d/2) r_k^2} & \le \frac{2 \pi^{d/2} R_K^2}{\Gamma(d/2) r_k} \left (1 + \frac{2 b}{r_K} \right), \\
      2 L^2 \frac{\pi^{d/2} R_K^2}{2 \Gamma(d/2) r_K^3}
              +
              \frac{2 h L^2 L_p}{p(\xv)} 
              \cdot
              \frac{
                  4 \pi^{d/2} R_K^2
              }{3 \Gamma(d/2) r_K^4} & \le 
        L^2 \frac{\pi^{d/2} R_K^2}{ \Gamma(d/2) r_K^3} \left( 1 + \frac{4 b}{3 r_K} \right).
    \end{align*}
    Consequently, $\mathtt{r}_2$ is at most $\mathtt{c}_2 t^2 + \mathtt{c}_3 h^2$ for two constants $\mathtt{c}_2$ and $\mathtt{c}_3$ that do not depend on $\xv$. Clearly, $\mathtt{r}_3 = \mathtt{c}_4 t + \mathtt{c}_5 h$ where $\mathtt{c}_4$ and $\mathtt{c}_5$ do not depend on $\xv$ too. Bounding $\mathtt{r}_3 \le t$, we obtain
    \begin{align*}
      \frac{
             \mathtt{r}_1
          }{
              \mathtt{r}_2 / \mathtt{r}_1 + \mathtt{r}_3 / 3
          }
      \ge \frac{1}{4} \frac{
          (t - \mathtt{c}_1 h^2/p(\xv))^2
      }{
          \mathtt{c}_2 t^2 + \mathtt{c}_3 h^2 + (\mathtt{c}_4 t^2 + \mathtt{c}_5 t h) / 3
      }.
    \end{align*}
    The third condition of the corollary guarantees that $t \ge 2 \mathtt{c}_1 h^2 / p(\xv)$, so the right-hand side of the above can not be zero. It is bounded below by 
    \begin{align*}
        \frac{1}{16} \frac{
          t^2
      }{
          \mathtt{c}_2 t^2 + \mathtt{c}_3 h^2 + (\mathtt{c}_4 t^2 + \mathtt{c}_5 t h) / 3
      } \ge C h^2,
    \end{align*}
    since $t \ge 2 \mathtt{c}_1 h^2 / p(\xv)$ and $h$ is bounded by $C_p / (2 L_p b)$.
    Applying Proposition~\ref{proposition: deviations of kernal estimator}, we obtain the first part of the corollary. If $t > 2 \constCorTwoPrime h$, then we have
    \begin{align*}
        \frac{1}{16} \frac{
          t^2
      }{
          \mathtt{c}_2 t^2 + \mathtt{c}_3 h^2 + (\mathtt{c}_4 t^2 + \mathtt{c}_5 t h) / 3
      } \ge C',
    \end{align*}
    and the second part of the corollary follows.
  \end{proof}

    While the above provides probabilistic bound, we also requires deterministic bound:

  \begin{proposition}
  \label{proposition: deterministic bound}
    Suppose a function $g$ is $L$-Lipschitz and for a kernel $K(\cdot)$ Assumption~\ref{assumption: kernel assumptions} holds. Then
    \begin{align*}
        \sum_{i = 1}^n \left| g(X_i) - g(\xv) \right|^{s} \weights_i \le \frac{2 L^s h^s}{r_K^s} \left(s \vee \log \frac{n R_K}{\sum_{i = 1}^n K \left( \frac{X_i - \xv}{h} \right)} \right)^s.
    \end{align*}
  \end{proposition}

  \begin{proof}
    The proof is straightforward. We start with
    \begin{equation*}
        \sum_{i = 1}^n \left| g(X_i) - g(\xv) \right|^{s} \weights_i \le L^s \sum_{i = 1}^n \Vert X_i - \xv \Vert^s \weights_i
        = 
        \frac{
            L^s \sum_{i = 1}^n \Vert X_i - \xv \Vert^s K \left( \frac{X_i - \xv}{h} \right)
        }{
            \sum_{i = 1}^n K \left( \frac{X_i - \xv}{h} \right)
        }.
    \end{equation*}
    Then impose some parameter $t_0 \ge \frac{s}{r_K}$. Such a restriction guarantees that $t^s e^{-r_K t} \le t_0^s e^{-R_K t_0}$ for any $t \ge t_0$. Consider
    \begin{align*}
        & \sum_{i = 1}^n \Vert X_i - \xv \Vert^s K \left( \frac{X_i - \xv}{h} \right)
        =
        \sum_{i \mid X_i \in \ball_{h t_0} (\xv)} \Vert X_i - \xv \Vert^s K \left( \frac{X_i - \xv}{h} \right)
        +
        \sum_{i \mid X_i \not \in \ball_{h t_0} (\xv)} \Vert X_i - \xv \Vert^s K \left( \frac{X_i - \xv}{h} \right) \\
        & \le h^s t_0^s \sum_{i = 1}^n K \left( \frac{X_i - \xv}{h} \right) + n R_k h^s t_0^s e^{- r_K t_0}
        = h^s t_0^s \left( \sum_{i = 1}^n K \left( \frac{X_i - \xv}{h} \right) + n R_k e^{- r_K t_0} \right).
    \end{align*}
    If $\sum_{i = 1}^n K \left( \frac{X_i - \xv}{h} \right) > n R_K e^{-s}$, set $t_0 = s/r_K$. Then
    \begin{equation*}
        \frac{
            L^s \sum_{i = 1}^n \Vert X_i - \xv \Vert^s K \left( \frac{X_i - \xv}{h} \right)
        }{
            \sum_{i = 1}^n K \left( \frac{X_i - \xv}{h} \right)
        }
        \le \frac{
            L^s h^s t_0^s \cdot 2 \sum_{i = 1}^n K \left( \frac{X_i - \xv}{h} \right) 
        }{
            \sum_{i = 1}^n K \left( \frac{X_i - \xv}{h} \right)
        }
        \le 2 L^s h^s s^s / r_K^s.
    \end{equation*}
    Otherwise choose $t_0$ such that
    \begin{align*}
        n R_K e^{-r_K t_0} = \sum_{i = 1}^n K \left( \frac{X_i - \xv}{h} \right). 
    \end{align*}
    Then
    \begin{align*}
         \frac{
            L^s \sum_{i = 1}^n \Vert X_i - \xv \Vert^s K \left( \frac{X_i - \xv}{h} \right)
        }{
            \sum_{i = 1}^n K \left( \frac{X_i - \xv}{h} \right)
        } \le 2 L^s h^s t_0^s
    \end{align*}
    for
    \begin{align*}
        t_0  = \frac{1}{r_K} \log \frac{n R_K}{\sum_{i = 1}^n K \left(\frac{X_i - \xv}{h} \right)}.
    \end{align*}
    Thus, the statement holds.
  \end{proof}

\subsection{Estimation of variance}
  In this section we establish the concentration properties of the estimator
  \begin{align}
    \estimatorn{\sigma}^2(\xv) & = \labels^\T \diag_{\weights} \labels - \labels^\T \weights \weights^\T \labels \nonumber \\
    & =  (\labels - \meanLabels)^\T (\diag_{\weights} - \weights \weights^\T) (\labels - \meanLabels) \label{eq: section 5.3, sigma_sq quadratic term} \\
    & \quad + 2 (\labels - \meanLabels)^\T (\diag_{\weights} - \weights \weights^\T) \meanLabels \label{eq: section 5.3, sigma_sq linear term} \\
    & \quad + \meanLabels^\T (\diag_{\weights} - \weights \weights^\T) \meanLabels. \label{eq: section 5.3, sigma_sq constant term}
  \end{align}
  We estimate each term separately. First, we bound deviations of term~\ref{eq: section 5.3, sigma_sq quadratic term}. Define $n$ independent random variables $Z_i \sim \normDistribution{0}{1}$. Further, we will show that 
  \begin{align*}
    (\labels - \meanLabels)^\T (\diag_{\weights} - \weights \weights^\T) (\labels - \meanLabels) 
    \overset{d}{=}
    \sum_{i = 1}^n \lambda_i(\Sigma) Z_i^2,
  \end{align*}
  where $\Sigma = \diag_{\variances} (\diag_{\weights} - \weights \weights^\T)$. The following proposition allows to establish large deviations inequality:
  \begin{proposition}
  \label{proposition: Hanson-Wright ineq}
    For any $\xv$ and $t > 0$, we have
    \begin{align*}
        \PP \left( 
            \left|
                (\labels - \meanLabels)^\T (\diag_{\weights} - \weights \weights^\T) (\labels - \meanLabels)
                -
                \tr (\Sigma)
            \right|
            \ge 16 \max\{
                \sqrt{\tr(\Sigma^2) t},
                \Vert \Sigma \Vert t
            \}
        \right)
        \le 2 e^{-t},
    \end{align*}
    where $\Sigma = \diag_{\variances} (\diag_{\weights} - \weights \weights^\T)$.
  \end{proposition}

  \begin{proof}
    Since term~\eqref{eq: section 5.3, sigma_sq quadratic term} is a quadratic form of Gaussian vector, it admits the representation
    \begin{align*}
      \eqref{eq: section 5.3, sigma_sq quadratic term} = \sum_{i = 1}^n \lambda_i(\Sigma') Z_i^2,
    \end{align*}
    where $Z_i$ are independent random variables from standard normal distribution and
    \begin{align*}
      \Sigma' = \diag_{\variances}^{1/2} (\diag_{\weights} - \weights \weights^\T) \diag_{\variances}^{1/2}.
    \end{align*}
    Meanwhile, the non-zero part of the spectrum of arbitrary matrix product $\mathbf{A} \mathbf{B}$ coincides with the one of $\mathbf{B} \mathbf{A}$. Consequently, the mean of~\eqref{eq: section 5.3, sigma_sq quadratic term} is equal to $\sum_{i = 1}^n \lambda_i(\Sigma') = \tr(\Sigma)$, $\max_i |\lambda_i(\Sigma')| = \max_i |\lambda_i(\Sigma)| = \Vert \Sigma \Vert$ and $\tr((\Sigma')^2) = \tr(\Sigma^2)$. Thus, applying standard inequality for sum of sub-exponential random variables (see, for example,~\citep{wainwright_high-dimensional_2019}), we have:
    \begin{align*}
        \PP \left(
            \left|
                (\labels - \meanLabels)^\T (\diag_{\weights} - \weights \weights^\T) (\labels - \meanLabels)
                -
                \tr (\Sigma)
            \right|
            \ge r
            \mid 
            X
        \right)
        \le 2 \cdot 
        \begin{cases}
            e^{- \frac{r^2}{16 \tr(\Sigma^2)}}, & \text{ if } 0 \le r \le \frac{\tr(\Sigma^2)}{\Vert \Sigma \Vert}, \\
            e^{- \frac{r}{16 \Vert \Sigma \Vert}} & \text{ if } r > \frac{\tr(\Sigma^2)}{\Vert \Sigma \Vert}.
        \end{cases}
    \end{align*}
    Substituting $r$ with $16 \max \{\sqrt{\tr(\Sigma^2) t}, \Vert \Sigma \Vert t \} $, we obtain the statement of the proposition.
  \end{proof}

  The above propositions allows to establish precise large deviation bounds for term~\ref{eq: section 5.3, sigma_sq quadratic term}. Let $\constCorTwo^f, \constCorTwo^\sigma$ be constants obtained from Corollary~\ref{corollary: nhp concentration bound} applied for functions $\sigma^2(\cdot)$ and $f(\cdot)$ as $\constCorTwo$:
  \begin{align*}
    \constCorTwo^f & = \frac{1}{L_p b} \left\{
                (2 L_f L_p + H_f C_p) \frac{2 \pi^{d/2} R_K}{\Gamma(d/2) r_K^3}
            \right\}, \\
     \constCorTwo^\sigma & = \frac{1}{L_p b} \left\{
                (2 L_f L_p + H_{\sigma^2} C_p) \frac{2 \pi^{d/2} R_K}{\Gamma(d/2) r_K^3} 
            \right\}.
  \end{align*}
  Analogously, we define constants $\constCorTwoPrime^f, \constCorTwoPrime^\sigma$.

  \begin{lemma}
  \label{lemma: sigma quadratic term bound}
    Suppose Assumptions~\ref{assumption: kernel assumptions}-\ref{assumption: density assumption} hold. Assume that
    \begin{enumerate}
      \item for the density $p(\xv)$ we have $p(\xv) \ge 2 L_p h \cdot \max \{b, \frac{2 \pi^{d/2} R_K}{\Gamma(d/2) r_K^2}\}$,
      \item the Euclidean distance $d(\xv, \partial \support)$ from $\xv$ to the bound of $\support$ is at least $b h$,
       \item for a function $\delta(\xv)$ we have:
       \begin{align*}
           \delta(\xv) \ge \frac{10 R_K \sigma^2(\xv)}{\logProb} + \frac{5 \constCorTwo^\sigma h^2}{p(\xv)} \left(1 + \frac{R_K}{\logProb}\right).
       \end{align*}
    \end{enumerate}
    When 
    \begin{align*}
      \PP & \left( 
          \left|
              (\labels - \meanLabels)^\T (\diag_{\weights} - \weights \weights^\T) (\labels - \meanLabels)
                  -
                  \sigma^2(\xv)
          \right|
          \ge \frac{2 \delta(\xv)}{5}
      \right)  \\
      & \le O(1) \cdot e^{-\Omega(n h^{d + 2} p(\xv))} + O(1) \cdot \exp \left(
          \frac{- \Omega(n h^d p(\xv) \cdot \delta(\xv))}{\sigma^2(\xv) + \frac{L_{\sigma^2} h}{e r_K}} 
          \cdot
          \min \left\{
              \frac{\delta(\xv) / 80}{\sigma^2(\xv) + \frac{\constCorTwo^\sigma h^2}{p(\xv)}}, \frac{1}{2}
          \right\}
      \right). 
    \end{align*}
    where constants inside $\Omega$ do not depend on $\xv$. Additionally, if 
    \begin{align*}
        \delta(\xv) \ge \frac{10 R_K \sigma^2(\xv)}{\logProb} + \constCorTwoPrime^\sigma h \left(1 + \frac{R_K}{\logProb}\right),
    \end{align*}
    then 
    \begin{align*}
      \PP & \left( 
          \left|
              (\labels - \meanLabels)^\T (\diag_{\weights} - \weights \weights^\T) (\labels - \meanLabels)
                  -
                  \sigma^2(\xv)
          \right|
          \ge \frac{2 \delta(\xv)}{5}
      \right)  \\
      & \le O(1) \cdot e^{-\Omega(n h^{d} p(\xv))} + O(1) \cdot \exp \left(
          \frac{- \Omega(n h^d p(\xv) \cdot \delta(\xv))}{\sigma^2(\xv) + \frac{L_{\sigma^2} h}{e r_K}} 
          \cdot
          \min \left\{
              \frac{\delta(\xv) / 80}{\sigma^2(\xv) + \constCorTwoPrime^\sigma h}, \frac{1}{2}
          \right\}
      \right). 
    \end{align*}
  \end{lemma}

  \begin{proof}
    First, we bound
    \begin{align*}
      \PP \left(
          \left| 
              \tr(\Sigma) - \sigma^2(\xv)
          \right|
          \ge 
          \frac{\delta(\xv)}{5}
      \right).
    \end{align*}
    By the definition of $\Sigma$, we have $\tr(\Sigma) = \sum_{i = 1}^n \sigma^2(X_i) \weights_i - \weights^\T \diag_{\variances} \weights$. Meanwhile, due to Corollary~\ref{corollary: weights bound}, we have
    \begin{equation*}
      \weights^\T \diag_{\variances} \weights = \sum_{i = 1}^n \sigma^2(X_i) \weights_i^2
      \le \frac{2 R_K}{\logProb} \left\{ \sigma^2(\xv) + \sum_{i = 1} \sigma^2(X_i) \weights_i - \sigma^2(\xv) \right\}
    \end{equation*}
    with probability $1 - e^{- \logProb / 26} = 1 - e^{- \Omega(n h^d p(\xv))}$. Under this event we have
    \begin{align*}
      \left|
          \tr(\Sigma) - \sigma^2(\xv)
      \right| \ge \frac{\delta(\xv)}{5}
      \implies
      \left( 
          1 + \frac{R_K}{\logProb}
      \right) \left|
          \sum_{i = 1}^n \sigma^2(X_i) \weights_i - \sigma^2(\xv)
      \right| \ge \frac{\delta(\xv)}{5} - \frac{2 R_K \sigma^2(\xv)}{\logProb}.
    \end{align*}
    If $\left( \frac{\delta(\xv)}{5} - \frac{2 R_K \sigma^2(\xv)}{\logProb} \right) / (1 + \frac{R_K}{\logProb}) \ge \constCorTwo^\sigma h^2 / p(\xv)$, we may apply Corollary~\ref{corollary: nhp concentration bound}. Thus, condition 3 of the lemma ensures that
    \begin{align}
    \label{eq: trace deviation}
      \PP \left(
          \left| 
              \tr(\Sigma) - \sigma^2(\xv)
          \right|
          \ge 
          \frac{\delta(\xv)}{5}
      \right) = \exp \left( - \Omega(n h^{d + 2} p(\xv)) \right).
    \end{align}

    Next, we will bound 
    \begin{align*}
      \PP \left( 
          \left|
              (\labels - \meanLabels)^\T (\diag_{\weights} - \weights \weights^\T) (\labels - \meanLabels)
                  -
                  \tr(\Sigma)
          \right|
          \ge \frac{\delta (\xv)}{5}
      \right)
    \end{align*}
    using Proposition~\ref{proposition: Hanson-Wright ineq}. First, 
    \begin{align*}
      \Vert \Sigma \Vert = \Vert \diag_{\variances} (\diag_{\weights} - \weights \weights^\T) \Vert \le \max_{i} \left\{ \sigma^2(X_i) \weights_i + \sum_{j} \sigma^2(X_i) \weights_i \weights_j \right\}
    \end{align*}
    due to Gershgorin's circle theorem. Thus, we have
    \begin{align*}
      \Vert \Sigma \Vert & \le 2 \max_i \sigma^2(X_i) \weights_i \\
      & \le 2 \max_i \sigma^2(\xv) \weights_i + 2 \max_i |\sigma^2(X_i) - \sigma^2(\xv) | \weights_i \\
      & \le 2 \sigma^2(\xv) \max_i \weights_i + 2 L_{\sigma^2} \frac{\max_i \Vert X_i - \xv \Vert K \left(\frac{X_i - \xv}{h} \right)}{\sum_{i = 1}^n K \left(\frac{X_i - \xv}{h} \right)} \\
      & \le 2 \sigma^2(\xv) \frac{2 R_K}{\logProb} + \frac{2 L_{\sigma^2} R_K h}{r_K e \logProb}
    \end{align*}
    with probability $1 - e^{- C n h^d p(\xv)}$. At the same time
    \begin{align*}
      \tr(\Sigma^2) & = \sum_{i = 1}^n \sigma^4(X_i) \weights_i^2 - 2 \weights^\T \diag_{\variances} \diag_{\weights} \diag_{\variances} \weights + (\weights^\T \diag_{\variances} \weights)^2 \\
      & = \sum_{i = 1}^n \sigma^4(X_i) \weights_i^2 - 2 \sum_{i = 1}^n \sigma^4(X_i) \weights_i^3 + \left(\sum_{i = 1}^n \sigma^2(X_i) \weights_i^2 \right)^2 \\
      & \le \sum_{i = 1}^n \sigma^4(X_i) \weights_i^2 + \left(\sum_{i = 1}^n \sigma^2(X_i) \weights_i^2 \right)^2
      \le  2 \sum_{i = 1}^n \sigma^4(X_i) \weights_i^2,
    \end{align*}
    since
    \begin{align*}
      \frac{\sum_{i = 1}^n \sigma^4(X_i) \weights_i^2}{\sum_{i = 1}^n \weights_i^2} 
      \ge \left( \frac{\sum_{i = 1}^n \sigma^2(X_i) \weights^2_i}{\sum_{i = 1}^n \weights_i^2}\right)^2,
    \end{align*}
    or, rearranging terms,
    \begin{align*}
      \sum_{i = 1}^n \sigma^2(X_i) \weights_i^2 \ge \sum_{i = 1}^n \sigma^2(X_i) \weights_i^2 \cdot \sum_{i = 1}^n \weights_i^2 \ge \left( \sum_{i = 1}^n \sigma^2(X_i) \weights^2_i \right)^2. 
    \end{align*}
    As previously, bounding
    \begin{align*}
      \max_{i \in [n]} \sigma^2(\xv) \weights_i \le \sigma^2(\xv) \frac{2 R_K}{\logProb} + \frac{L_{\sigma^2} R_K h}{r_K e \logProb},
    \end{align*}
    we obtain
    \begin{align*}
      \tr(\Sigma^2) & \le 2 \left\{ 
          \sigma^2(\xv) \frac{2 R_K}{\logProb} + \frac{L_{\sigma^2} R_K h}{r_K e \logProb}
      \right\} \sum_{i = 1}^n \sigma^2(X_i) \weights_i \\
      & \overset{\text{Cor.~\ref{corollary: nhp concentration bound}}}{\le} \frac{2 R_K}{\logProb} \left\{ \sigma^2(\xv) + \frac{L_{\sigma^2} h}{e r_K}  \right\} \left\{ \sigma^2(\xv) + \frac{\constCorTwo^\sigma h^2}{p(\xv)} \right\}
    \end{align*}
    with probability $1 - e^{-C n h^{d + 2} p(\xv)}$. Choose $t$ such that $16 \sqrt{\tr(\Sigma^2) t} \le \delta(\xv) / 5$ and $16 \Vert \Sigma \Vert t \le \delta(\xv) / 5$, e.g.
    \begin{align*}
      t = \frac{\delta(\xv)}{160 R_K} \cdot \frac{\logProb}{\sigma^2(\xv) + \frac{L_{\sigma^2} h}{e r_K}} \cdot
      \min \left\{
          \frac{\delta(\xv) / 80}{\sigma^2(\xv) + \frac{\constCorTwo^\sigma h^2}{p(\xv)}}, \frac{1}{2}
      \right\}.
    \end{align*}
    Thus, from Proposition~\ref{proposition: Hanson-Wright ineq} we infer
    \begin{align*}
      \PP \left( 
          \left|
              (\labels - \meanLabels)^\T (\diag_{\weights} - \weights \weights^\T) (\labels - \meanLabels)
                  -
                  \tr (\Sigma)
          \right|
          \ge \frac{\delta (\xv)}{5}
      \right)
      \le 2 e^{-t} + O(1) \cdot e^{-\Omega(n h^{d + 2} p(\xv))}.
    \end{align*}
    Combining the above with~\eqref{eq: trace deviation} finalizes the proof of the first part of the lemma. The second part can be obtained analogously using the second part of Corollary~\ref{corollary: nhp concentration bound}.
  \end{proof}

  Next, we analyze term~\eqref{eq: section 5.3, sigma_sq linear term}. 
  \begin{lemma}
  \label{lemma: linear term order}
    Suppose Assumptions~\ref{assumption: mean assumption}-\ref{assumption: density assumption} hold, $p(\xv) \ge 2 L_p h \cdot \max \{b, \frac{2 \pi^{d/2} R_K}{\Gamma(d/2) r_K^2}\}$ and $d(\xv, \partial \support) \ge b h$. Then it holds that
    \begin{align*}
      & \PP \left( 
            \left| 
              (\meanLabels - \labels)^\T (\diag_{\weights} - \weights \weights)^\T \meanLabels 
          \right| \ge \frac{\delta(\xv)}{5}
      \right)
      \le O(1) \cdot \exp \left(
          - \Omega ( n h^{d} p(\xv))
      \right) \\
       & \qquad +
      2  \exp \left(
          - \frac{\delta^2(\xv)}{50 \left(
              \sigma^2(\xv) + \frac{L_{\sigma^2} h}{e r_K}
          \right)}
          \cdot 
          \frac{\Omega(n h^{d - 2} p(\xv))}{2 R_K} \cdot \left(
              [\constCorTwoPrime^f]^2 + \frac{2 L_f^2}{r_K^2} \log^2 \frac{e^2 n R_K}{\logProb}
          \right)^{-1}
      \right).
    \end{align*}
  \end{lemma}

  \begin{proof}
    Notice, that
    $
        (\meanLabels - \labels)^\T (\diag_{\weights} - \weights \weights)^\T \meanLabels \mid X
    $
    is distributed as $\mathcal{N}(0, v^2)$ for
    \begin{align*}
      v^2 & =  \sum_{i = 1}^n \meanLabels_i^2 \sigma^2(X_i) \weights_i^2 - 2 (\meanLabels^\T \weights) \cdot \weights^\T \diag_{\variances} \diag_{\weights} \meanLabels + (\meanLabels^\T \weights)^2 (\weights^\T \diag_{\variances} \weights) \\
      & = \sum_{i = 1}^n (\meanLabels_i - \meanLabels^\T \weights)^2 \sigma^2(X_i) \weights_i^2.
    \end{align*}
    Hence,
    \begin{align}
      \PP \left(
           \left| 
              (\meanLabels - \labels)^\T (\diag_{\weights} - \weights \weights)^\T \meanLabels 
          \right|
           \ge 
           r v 
           \mid X
      \right)
      \le
      2 \exp \left( - \frac{r^2 }{2}\right). \label{eq: subgaussian inequality}
    \end{align}
    We can remove conditioning on $X$ for any fixed positive $r$. Then, bound
    \begin{align*}
      v^2 & \le (\max_{i} \sigma^2(X_i) \weights_i) \cdot \sum_{i = 1}^n \left(\meanLabels_i - \meanLabels^\T \weights \right)^2 \weights_i \\
      & \le 2 \left(\max_i \sigma^2(X_i) \weights_i \right) \cdot \left[ (\meanLabels^\T \weights - \meanY(\xv))^2 + \sum_{i = 1}^n (\meanLabels_i - \meanY(\xv))^2 \weights_i \right].
    \end{align*}
    Applying Corollary~\ref{corollary: weights bound} and Corollary~\ref{corollary: nhp concentration bound}, we obtain
    \begin{align*}
      \max_i \sigma^2(X_i) \weights_i & \le \frac{R_K}{\logProb} \left( \sigma^2(\xv) + \frac{L_{\sigma^2} h}{e r_K} \right), \\
      \bigl(\meanLabels^\T \weights - \meanY(\xv)\bigr)^2 & \le (\constCorTwoPrime^f h)^2
    \end{align*}
    with probability $1 - e^{- \Omega\left(n h^d p(\xv) \right)}$. 
    We bound the sum $\sum_{i = 1}^n \bigl(\meanLabels_i - \meanY(\xv)\bigr)^2 \weights_i$  using Proposition~\ref{proposition: deterministic bound}. Thus, with probability $1 - e^{-\Omega(n h^d p(\xv))}$ we obtain
    \begin{align*}
      \sum_{i = 1}^n \bigl(\meanLabels_i - \meanY(\xv)\bigr)^2 \weights_i 
      & \le 
      \frac{2 L^2_f h^2}{r_K^2}
      \left( 
          2 \vee \log \frac{n R_K}{\sum_{i = 1}^n K \left( \frac{X_i - \xv}{h} \right)} 
      \right)^2 \\
      & \le \frac{2 L_f^2 h^2}{r_K^2} \log^2 \frac{e^2 n R_K}{\sum_{i = 1}^n K \left( \frac{X_i - \xv}{h} \right)} \\
      & \le \frac{2 L_f^2 h^2}{r_K^2} \log^2 \frac{e^2 n R_K}{\logProb}.
    \end{align*}
    Consequently, with probability $1 - O(1) \cdot e^{-\Omega(n h^d p(\xv))}$ we have
    \begin{align*}
      v \le \sqrt{
          \frac{2 R_K h^2}{\logProb}
          \left( \sigma^2(\xv) + \frac{L_{\sigma^2} h}{e r_K} \right)
          \left( [\constCorTwoPrime^f]^2 + \frac{2 L_f^2}{r_K^2} \log^2 \frac{e^2 n R_K}{\logProb} \right)
      }.
    \end{align*}
    Choose $r$ such that $r v \le \frac{\delta(\xv)}{5}$ and apply~\eqref{eq: subgaussian inequality} to finalize the proof.
  \end{proof}

  Finally, we analyze term~\eqref{eq: section 5.3, sigma_sq constant term}.

  \begin{lemma}
  \label{lemma: constant term upper bound}
    Suppose $p(\xv) \ge 2 L_p b h$ and $d(\xv, \partial \support) \ge b h$. Then under Assumptions~\ref{assumption: mean assumption}-\ref{assumption: kernel assumptions}, we have
    \begin{align*}
        0 \le \meanLabels^\T (\diag_{\weights} - \weights \weights^\T) \meanLabels \le \frac{L_f^2 h^2}{r_K^2} \log^2 \frac{e^2 n R_K}{\logProb}
    \end{align*}
    with probability $1 - e^{-\Omega \left(n h^d p(\xv) \right)}$ where a constant in $\Omega(\cdot)$ does not depend on $\xv$.
  \end{lemma}

  \begin{proof}
    From the Gershgorin circle theorem, we infer that $\diag_{\weights} - \weights \weights^\T$ is non-negative defined. Then notice, that
    \begin{align*}
        & \quad \meanLabels^\T \left( \diag_{\weights} - \weights \weights^\T \right) \meanLabels = \sum_{i = 1}^n \meanLabels_i^2 \weights_i - (\meanLabels^\T \weights)^2 \\
        & = \sum_{i = 1}^n \bigl(\meanLabels_i^2 - \meanY(\xv)^2\bigr) \weights_i + (\meanY(\xv) - \meanLabels^\T \weights) (\meanY(\xv) + \meanLabels^\T \weights) \\
        & = \sum_{i = 1}^n \bigl(\meanLabels_i - \meanY(\xv)\bigr)^2 \weights_i + 2 \meanY(\xv) \sum_{i = 1}^n \bigl(\meanLabels_i - \meanY(\xv)\bigr) \weights_i 
        - (\meanY(\xv) - \meanLabels^\T \weights)^2 + 2 \meanY(\xv) (\meanY(\xv) - \meanLabels^\T \weights) \\
        & = \sum_{i = 1}^n \bigl(\meanLabels_i - \meanY(\xv)\bigr)^2 \weights_i - (\meanY(\xv) - \meanLabels^\T \weights)^2
        \le \sum_{i = 1}^n \bigl(\meanLabels_i - \meanY(\xv)\bigr)^2 \weights_i.
    \end{align*}
    The sum can be bounded via Proposition~\ref{proposition: deterministic bound}. Thus,
    \begin{align*}
        \meanLabels^\T \left( \diag_{\weights} - \weights \weights^\T \right) \meanLabels \le \frac{L_f^2 h^2}{r_K^2} \log^2 \frac{e^2 n R_K}{\sum_{i = 1}^n K \left( \frac{X_i - \xv}{h} \right)}.
    \end{align*}
    Application of Proposition~\ref{proposition: denominator lower bound} finalizes the proof.
  \end{proof}

  We summarize this section with the following corollary.
  \begin{corollary}
  \label{corollary: variance deviation bound}
    Suppose Assumptions~\ref{assumption: mean assumption}-\ref{assumption: kernel assumptions} hold. Assume that
    \begin{enumerate}
        \item the probability density $p(\xv)$ is at least $2 L_p h \cdot \max \{b, \frac{2 \pi^{d/2} R_K}{\Gamma(d/2) r_K^2}\}$,
        \item the distance $d(\xv, \partial \support) \ge b h$,
        \item for a function $\delta(\xv)$ we have
        \begin{align*}
            \delta(\xv) & \ge \frac{10 R_K \sigma^2(\xv)}{\logProb} +5 \frac{\constCorTwo^\sigma h^2}{p(\xv)} \left( 1 + \frac{R_K}{\logProb} \right), \\
            \delta(\xv) & \ge \frac{5 L_f^2 h^2}{r_K^2} \log^2 \frac{e^2 n R_K}{\logProb}.
        \end{align*}
    \end{enumerate}
    Then
    \begin{align*}
        & \PP \left(
            |\estimatorn{\sigma}^2(\xv) - \sigma^2(\xv)| \ge \delta(\xv)
        \right)
        \lesssim e^{- \Omega(n h^{d + 2} p(\xv))} \\
        & \qquad + \exp \left(
            \frac{- \Omega(n h^d p(\xv) \cdot \delta(\xv))}{\sigma^2(\xv) + \frac{L_{\sigma^2} h}{e r_K}} 
            \cdot
            \min \left\{
                \frac{\delta(\xv) / 80}{\sigma^2(\xv) + \constCorTwo^\sigma h}, \frac{1}{2}
            \right\}
        \right) \\
         & \qquad +
           \exp \left(
                - \frac{\delta^2(\xv)}{
                    \sigma^2(\xv) + \frac{L_{\sigma^2} h}{e r_K}
                }
                \cdot 
                \frac{\Omega(n h^{d - 2} p(\xv))}{2 R_K} \cdot \left(
                    [\constCorTwo^f]^2 + \frac{2 L_f^2}{r_K^2} \log^2 \frac{e^2 n R_K}{\logProb}
                \right)^{-1}
            \right).
    \end{align*}
    Additionally, if we have
    \begin{align*}
        \delta(\xv) & \ge \frac{10 R_K \sigma^2(\xv)}{\logProb} +5 \constCorTwoPrime^\sigma h \left( 1 + \frac{R_K}{\logProb} \right),
    \end{align*}
    then it holds that
    \begin{align*}
        & \PP \left(
            |\estimatorn{\sigma}^2(\xv) - \sigma^2(\xv)| \ge \delta(\xv)
        \right)
        \lesssim e^{- \Omega(n h^{d} p(\xv))} \\
        & \qquad + \exp \left(
            \frac{- \Omega(n h^d p(\xv) \cdot \delta(\xv))}{\sigma^2(\xv) + \frac{L_{\sigma^2} h}{e r_K}} 
            \cdot
            \min \left\{
                \frac{\delta(\xv) / 80}{\sigma^2(\xv) + \constCorTwo^\sigma h}, \frac{1}{2}
            \right\}
        \right) \\
         & \qquad +
           \exp \left(
                - \frac{\delta^2(\xv)}{
                    \sigma^2(\xv) + \frac{L_{\sigma^2} h}{e r_K}
                }
                \cdot 
                \frac{\Omega(n h^{d - 2} p(\xv))}{2 R_K} \cdot \left(
                    [\constCorTwo^f]^2 + \frac{2 L_f^2}{r_K^2} \log^2 \frac{e^2 n R_K}{\logProb}
                \right)^{-1}
            \right).
    \end{align*}
  \end{corollary}

  \begin{proof}
    We decompose $\estimatorn{\sigma}^2(\xv)$ as $\estimatorn{\sigma}^2(\xv) = \eqref{eq: section 5.3, sigma_sq quadratic term} + \eqref{eq: section 5.3, sigma_sq linear term} + \eqref{eq: section 5.3, sigma_sq constant term}$. Then
    \begin{align*}
        & \quad \PP \left( | \estimatorn{\sigma}^2(\xv) - \sigma^2(\xv)| \ge \delta(\xv) \right) \le \\
        & \le
        \PP \left( |\eqref{eq: section 5.3, sigma_sq quadratic term} - \sigma^2(\xv)| \ge \frac{2 \delta(\xv)}{5} \right)
        + \PP \left( 2 |\eqref{eq: section 5.3, sigma_sq linear term}| \ge \frac{2 \delta(\xv)}{5} \right) 
        + \PP \left( |\eqref{eq: section 5.3, sigma_sq constant term}| \ge \frac{\delta(\xv)}{5} \right).
    \end{align*}
    The first term and the second term can be bounded via Lemma~\ref{lemma: sigma quadratic term bound} and Lemma~\ref{lemma: linear term order} respectively. For the third term, we use Lemma~\ref{lemma: constant term upper bound} and the third condition of the corollary.
  \end{proof}

\subsection{Bias-Variance tradeoff for $L_p$-risk}
  \begin{proposition}
  \label{proposition: binomial expectation}
     For a binomial random variable $B(n, q)$ it holds that
     \begin{align*}
         \EE \frac{1}{(r + B(n, q))^r} \le \frac{1}{(n + 1)^r q^r}.
     \end{align*}
  \end{proposition}

  \begin{proof}
     For a binomial random variable $B(n, q)$ it holds
    \begin{align*}
        & \quad \EE \frac{1}{(r + B(n, q))^r} = \sum_{k = 0}^n \frac{1}{(r + k)^r} \binom{n}{k} q^k (1 - q)^{n - k} \\
        & \le \sum_{k = 0}^n \prod_{j = 1}^r \frac{1}{(j + k)} \binom{n}{k} q^k (1 - q)^{n - k} 
        = \prod_{j = 1}^r \frac{1}{(n + j)} \sum_{k = 0}^n \binom{n + r}{k + r} q^k (1 - q)^{n - k} \\
        & \le \frac{1}{(n + 1)^r q^r} \sum_{k = 0}^{n} \binom{n + r}{k + r} q^{k + r} (1 - q)^{n - k}
        = \frac{1}{(n + 1)^r q^r} \sum_{k' = r}^{n + r} \binom{n + r}{k'} q^{k'} (1 - q)^{n + r - k'} \\
        & \le \frac{1}{(n + 1)^r q^r} (q + (1 - q))^{n + r}
        = \frac{1}{(n + 1)^r q^r}.
    \end{align*}
  \end{proof}

  \begin{lemma}
  \label{lemma: kernel reverse sum}
    Suppose that Assumptions~\ref{assumption: kernel assumptions},\ref{assumption: density assumption} hold and $p(\xv) \ge 2 L_p h \cdot \max \{b, \frac{2 \pi^{d/2} R_K}{\Gamma(d/2) r_K^2}\}$, $d(\xv, \partial S) \ge b h$. Then
    \begin{align*}
        \EE \left[
            \frac{
                1
            }{
                \sum_{i = 1}^n K \left( \frac{X_i - \xv}{h} \right)
            } 
            \indicator{\estimatorn{p}(\xv) \ge \frac{r a}{n h^d}}
        \right]^r
        \le 
        \left(
            \frac{a \omega_d b^d}{4} n h^d p(\xv)
        \right)^{-r}.
    \end{align*}
  \end{lemma}

  \begin{proof}
    By the definition of $\estimatorn{p}(\xv)$: $\estimatorn{p}(\xv) \ge a r / (n h^d)$ implies
    \begin{align*}
      \sum_{i = 1}^n K \left( \frac{X_i - \xv}{h} \right) \ge a r.
    \end{align*}
    In particular, it means that
    \begin{align*}
      \sum_{i = 1}^n K \left(\frac{X_i - \xv}{h} \right) \ge \frac{a r + \sum_{i = 1}^n K \left( \frac{X_i - \xv}{h} \right)}{2}.
    \end{align*}
    Thus,
    \begin{align}
    \label{eq: revrese sum lower bound}
      \frac{1}{
          \sum_{i = 1}^n K \left(\frac{X_i - \xv}{h} \right)
      }
      \le 
      \frac{2}{
          a r + \sum_{i = 1}^n K \left( \frac{X_i - \xv}{h} \right)
      } 
      \le
      \frac{2}{a  r + a \sum_{i = 1}^n \indicator{X_i \in \ball_{h b}(\xv)}}
    \end{align}
    due to Assumption~\ref{assumption: kernel assumptions}. The sum $\sum_{i = 1}^n \indicator{X_i \in \ball_{h b}(\xv)}$ is a binomial random variable. Due to Proposition~\ref{proposition: binomial expectation}, we have
    \begin{align*}
      \EE \frac{1}{
          \left(r + \sum_{i = 1}^n \indicator{X_i \in \ball_{h b}(\xv)}\right)^r
      }
      \le 
      \left\{(n + 1) \PP (X_1 \in \ball_{b h} (\xv)) \right\}^{-r}.
    \end{align*}
    Using~\eqref{eq: revrese sum lower bound} and the above, we obtain
    \begin{align}
    \label{eq: expectation power bound}
      \EE \left[
              \frac{
                  1
              }{
                  \sum_{i = 1}^n K \left( \frac{X_i - \xv}{h} \right)
              } 
              \indicator{\estimatorn{p}(\xv) \ge \frac{r a}{n h^d}}
          \right]^r
          \le
          \left\{\frac{a}{2} (n + 1) \PP (X_1 \in \ball_{b h} (\xv)) \right\}^{-r}.
    \end{align}
    Finally,
    \begin{align*}
      \int_{\ball_{b h} (\xv)} p(\yv) d \mu(\yv) \ge \left(p(\xv) - L_p b h \right) \int_{\ball_{b h} (\xv)} d \mu(\yv) = \mu(\ball_{b h}) \left( p(\xv) - L_p b h \right).
    \end{align*}
    Since $p(\xv) \ge 2 L_p b h$, we have $p(\xv) - L_p b h \ge p(\xv) / 2$, and, thus,
    \begin{align*}
      \PP (X_1 \in \ball_{b h} (\xv)) \ge \frac{p(\xv) \omega_d b^d}{2} h^d.
    \end{align*}
    Combining the above with~\eqref{eq: expectation power bound}, we infer the statement.
  \end{proof}

  \begin{lemma}
  \label{lemma: lp bias variance tradeoff}
    Suppose Assumptions~\ref{assumption: mean assumption}-\ref{assumption: density assumption} hold and assume that $p(\xv) \ge 2 L_p b h$ and $d(\xv, \partial S) \ge b h$. Define
    \begin{align*}
      \mathtt{c}_t & = 
              1 - \frac{L_p h }{p(\xv)} \cdot \frac{2\pi^{d/2} R_K}{\Gamma(d/2) r_K^2}, \\
          \mathtt{c}_h & = \frac{2 L_f L_p + H_f C_p}{2} 
              \cdot 
              \frac{4 \pi^{d/2} R_K}{\Gamma(d/2) r_K^3}, \\
          \mathtt{c}'_t & = \frac{2 \pi^{d/2} R_K^2}{\Gamma(d/2) r_k} + \frac{2 h L_p}{p(\xv)} \cdot \frac{4 \pi^{d/2} R_k}{\Gamma(d/2) r_k^2}, \\
          \mathtt{c}_h' & = L^2_f \frac{\pi^{d/2} R_K^2}{\Gamma(d/2) r_K^3}
              +
              \frac{2 h L^2_f L_p}{p(\xv)} 
              \cdot
              \frac{
                  4 \pi^{d/2} R_K^2
              }{3 \Gamma(d/2) r_K^4},
    \end{align*}
    Assume $\mathtt{c}_t > 0$. Then
    \begin{align*}
      & \quad \EE \left| 
          \estimatorn{f}(\xv) - \meanY(\xv)
      \right|^{r} \indicator{\estimatorn{p}(\xv) \ge \frac{a r}{n h^d}}
      \le \frac{\Gamma \left( \frac{r + 1}{2}\right)}{2 \sqrt{\pi}} \left( 
              \frac{32 R_K \left\{ 
                  \sigma^2(\xv) + \frac{2 L_{\sigma} h}{r_K} \log \frac{e n R_K}{a r}
              \right\}}{
                  a \omega_d b^d \cdot n h^d p(\xv)
              }
          \right)^{r/2} \\
          & + \left(\frac{4 h^2}{p(\xv)} \cdot \frac{\mathtt{c}_h}{\mathtt{c}_t} \right)^r + \frac{r 2^{3 r/2}\Gamma \left(\frac{r + 1}{2} \right)}{\sqrt{\pi}} \cdot h^r
          \times
          \left( 
              \frac{
                  \left\{ \mathtt{c}_t' + \frac{\mathtt{c}_t R_K}{3} \right\} \cdot 4^{1/r} \frac{L_f^2}{r_K^2} \log^2 \frac{e^r n R_K}{a r} + \mathtt{c}_h' + \frac{R_K}{3} \frac{2^{1/r} L_f^2}{e r_K ^2 } \log \frac{e^r n R_K}{a r}
              }{n h^d p(\xv) \mathtt{c}_t^2} \right)^{r/2}.
    \end{align*}
  \end{lemma}

  \begin{proof}
    First, we have
    \begin{align}
        \EE \left| 
            \estimatorn{f}(\xv) - \meanY(\xv)
        \right|^{r} \indicator{\estimatorn{p}(\xv) \ge \frac{a r}{n}}
        & \le
        2^{r - 1} \EE \left| \estimatorn{f}(\xv) - \EE [\estimatorn{f}(\xv) \mid X] \right|^{r} \indicator{\estimatorn{p}(\xv) \ge \frac{a r}{n h^d}} \nonumber \\
        & \quad + 2^{r - 1} \EE \left| \EE [\estimatorn{f}(\xv) \mid X] - \meanY(\xv) \right|^{r} \indicator{\estimatorn{p}(\xv) \ge \frac{a r}{n h^d}}. \label{eq: bias-variance decomposition}
    \end{align}
    We analyze each term separately. For the first term, we have
    \begin{align*}
        \EE \left| \estimatorn{f}(\xv) - \EE [\estimatorn{f}(\xv) \mid X] \right|^{r} \indicator{\estimatorn{p}(\xv) \ge \frac{a r}{n h^d}} 
        = 
        \EE \left\{ \indicator{\estimatorn{p}(\xv) \ge \frac{a r}{n h^d}} \cdot \EE \left[ \left| \estimatorn{f}(\xv) - \EE [\estimatorn{f}(\xv) \mid X] \right|^{r} \mid X \right] \right\}.
    \end{align*}
    Conditioned on $X$, the random variable $\estimatorn{f}(\xv) - \EE [\estimatorn{f}(\xv) \mid X]$ is Gaussian with zero mean and variance $v^2 = \sum_{i = 1}^n \weights_i^2 \sigma^2(X_i)$. Thus,
    \begin{align}
    \label{eq: bias-variance lemma term 1}
        \EE \left| \estimatorn{f}(\xv) - \EE [\estimatorn{f}(\xv) \mid X] \right|^{r} \indicator{\estimatorn{p}(\xv) \ge \frac{a r}{n h^d}} = 2^{r/2} \frac{\Gamma \left( \frac{r + 1}{2}\right)}{\sqrt{\pi}} \EE v^r \indicator{\estimatorn{p}(\xv) \ge \frac{a r}{n h^d}}.
    \end{align}
    We bound
    \begin{align*}
        v^2 = \sum_{i = 1}^n \weights_i^2 \sigma^2(X_i) & \le (\max_{i} \weights_i) \sum_{i = 1}^n \weights_i \sigma^2(X_i) \\
        & \le \frac{R_K}{\sum_{i = 1}^n K \left( \frac{X_i - \xv}{h} \right)} \left\{ 
            \sigma^2(\xv) + \sum_{i = 1}^n \weights_i |\sigma^2(X_i) - \sigma^2(\xv)|
        \right\}.
    \end{align*}
    From Proposition~\ref{proposition: deterministic bound}, we have
    \begin{align*}
        v^2 \le \frac{R_K}{\sum_{i = 1}^n K \left( \frac{X_i - \xv}{h} \right)} \left\{ 
            \sigma^2(\xv) + \frac{2 L_{\sigma} h}{r_K} \log \frac{e n R_K}{\sum_{i = 1}^n K \left( \frac{X_i - \xv}{h}\right)}
        \right\}.
    \end{align*}
    Substituting the above into~\eqref{eq: bias-variance lemma term 1}, we obtain
    \begin{align*}
        \eqref{eq: bias-variance lemma term 1} & \le 2^{r/2} \frac{\Gamma \left( \frac{r + 1}{2}\right)}{\sqrt{\pi}} \EE \indicator{\estimatorn{p}(\xv)  \ge \frac{a r}{n h^d}} \times \\
        & \quad \times \left(
            \frac{R_K}{\sum_{i = 1}^n K \left( \frac{X_i - \xv}{h} \right)} \left\{ 
            \sigma^2(\xv) + \frac{2 L_{\sigma} h}{r_K} \log \frac{e n R_K}{\sum_{i = 1}^n K \left( \frac{X_i - \xv}{h}\right)}
        \right\}
        \right)^{r/2} \\
        & \le 2^{r/2} \frac{\Gamma \left( \frac{r + 1}{2}\right)}{\sqrt{\pi}} \EE \indicator{\estimatorn{p}(\xv)  \ge \frac{a r}{n h^d}} \times \\
        & \quad \times \left(
            \frac{R_K}{\sum_{i = 1}^n K \left( \frac{X_i - \xv}{h} \right)} \left\{ 
            \sigma^2(\xv) + \frac{2 L_{\sigma} h}{r_K} \log \frac{e n R_K}{a r}
        \right\}
        \right)^{r/2} \\
        & \le 2^{r/2} \frac{\Gamma \left( \frac{r + 1}{2}\right)}{\sqrt{\pi}} \left( R_K \left\{ 
            \sigma^2(\xv) + \frac{2 L_{\sigma} h}{r_K} \log \frac{e n R_K}{a r}
        \right\} \right)^{r/2} \EE \left( 
            \frac{\indicator{\estimatorn{p}(\xv) \ge \frac{a r}{n h^d}}}{\sum_{i = 1}^n K \left( \frac{X_i - \xv}{h} \right)}
        \right)^{r/2}.
    \end{align*}
    Applying Lemma~\ref{lemma: kernel reverse sum} to bound the expectation, we obtain the variance term:
    \begin{align}
    \label{eq: variance term upper bound}
        \EE \left| \estimatorn{f}(\xv) - \EE [\estimatorn{f}(\xv) \mid X] \right|^{r} \indicator{\estimatorn{p}(\xv) \ge \frac{a r}{n h^d}}
        \le 
        \frac{\Gamma \left( \frac{r + 1}{2}\right)}{\sqrt{\pi}} \left( 
            \frac{8 R_K \left\{ 
                \sigma^2(\xv) + \frac{2 L_{\sigma} h}{r_K} \log \frac{e n R_K}{a r}
            \right\}}{
                a \omega_d b^d \cdot n h^d p(\xv)
            }
        \right)^{r/2}.
    \end{align}
    Analysis of the bias term is much simpler. We have
    \begin{align*}
        \EE \left| \EE [\estimatorn{f}(\xv) \mid X] - \meanY(\xv) \right|^{r} & \indicator{\estimatorn{p}(\xv) \ge \frac{a r}{n h^d}} \\ 
        & = \int_{0}^{\infty} \PP \left( \left| \sum_{i = 1}^n \weights_i \meanY(X_i) - \meanY(\xv) \right|^r \ge t \text{ and } \estimatorn{p}(\xv) \ge \frac{r a}{n h^d} \right) dt. \label{eq: bias term upper bound}
    \end{align*}
    Since $ \left| \sum_{i = 1}^n \weights_i \meanY(X_i) - \meanY(\xv) \right|^r \le \sum_{i = 1}^n \left| \meanY(X_i) - \meanY(\xv) \right|^r \weights_i$ from the Jensen inequality, we obtain the following:
    \begin{align*}
         \left| \sum_{i = 1}^n \weights_i \meanY(X_i) - \meanY(\xv) \right|^r \le \frac{2 L^r h^r}{r_K^r} \log^r \frac{e^r n R_K}{\sum_{i = 1}^n K \left( \frac{\xv - X_i}{h} \right)}
    \end{align*}
    via Proposition~\ref{proposition: deterministic bound}. If $\estimatorn{p}(\xv) \ge \frac{r a}{n h^d}$, the sum $\sum_{i = 1}^n K \left( \frac{\xv - X_i}{h} \right)$ is at least $a r$, and, consequently,
    \begin{align}
        \EE \left| \EE [\estimatorn{f}(\xv) \mid X] - \meanY(\xv) \right|^{r} & \indicator{\estimatorn{p}(\xv) \ge \frac{a r}{n h^d}} \nonumber \\ 
        & \le \int_{0}^{
            \frac{2 L^r h^r}{r_K^r} \log^r \frac{e^r n R_K}{a r}
        } \PP \left(  \left| \sum_{i = 1}^n \weights_i \meanY(X_i) - \meanY(\xv) \right|^r \ge t \right) dt. 
    \end{align}
    Define
    \begin{equation*}
        t_0 = \left(\frac{h^2}{p(\xv)} \cdot \frac{\mathtt{c}_h}{\mathtt{c}_t} \right)^r, \quad
        t_{\max} = \frac{2 L^r h^r}{r_K^r} \log^r \frac{e^r n R_K}{a r}.
    \end{equation*}
    According to Lemma~\ref{proposition: deviations of kernal estimator}, for $t \ge t_0$ we have
    \begin{align*}
        \PP \left(  \left| \sum_{i = 1}^n \weights_i \meanY(X_i) - \meanY(\xv) \right|^r \ge t \right) \le 2 \exp \left(
            - \frac{1}{2}
            \frac{
                n h^d p(\xv) \left(\mathtt{c}_t t^{1/p} - \mathtt{c}_h \frac{h^2}{p(\xv)} \right)^2
            }{
                \mathtt{c}_t' t_{\max}^{2/r} + \mathtt{c}_h' h^2 + \frac{R_K}{3} \left(\frac{L_f h}{e r_K} + t_{\max}^{1/r} \right) c_t t_{\max}^{1/r}
            }
        \right).
    \end{align*}
    From the above, we obtain
    \begin{align}
         \eqref{eq: integral upper bound} & \le \int_{0}^{t_0} dt + 2 \int_{t_0}^{t_{\max}} \exp \left(
            - \frac{1}{2}
            \frac{
                n h^d p(\xv) \left(\mathtt{c}_t t^{1/p} - \mathtt{c}_h \frac{h^2}{p(\xv)} \right)^2
            }{
                \mathtt{c}_t' t_{\max}^{2/r} + \mathtt{c}_h' h^2 + \frac{R_K}{3} \left(\frac{L_f h}{e r_K} + t_{\max}^{1/r} \right) c_t t_{\max}^{1/r}
            }
        \right) dt \nonumber \\
        & = t_0 + 2 r \int_{0}^{t_{\max}^{1/r}} y^{r - 1} \exp \left(
            - \frac{1}{2}
            \frac{
                n h^d p(\xv) \left(\mathtt{c}_t y - \mathtt{c}_h \frac{h^2}{p(\xv)} \right)^2
            }{
                \mathtt{c}_t' t_{\max}^{2/r} + \mathtt{c}_h' h^2 + \frac{R_K}{3} \left(\frac{L_f h}{e r_K} + t_{\max}^{1/r} \right) c_t t_{\max}^{1/r}
            }
        \right) dy \nonumber \\
        & \le t_0 + \frac{r 2^{r/2 + 1}\Gamma \left(\frac{r + 1}{2} \right)}{\sqrt{\pi}} \left( \frac{\mathtt{c}_t' t_{\max}^{2/r} + \mathtt{c}_h' h^2 + \frac{R_K}{3} \left(\frac{L_f h}{e r_K} + t_{\max}^{1/r} \right) c_t t_{\max}^{1/r}}{n h^d p(\xv) \mathtt{c}_t^2} \right)^{p/2}.
      \label{eq: integral upper bound}
    \end{align}
    Applying the bounds~\eqref{eq: variance term upper bound},~\eqref{eq: bias term upper bound} to the sum~\eqref{eq: bias-variance decomposition}, we obtain the statement.
  \end{proof}

\section{Proof of Theorem~\ref{theorem: finite-sample thm}}
\label{sec:main_proof}
  We require one additional proposition.
  \begin{proposition}
      \label{proposition: density upper bound}
      Suppose Assumptions~\ref{assumption: kernel assumptions}-\ref{assumption: density assumption} hold. Assume $p(\xv) \ge 2 L_p b h$ and $d(\xv, \partial \support) \ge $. Then
      \begin{align*}
          \PP \left ( 
                \frac{1}{n h^d} \sum_{i = 1}^n K \left ( \frac{X_i - \xv}{h} \right ) \ge 2 p(\xv) + L_p h \int_{\RR^d} \Vert \tv \Vert K(\tv) d \mu(\tv)
          \right ) \le \exp \left (- \logProbBelow \right ), \\
      \end{align*}
      where
      \begin{align*}
          \logProbBelow = 2 n h^d p(\xv) \cdot \left \{
                R_K/3 + \frac{\pi^{d/2} R_K^2}{\Gamma(d / 2) r_K} + \frac{\pi^{d/2} R_K^2}{b \Gamma(d/2)  r_K^2 } 
          \right \}.
      \end{align*}
  \end{proposition}

  \begin{proof}
      First, we study the mathematical expectation of $\estimatorn{p}(\xv) = \frac{1}{n h^d} \sum_{i = 1}^n K \left ( \frac{X_i - \xv}{h} \right )$. We have
      \begin{align*}
          \EE K \left ( \frac{X_i - \xv}{h} \right ) & = h^d \int_{\RR^n} K(\tv) p(\xv + h \tv ) d \mu(\tv) \\
          & \le h^d \left \{ \int_{\RR^d} K(\tv) p(\xv) d \mu(\tv) + h L_p \int_{\RR^d} K(\tv) \Vert \tv \Vert d \mu(\tv)  \right \}\\
          & = h^d p(\xv) + L_p h^{d + 1} \int_{\RR^d} \Vert \tv \Vert K(\tv) d \mu(\tv).
      \end{align*}
      Next, we have
      \begin{align*}
          \Var K \left( \frac{X_i - \xv}{h} \right) & \le \EE K^2 \left( \frac{X_i - \xv}{h} \right)
          = h^d \int_{\RR^d} K^2(\tv) p(\xv + h \tv) d \mu(\tv) \\
          & \le h^d \int_{\RR^d} K^2(\tv) p(\xv) d \mu(\tv) + L_p h^{d + 1} \int K^2(\tv) \Vert \tv \Vert d \mu(\tv).
      \end{align*}
      Finally, we have $ (n h^d)^{-1} K \left ( \frac{\xv - X_i}{h} \right ) \le (n h^d)^{-1} R_K$. Consequently, we may apply the Bernstein inequality to bound  $\estimatorn{p}(\xv)$ in probability:
      \begin{align*}
          & \PP \left (
                \estimatorn{p}(\xv) \ge p(\xv) + L_p h \int_{\RR^d} \Vert \tv \Vert K(\tv) d \mu(\tv) + t
          \right ) \\
          & \quad \le \exp \left (
            - \frac{t^2/2}{\left \{ p(\xv) \int_{\RR^d} K^2(\tv) d \mu(\tv) + L_p h \int K^2(\tv) \Vert \tv \Vert d \mu(\tv) \right \} / (n h^d) + R_k (n h^{d})^{-1} t /3 }
          \right ).
      \end{align*}

      Set $t = p(\xv)$. Then,
      \begin{align*}
          & \PP \left (
                \estimatorn{p}(\xv) \ge 2 p(\xv) + L_p h \int_{\RR^d} \Vert \tv \Vert K(\tv) d \mu(\tv)
          \right ) \\
          & \quad \le \exp \left ( - \frac{n h^d p(\xv) / 2}{
            \left \{ \int_{\RR^d} K^2(\tv) d \mu(\tv) + L_p h / p(\xv) \int K^2(\tv) \Vert \tv \Vert d \mu(\tv) \right \} + R_K/3
          } \right ).
      \end{align*}
      The condition of the proposition ensures that $h / p(\xv) \le 1 / (2 L_p b)$. Combining it with bounds on integrals provided by Proposition~\ref{proposition: bounds on maxima and integrals}, we simplify the denominator and finalize the proof.

  \end{proof}

  Now, we prove Theorem~\ref{theorem: finite-sample thm}.
  \begin{proof}
    Due to Proposition~\ref{proposition: excess risk decomposition}, we have
    \begin{align}
    \label{eq: excess risk decomposition}
        \EE_{\mathcal{D}} \risk_\lambda(\xv) - \risk^*_\lambda(\xv) = \EE \left[ \left( \meanY(\xv) - \estimatorn{f}(\xv) \right)^2 \indicator{\estimatorn{a}(\xv) = 0} \right] + \Delta(\xv) \cdot \PP \left(
            \estimatorn{a}(\xv) \neq a^*(\xv)
        \right).
    \end{align}

    We consider two cases.
    
    \textbf{Case 1.} If $\sigma^2(\xv) \ge \lambda$, then
    \begin{align*}
        \indicator{\estimatorn{a}(\xv) = 0} = \indicator{\estimatorn{p}(\xv) \ge \frac{4 a }{n h^d}} \cdot \indicator{\estimatorn{a}(\xv) \neq a^*(\xv)},
    \end{align*}
    consequently,
    \begin{align*}
         \EE \left[ \left( \meanY(\xv) - \estimatorn{f}(\xv) \right)^2 \indicator{\estimatorn{a}(\xv) = 0} \right] \le \sqrt{\EE \left[ \left( \meanY(\xv) - \estimatorn{f}(\xv) \right)^4 \indicator{\estimatorn{p}(\xv) \ge \frac{4 a}{n h^d}} \right]} \cdot \PP^{1/2} \left(
            \estimatorn{a}(\xv) \neq a^*(\xv)
        \right).
    \end{align*}
    Hence,
    \begin{align*}
        \EE_{\mathcal{D}}\risk_\lambda(\xv) - \risk^*_\lambda(\xv) \le \left\{ \sqrt{\EE \left[ \left( \meanY(\xv) - \estimatorn{f}(\xv) \right)^4 \indicator{\estimatorn{p}(\xv) \ge \frac{4 a}{n h^d}} \right]} + \Delta(\xv) \right\} \cdot \PP^{1/2} \left(
            \estimatorn{\alpha}(\xv) \neq \alpha(\xv)
        \right).
    \end{align*}
    Using $\sqrt{\sum_{i = 1}^n a_i} \le \sum_{i = 1}^n \sqrt{a_i}$ for any sequence of positive numbers $a_1, \ldots, a_n$, we bound the square root of the expectation with Lemma~\ref{lemma: lp bias variance tradeoff}. We may consider $\mathtt{c}_t, \mathtt{c}_h, \mathtt{c}'_t, \mathtt{c}_h'$ as constants since the conditions of the theorem bound $h / p(\xv)$ and $h$. At the same time, $\estimatorn{\alpha}(\xv) \neq \alpha(\xv)$ implies
    \begin{align*}
        \estimatorn{\sigma}^2(\xv) \le \lambda \left[ 
                1 - \frac{
                    \sqrt{2} \Vert K \Vert_2 z_{1 - \beta}
                }{
                    \sqrt{n h^d \estimatorn{p}(\xv)}
                }
            \right] \quad \text{ or } \quad \estimatorn{p}(\xv) < \frac{4 a}{n h^d}.
    \end{align*}
    Since $p(\xv) \ge 8/ (\omega_d b^d n h^d)$, we have $\estimatorn{p}(\xv) < \frac{4 a}{n h^d}$ with probability $e^{-\logProb/26}$ due to Proposition~\ref{proposition: denominator lower bound}. At the same time, $\estimatorn{p}(\xv) < 2 p(\xv) + L_p h \int_{\RR^d} \Vert \tv \Vert K(\tv) d \mu(\tv)$ with probability $1 - \exp \left (- \Omega(n h^d p(\xv)) \right )$. Consequently,
    \begin{align*}
        \PP \left ( 
            \estimatorn{\alpha}(\xv) \neq \alpha(\xv)
        \right ) & \le \PP \left (
            \estimatorn{\sigma}^2(\xv) \le \lambda \left[ 
                1 - \frac{
                    \sqrt{2} \Vert K \Vert_2 z_{1 - \beta}
                }{
                    \sqrt{n h^d \left (2 p(\xv) + L_p h \int_{\RR^d} \Vert \tv \Vert K(\tv) d \mu(\tv) \right )}
                }
            \right] 
        \right ) \\
        & \quad + 2 \exp \left ( - \Omega(n h^d p(\xv)) \right ) \\
        & = \PP \left (
            \estimatorn{\sigma}^2(\xv) - \sigma^2(\xv) \le - \Delta(\xv) - \frac{
                    \sqrt{2} \lambda \Vert K \Vert_2 z_{1 - \beta}
                }{
                    \sqrt{n h^d \left (2 p(\xv) + L_p h \int_{\RR^d} \Vert \tv \Vert K(\tv) d \mu(\tv) \right )}
                }
        \right ) \\
        & \quad + 2 \exp \left ( - \Omega(n h^d p(\xv)) \right ).
    \end{align*}
    Thus, we may apply Corollary~\ref{corollary: variance deviation bound} with $\delta(\xv) = \Delta(\xv) + \frac{
                    \sqrt{2} \lambda \Vert K \Vert_2 z_{1 - \beta}
                }{
                    \sqrt{n h^d \left (2 p(\xv) + L_p h\int_{\RR^d} \Vert \tv \Vert K(\tv) d \mu(\tv) \right )}
                }$ and obtain the first part of the statement.

    \textbf{Case 2.} If $\sigma^2(\xv) \le \lambda$, we bound
    \begin{align*}
        \risk(\xv) - \risk^*(\xv) \le \EE \left[ \left( \meanY(\xv) - \estimatorn{f}(\xv) \right)^2 \indicator{\estimatorn{p}(\xv) \ge \frac{4 a}{n h^d}} \right] + \Delta(\xv) \cdot \PP \left(
            \estimatorn{\alpha}(\xv) \neq \alpha(\xv)
        \right).
    \end{align*}
    The expectation can be bounded via Lemma~\ref{lemma: lp bias variance tradeoff}. The event $\estimatorn{\alpha}(\xv) \neq \alpha(\xv)$ means that
    \begin{align*}
        \estimatorn{\sigma}^2(\xv) \ge \lambda \left[ 
                1 - \frac{
                    \sqrt{2} (4 \pi)^{- d/4} z_{1 - \beta}
                }{
                    \sqrt{n h^d \estimatorn{p}(\xv)}
                }
            \right].
    \end{align*}
    Consider the event $\estimatorn{p}(\xv) \ge \alpha b^d \omega_d \cdot p(\xv) / 2$. Its complement has probability $\exp \left( -\Omega(n h^d p(\xv)) \right)$ due to Proposition~\ref{proposition: denominator lower bound}. Thus, we just consider the event
    \begin{align*}
        \estimatorn{\sigma}^2(\xv) \ge \lambda \left[ 
                1 - \frac{
                    2 (4 \pi)^{- d/4} z_{1 - \beta}
                }{
                    \sqrt{a b^d \omega_d \cdot p(\xv) \cdot n h^d }
                }
            \right].
    \end{align*}
    It implies
    \begin{align*}
        \estimatorn{\sigma}^2(\xv) - \sigma^2(\xv) \ge \Delta(\xv) - \frac{
                    2 \lambda (4 \pi)^{- d/4} z_{1 - \beta}
                }{
                    \sqrt{a b^d \omega_d \cdot p(\xv) \cdot n h^d }
                },
    \end{align*}
    and bounding its probability with Corollary~\ref{corollary: variance deviation bound}, we finalize case 2.

    For any $\sigma^2(\xv)$, we may bound the probability $\PP\bigl(\estimatorn{\alpha}(\xv) \neq \alpha(\xv)\bigr)$ by one. The remaining term can be bounded via Lemma~\ref{lemma: lp bias variance tradeoff}. That finalizes our proof.
  \end{proof}

\section{Additional experiments}

\subsection{Synthetic data}

\subsubsection{Acceptance probability}
\label{sec:suppl_acc_prob}

  On Figure~\ref{fig:acc_vs_x_supp} we present additional experiments for acceptance experiments with different combinations of data distribution and standard deviation. For normal data we see the expected decline on the left side due to lower $\estimatorn{p}(x)$.

  \begin{figure}[htbp]
  \floatconts
    {fig:acc_vs_x_supp}
    {\caption{Acceptance probability. We sample 100 datasets of sizes $[100, 200, 500, 1000]$ and for points in $[-2, 2]$ calculate ratio of points where our method accepts the regression result. Optimal bandwidth is selected using leave-one-out  cross-validation with mean squared error.}}
    {%
        \subfigure[$X \sim \mathcal{N}(0,1), \: \sigma(x) = \mathtt{sigmoid}(x)$]{%
            \includegraphics[width=.45\linewidth]{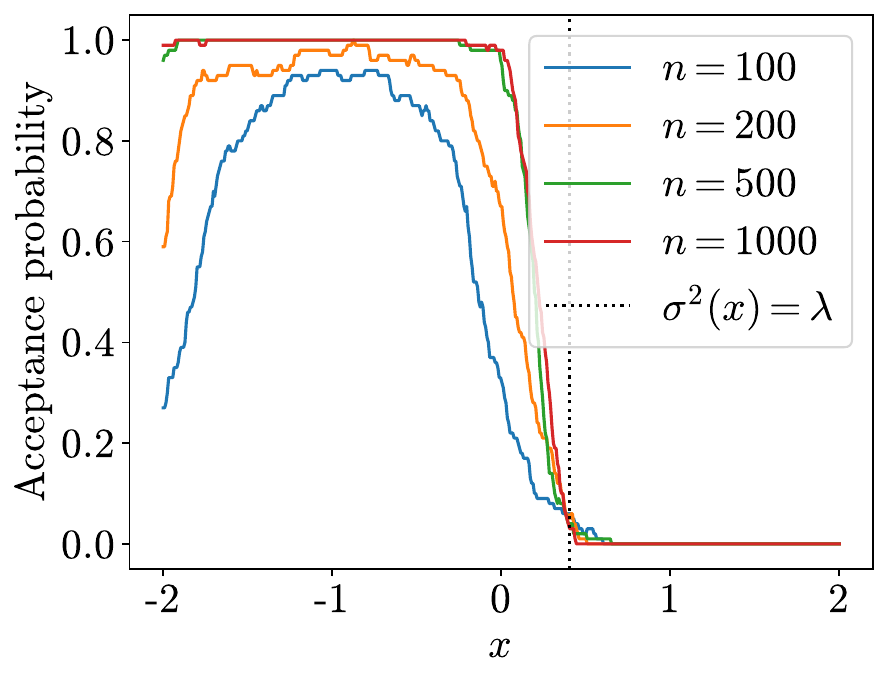}
        }\qquad 
        \subfigure[$X \sim \mathcal{N}(0,1), \: \sigma(x) = \mathtt{Heaviside}(x)$]{%
            \includegraphics[width=.45\linewidth]{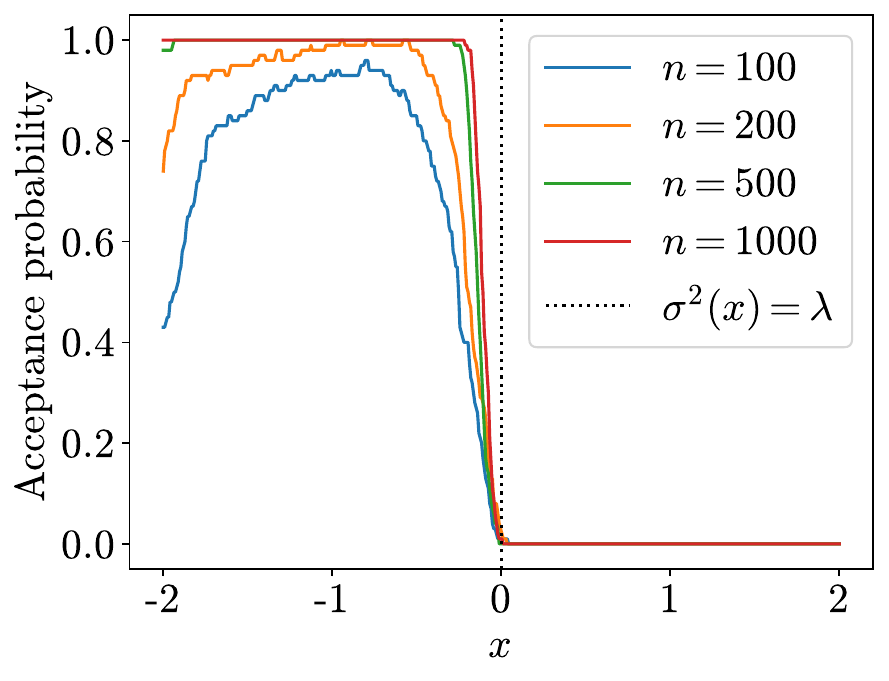}
        }\qquad
        \subfigure[$X \sim \mathcal{U}(-2,2), \: \sigma(x) = \mathtt{sigmoid}(x)$]{%
            \includegraphics[width=.45\linewidth]{figures/synthetic/xuniform_ysinx_varsigmoid_acceptance_vs_x.pdf}
        }\qquad 
        \subfigure[$X \sim \mathcal{U}(-2,2), \: \sigma(x) = \mathtt{Heaviside}(x)$]{%
            \includegraphics[width=.45\linewidth]{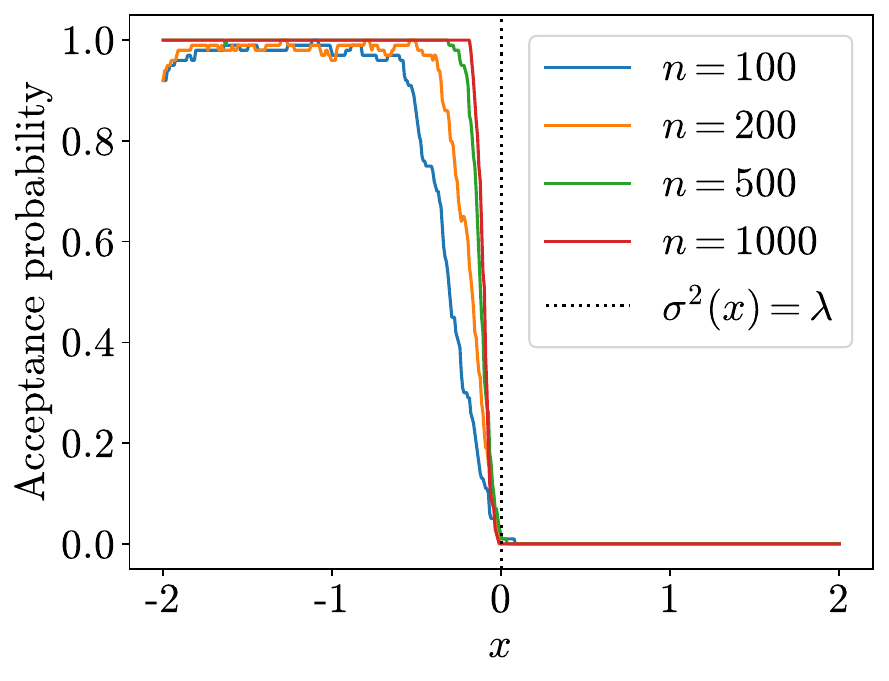}
        }
    }
  \end{figure}

\subsubsection{Expected excess risk}
\label{sec:suppl_excess_risk}
  Point-wise expected excess risk with varying sample size is presented on Figure~\ref{fig:risk_vs_x_supp}. With low sample size our method struggles in the low variance regions, while in high variance the performance is much better than the baseline method.

  \begin{figure}[htbp]
  \floatconts
    {fig:risk_vs_x_supp}
    {\caption{Expected excess risk. We sample 100 datasets of each size and for points in $[-2, 2]$.}}
    {%
        \subfigure[$X \sim \mathcal{N}(0,1), \: \sigma(x) = \mathtt{sigmoid}(x)$]{%
            \includegraphics[width=.4\linewidth]{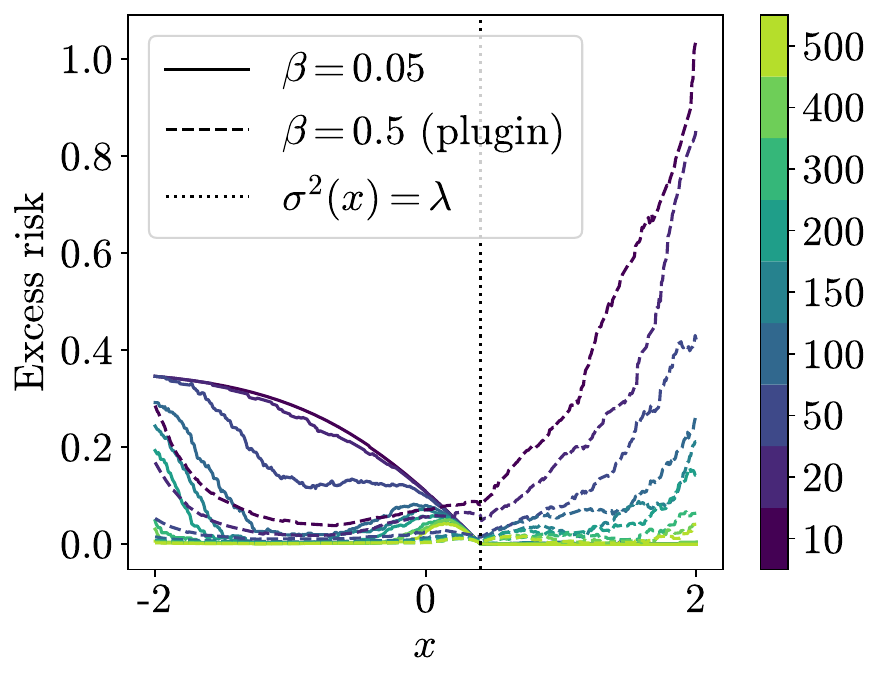}
        }\qquad 
        \subfigure[$X \sim \mathcal{N}(0,1), \: \sigma(x) = \mathtt{Heaviside}(x)$]{%
            \includegraphics[width=.4\linewidth]{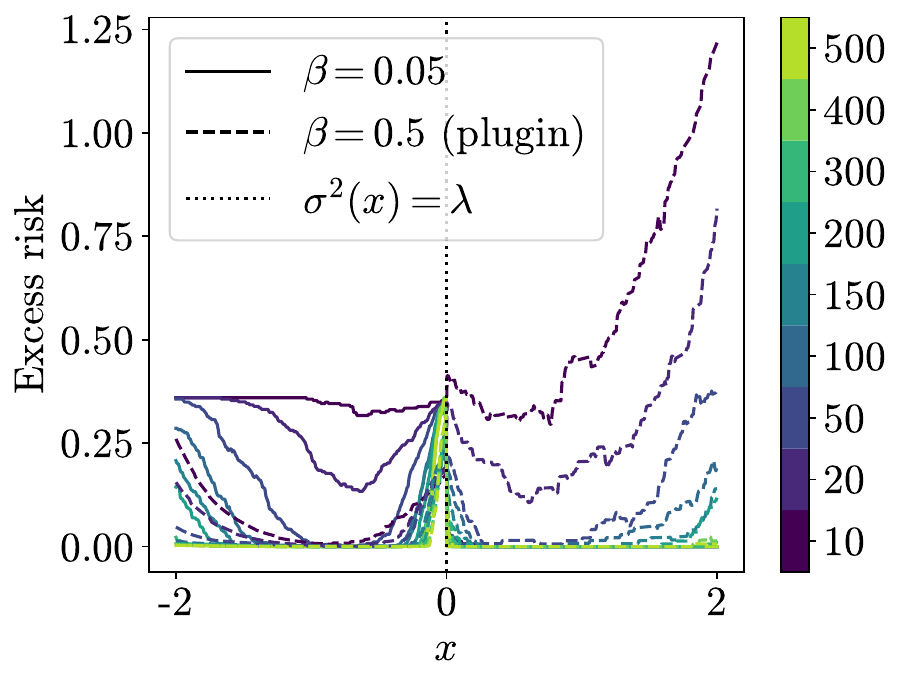}
        }\qquad
        \subfigure[$X \sim \mathcal{U}(-2,2), \: \sigma(x) = \mathtt{sigmoid}(x)$]{%
            \includegraphics[width=.4\linewidth]{figures/synthetic/xuniform_ysinx_varsigmoid_risk_vs_x.pdf}
        }\qquad 
        \subfigure[$X \sim \mathcal{U}(-2,2), \: \sigma(x) = \mathtt{Heaviside}(x)$]{%
            \includegraphics[width=.4\linewidth]{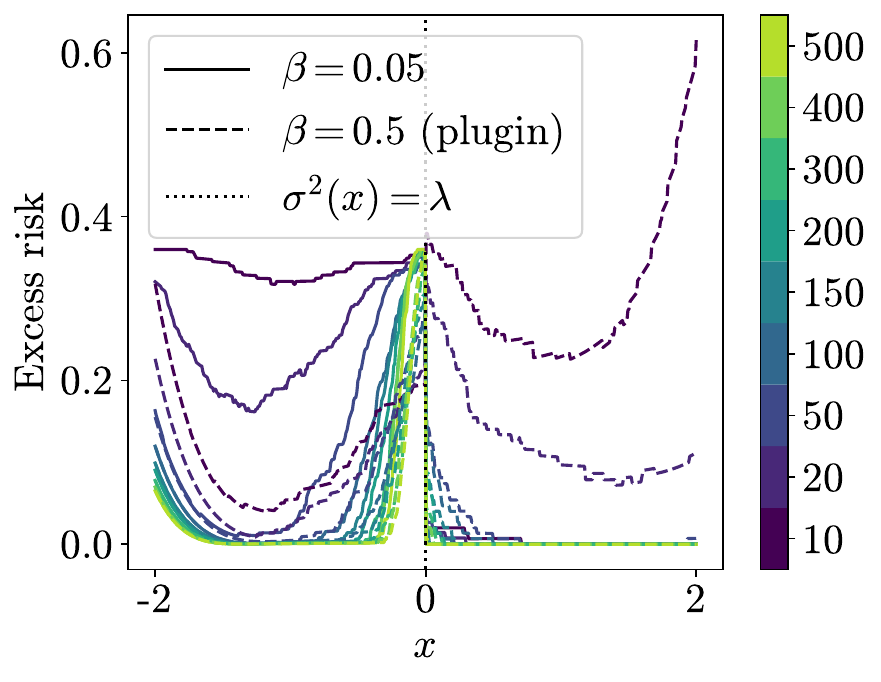}
        }
    }
  \end{figure}

  On Figure~\ref{fig:risk_vs_x_beta_supp} we present the same quantity for a fixed sample size. For simple data with step-like variance all settings perform nearly identical. For sigmoid variance, we observe the characteristic bump left of zero.

  \begin{figure}[htbp]
    \floatconts
    {fig:risk_vs_x_beta_supp}
    {\caption{Expected excess risk for a fixed sample size of $100$ and different values of $\beta$. Baseline method or ``plugin'' ($\beta = 0.5$) is shown in red.}}
    {%
        \subfigure[$X \sim \mathcal{N}(0,1), \: \sigma(x) = \mathtt{sigmoid}(x)$]{%
            \includegraphics[width=.4\linewidth]{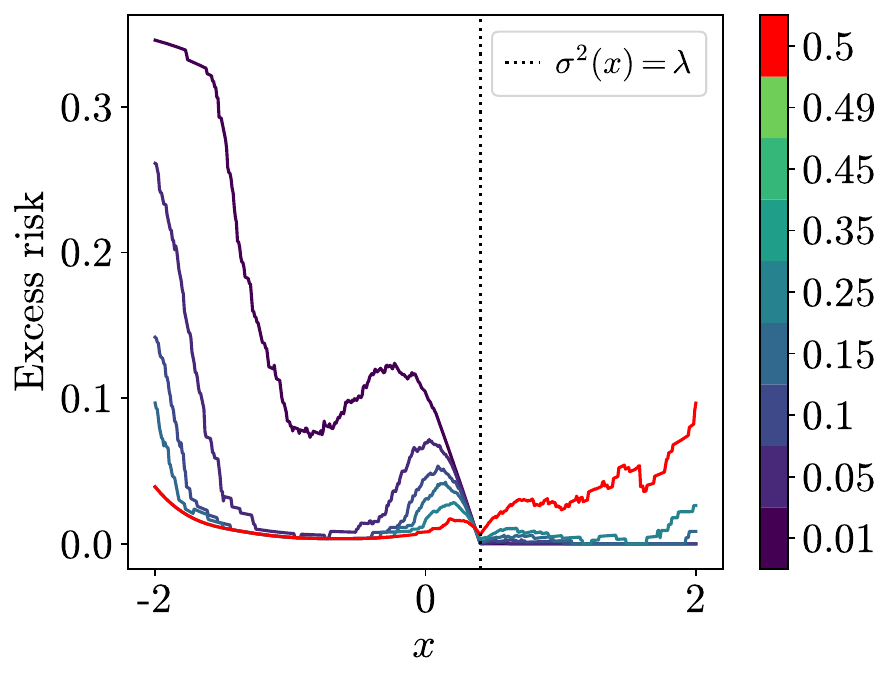}
        }\qquad 
        \subfigure[$X \sim \mathcal{N}(0,1), \: \sigma(x) = \mathtt{Heaviside}(x)$]{%
            \includegraphics[width=.4\linewidth]{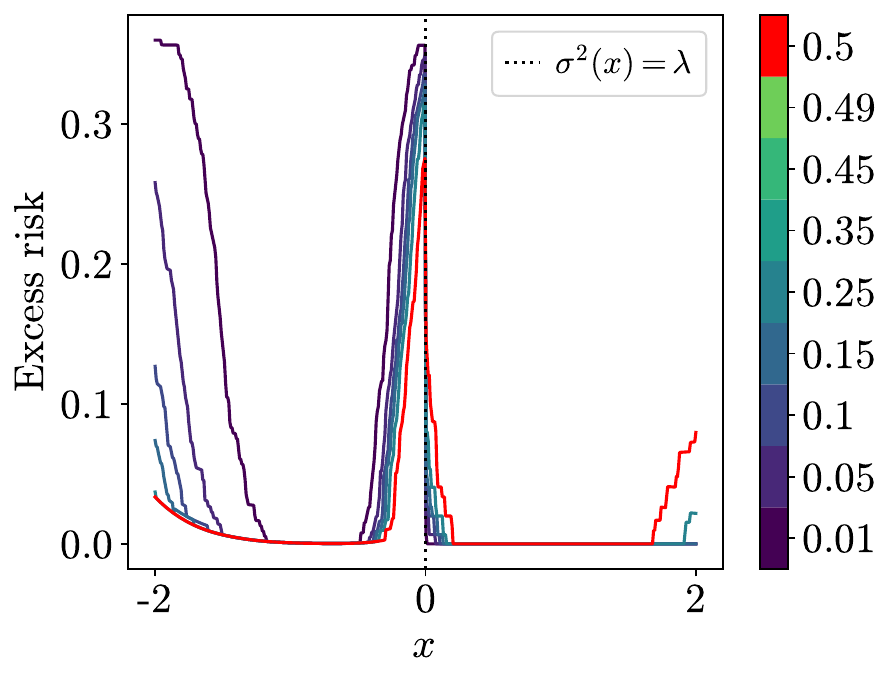}
        }\qquad
        \subfigure[$X \sim \mathcal{U}(-2,2), \: \sigma(x) = \mathtt{sigmoid}(x)$]{%
            \includegraphics[width=.4\linewidth]{figures/synthetic/xuniform_ysinx_varsigmoid_risk_vs_x_beta.pdf}
        }\qquad 
        \subfigure[$X \sim \mathcal{U}(-2,2), \: \sigma(x) = \mathtt{Heaviside}(x)$]{%
            \includegraphics[width=.4\linewidth]{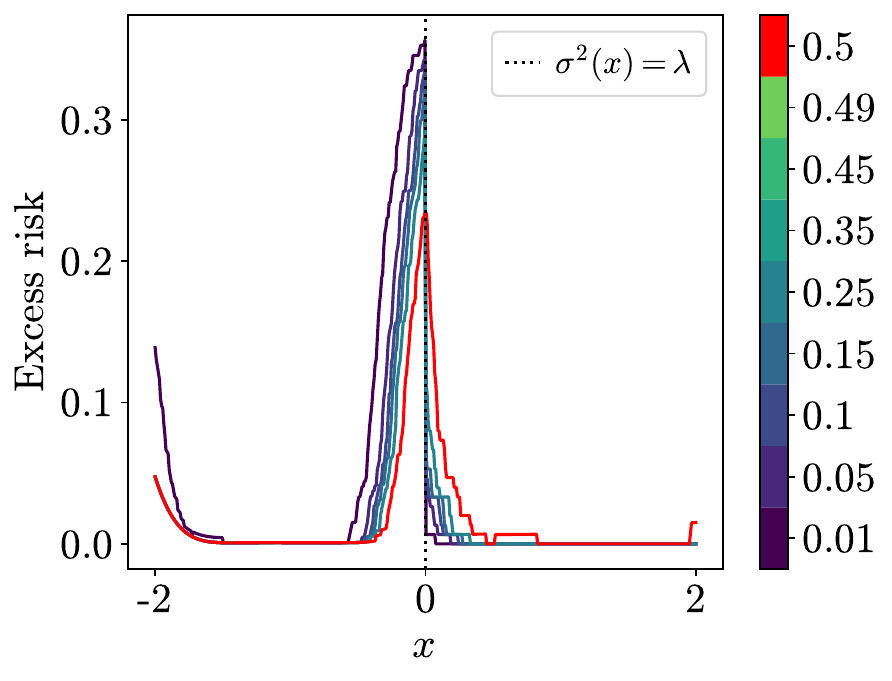}
        }
    }
  \end{figure}

\subsection{Airfoil Self-Noise Data Set}
\label{sec:suppl_airfoil}
  For the Airfoil dataset, we present additional charts for a different feature split. We can see that while acceptance has a similar dependence on $\lambda$, the actual mean squared error of accepted results highly depends on the data split. We plan to expand our study to higher-dimensional datasets to further investigate the behavior of the method.

  \begin{figure}[htbp]
  \floatconts
    {fig:air_acc_mse}
    {\caption{Experiments on Airfoil data, split 70/30 by different features.}}
    {%
        \subfigure[Split by feature 0, acceptance]{%
            \includegraphics[width=.45\linewidth]{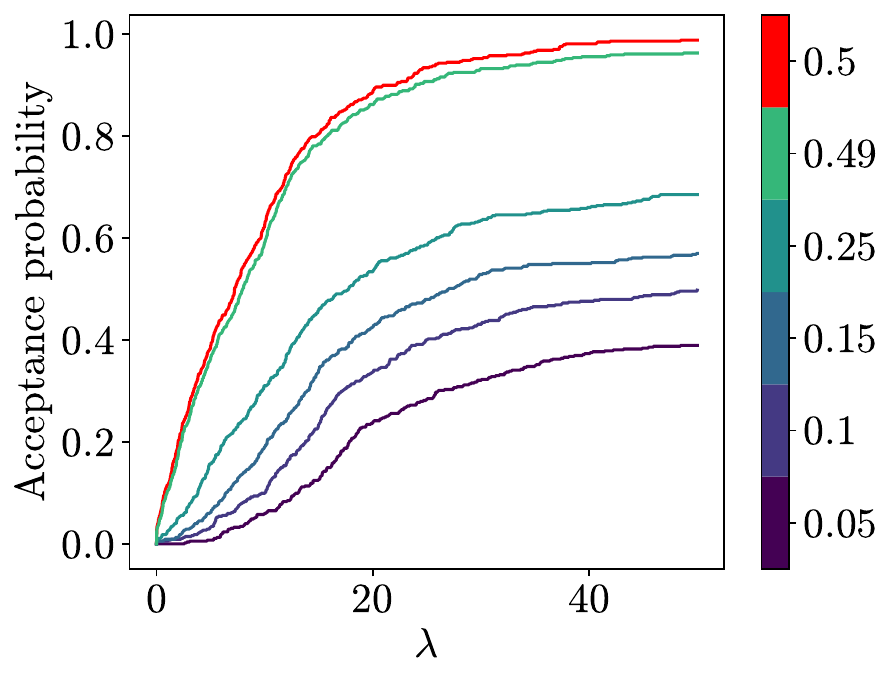}
        }\qquad 
        \subfigure[Split by feature 0, MSE]{%
            \includegraphics[width=.45\linewidth]{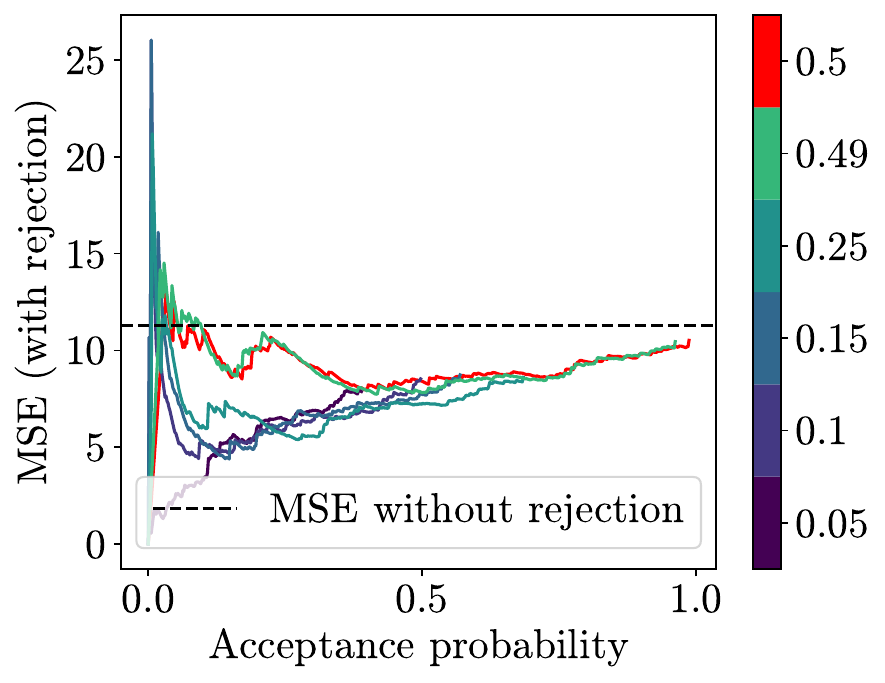}
        }\qquad
        \subfigure[Split by feature 1, acceptance]{%
            \includegraphics[width=.45\linewidth]{figures/airfoil/airfoil_acc_vs_lamda_feature1.pdf}
        }\qquad 
        \subfigure[Split by feature 1, MSE]{%
            \includegraphics[width=.45\linewidth]{figures/airfoil/airfoil_mse_vs_lamda_feature1.pdf}
        }
    }
  \end{figure}

\end{document}